\documentclass[11pt]{article}
\usepackage{caption}
\usepackage{amsmath}
\usepackage{graphicx}
\usepackage{enumerate}
\usepackage{natbib}
\usepackage{url}  

\usepackage{titlesec}

\usepackage{amsmath,amssymb,amsfonts,amsthm,amsbsy, epsfig, psfrag,epsf}
\usepackage[usenames]{color}
\usepackage{rotating}
\usepackage[colorlinks=true, linkcolor=blue, urlcolor=blue, citecolor=blue]{hyperref}
\usepackage{boxedminipage}
\usepackage{algorithm}
\usepackage{algorithmic}
\usepackage{enumerate}

\usepackage{authblk}
\usepackage{xcolor}
\usepackage{epstopdf}
\usepackage[title]{appendix}
\definecolor{DSgray}{cmyk}{0,1,0,0}
\usepackage{bbm}
\usepackage{subfigure}
\usepackage{multirow}
\usepackage{float}
\usepackage{booktabs}
\usepackage{verbatim}


\hypersetup{
colorlinks=true,
citecolor=blue,
anchorcolor = red,
}

\newtheorem{theorem}{Theorem}[section]
\newtheorem{corollary}[theorem]{Corollary}
\newtheorem{proposition}[theorem]{Proposition}
\newtheorem{lemma}[theorem]{Lemma}
\newtheorem{example}[theorem]{Example}
\newtheorem{remark}[theorem]{Remark}
\newtheorem{definition}[theorem]{Definition}
\newtheorem{assumption}{Assumption}

\newtheorem*{lemmax}{Theorem~\ref{thm:ls-fs}}

\usepackage{chngcntr}

\newcommand{\argmax}{\mathop{\rm arg\max}}

\renewcommand{\d}{{\,\mathrm{d}}}

\newcommand{\sgd}{{\tt{(SGD)}}}

\def\0{\boldsymbol{0}}
\def\A{\mathcal{A}}

\def\bp{\mathbf{P}}
\def\e{\boldsymbol{e}}
\def\E{\mathbb{E}}
\def\cE{\mathcal{E}}

\def\L{\mathcal{L}}

\def\R{\mathbb{R}}

\def\H{\mathcal{H}}

\def\ID{\mathbbm{1}}
\def\N{\mathcal{N}}

\def\P{\mathcal{P}}
\def\Pr{\mathbb{P}}

\def\ipw{{\tt IPW}}
\def\sipw{{\tt sqrt-IPW}}
\def\vanilla{{\tt vanilla}}

\def\limsup{\mathop{\overline{\rm lim}}}
\def\liminf{\mathop{\underline{\rm lim}}}
\let\hat\widehat

\definecolor{DSgray}{cmyk}{0,1,0,0}

\def\spacingset#1{\renewcommand{\baselinestretch}
{#1}\small\normalsize} \spacingset{1}
\begin{document}

{
\title{Online Statistical Inference for Contextual Bandits via Stochastic Gradient Descent}
\author{Xiangyu Chang$^1$ ~~~Xi Chen$^2$ ~~~Zehua Lai$^3$ ~~~He Li$^2$ ~~~Zhihong Liu$^4$ ~~~Yichen Zhang$^4$
\\  {\normalsize
  $^1$ Xi'an Jiaotong University \quad $^2$ New York University}\\
  {\normalsize$^3$ University of Texas at Austin\quad $^4$ Purdue University}
  }
\date{}
  \maketitle

\vspace{-3em}

\spacingset{1.75} 
\begin{abstract} 
With the fast development of big data, learning the optimal decision rule by recursively updating it and making online decisions has been easier than before. 
We study the online statistical inference of model parameters in a contextual bandit framework of sequential decision-making. We propose a general framework for an online and adaptive data collection environment that can update decision rules via weighted stochastic gradient descent. We allow different weighting schemes of the stochastic gradient and establish the asymptotic normality of the parameter estimator. Our proposed estimator significantly improves the asymptotic efficiency over the previous averaged SGD approach via inverse probability weights. We also conduct an optimality analysis on the weights in a linear regression setting. We provide a Bahadur representation of the proposed estimator and show that the remainder term in the Bahadur representation entails a slower convergence rate compared to classical SGD due to the adaptive data collection.
\end{abstract}


\section{Introduction} \label{sec:intro}
Following the seminal work of \cite{robbins1952some}, the stochastic multi-armed bandit problem has been studied extensively in the literature, where an agent aims to make optimal decisions sequentially among multiple arms, and only the selected arm reveals rewards consequently. Contextual bandit problems, where an agent's choices are influenced by covariates, have regained attention. With modern internet and data technology, they are pivotal in sequential decision-making across applications like online advertisement, precision medicine, e-commerce, and public policy.
Bandit algorithms are often formulated as minimizing the expected cumulative regret that the practitioner would have received if she knew the optimal action. While the importance of this regret minimization is undisputed, \emph{reliable uncertainty quantification} of the learned decision rule is evidently important in many featured applications. 
{ For instance, in personalized medicine with real-time treatment adaptation, making prompt, statistically reliable decisions on treatment efficacy can be critical, highlighting the need for inference to accompany dynamic decision-making. Further, an online shopping platform relying on batch inference for user preferences would risk missing timely engagement; online inference instead allows continuous, confident adaptation, enabling robust real-time personalization. Such examples underscore the crucial need for valid and reliable online inference to better guide sound policy interventions, assess risks (e.g., prompting alerts), and offer scientific insights like medication effectiveness.
} 

Consider a linear contextual bandit environment where the observed data is a triplet $\zeta_t = (X_t, A_t, Y_t)$ at each decision point $t \geq 1$, consisting of covariate $X_t$, action $A_t$, and reward $Y_t=X_t^\top\theta_{A_t}^*+\epsilon_t$ where $\theta_{A_t}^*\in\R^p$ is unknown parameters of interest governed by a set $\A$ of finite actions, and $\epsilon_t\in\R$ is the noise under certain modeling assumptions. For illustrative simplicity, we consider a binary action space $\A=\{0,1\}$ corresponding to a duplet of underlying parameters $(\theta_0^*,\theta_1^*)\in \R^{p}\times \R^p$.
Consider a decision rule $\pi: \R^p\rightarrow \R^{|\A|}$ which returns a distribution of actions $A\in \A$ given an observed covariate $X$. It is natural to believe that the optimal decision rule under certain covariate $X$ is corresponding to the arm with the largest expected reward, that is, $\pi^{opt}(X)=\argmax_{A\in\A}\E(Y\mid X,A)$.  Especially, in the linear contextual bandit, the optimal decision rule becomes
\begin{equation}\label{eq:opt decision}
    \pi^{opt}(X)=\argmax_{A\in\A}X^\top\theta_A^*=\ID\{X^\top\theta_1^*>X^\top\theta_0^*\}.
\end{equation}
In fact, \eqref{eq:opt decision} can be applied in extensive scenarios where the expected reward of each action is a monotonic function of $X^\top\theta^*_a$ (see Examples~\ref{eg:quant-reg} and \ref{eg:log-reg}). Since $\theta^*$ is unknown, we need to modify the optimal decision rule as, for example,
$\hat{\pi}^{opt}(X_t)=\ID\{X_t^\top\theta_{1,t-1}>X_t^\top\theta_{0,t-1}\}$,
where $(\theta_{0,t-1},\theta_{1,t-1})\in\R^{2p}$ is recursively updated according to some algorithm designed to approach $\theta^*$. Note that $\hat{\pi}^{opt}$ depends on $\H_{t-1}$ that denotes the trajectory of observations until $t-1$. A typical policy $\pi$ prefers the action with a higher expected reward practically realized through $\hat{\pi}^{opt}$, while reserving a small probability to explore random actions to avoid potential myopic short-sighted exploitation. In an example of $\varepsilon$-greedy policy,
\begin{align} \label{eq:eps-greedy-intro}
\Pr\big(A_t = a \mid X_t, \theta_{0,t-1},\theta_{1,t-1}\big) = (1-\varepsilon)\ID \big\{a=\argmax_{a\in\A}X_t^\top \theta_{a,t-1}\big\} + \frac{\varepsilon}{2},
\end{align}
where the action is selected according to the policy $A_t\sim\pi(X_t,\H_{t-1})$.
This procedure heavily relies on a series of estimators $\big(\theta_{0,t-1},\theta_{1,t-1}\big)\in\R^{2p}$ on-the-fly, of the underlying model parameters. Despite that a return-oriented policy would undoubtedly favor the action with a higher reward, it is often as crucial to obtain the confidence of decisions, i.e., conducting statistical inference for $(\theta_{0}^*,\theta_1^*)$ in the prescribed applications. This model of statistical inference of model parameters in decision-making problems appears recently in literature (see, e.g., \citealp{chen2021a,zhang2021statistical},
and a brief survey in Section \ref{subsec:related} below). A typical inferential task provides a confidence interval of the underlying parameters $(\theta_{0}^*,\theta_1^*)$ or significance levels when testing hypotheses of parameters. 

{Since the sequential decision-making relies on updating $\big(\theta_{0,t-1},\theta_{1,t-1}\big)$ for every $t$ throughout the horizon, it is crucial to provide a \emph{computationally efficient} algorithm for \emph{fully online} estimation and inferences. 
The existing literature on sequential decision-making mostly focuses on the convergence rate, while computation and storage efficiency of the algorithm are often optimistically neglected. Particularly, they often provide online decision-making procedures governed by an offline algorithm of parameter estimation. 
For example, in the linear regression settings, at each iteration $t$, an offline M-estimator $(\theta_{0,t},\theta_{1,t})$ is often obtained using the entire sample path $\big\{(X_s,Y_s)\big\}_{s\leq t}$ up to time $t$, which typically requires a $\mathcal{O}(t)$ per-iteration computation cost. As such, the total computation accumulates in a non-scalable manner to at least $\mathcal{O}(T^2)$ over the horizon $T$ (see Figure~B.10
 of the supplement).

To facilitate computationally efficient inference in fully online decision-making, we adopt the stochastic gradient descent (SGD) algorithms \citep{robbins1951stochastic}. Thanks to its computational and storage efficiency, SGD has been widely used in large-scale stochastic optimization. }Let $\theta_0$ denote an initial estimation. The SGD iteratively updates as follows,
\begin{align} \label{eq:vanilla-sgd}
\theta_t = \theta_{t-1} - \eta_t \nabla \ell(\theta_{t-1}; (X_t,Y_t)),
\end{align}
where $\eta_t$ is a positive non-increasing sequence referred to as the step-size sequence and $\nabla \ell$ is the gradient for smooth individual loss function $\ell$. For the SGD update above, under the \textit{i.i.d.} setting, the classical result by \cite{polyak1992acceleration} uses the average $\bar{\theta}_t^{\sgd} = t^{-1} \sum_{s=0}^{t-1} \theta_s$ as the final estimator to accelerate the estimation. They characterize the limiting distribution and statistical efficiency of the averaged SGD (ASGD), i.e., 
\begin{align*}
\sqrt{t} \big(\bar{\theta}^{\sgd}_t - \theta^*\big) \overset{d}{\rightarrow} \N\big(0, (H^{\sgd})^{-1} S^{\sgd} (H^{\sgd})^{-1} \big),
\end{align*}
given predetermined step sizes $\eta_t = \eta_0 t^{-\alpha}$ for $\eta_0>0$, $0.5 < \alpha < 1$. Here $(H^{\sgd},S^{\sgd})$ is the Hessian and Gram matrix of the gradient of loss at $\theta = \theta^*$. For well-specified models under \textit{i.i.d.} noises, this asymptotic covariance matrix matches the inverse Fisher information matrix, and thus the resulting averaged estimator $\bar{\theta}_t^{\sgd}$ is asymptotically efficient.

Even though the literature of SGD inference mainly focuses on \emph{i.i.d.} samples, the SGD algorithm indeed fits well into the online decision-making scheme, 
as the underlying parameter $(\theta_0^*,\theta_1^*)$ is the solution to the following stochastic optimization, 
\begin{align} \label{eq:oracle-intro}
\theta_a^* \in \underset{\theta \in \R^p}{\operatorname{argmin}} \, \mathcal{L}_a(\theta):=\E\left[\ell\big (\theta;(X_t,A_t,Y_t)\big) \mid X_t, A_t=a\right],\quad a\in\A,
\end{align}
where $\ell$ denotes the loss function designed according to the modeling assumptions of $\{\epsilon_t\}$. For example, in a linear model $Y_t=X_t^\top \theta_{A_t}^*+\epsilon_t$ with zero-mean noise $\{\epsilon_t\}$,  a natural choice of $\ell\big(\theta;(X_t,A_t,Y_t)\big)=\big(Y_t-X_t^\top\theta_{A_t}\big)^2$ is the squared loss. If $\epsilon_t$ is modeled with a zero median, a natural choice of $\ell\big(\theta;(X_t,A_t,Y_t)\big)$ is the least absolute deviation (LAD) loss, $ \big|Y_t-X_t^\top\theta_{A_t}\big|$, a special case of quantile loss. In both scenarios above, the minimizer of the population loss $\mathcal{L}_a(\theta)$ depends on the action $A_t=a$ but not on the distribution of $X_t$.

{ The application of SGD in contextual bandits has been explored in literature (e.g., \citealp{chen2021b}) for an $\varepsilon$-greedy policy. Particularly, a weighted SGD procedure updates
\begin{align} \label{eq:weighted-sgd-intro}
\theta_{a,t} = \theta_{a,t-1} - \eta_t w_t \nabla \ell\big(\theta_{t-1}; (X_t,A_t=a,Y_t)\big),\quad w_t=\frac{\ID_{\{A_t=a\}}}{2\Pr(A_t\mid X_t,\theta_{t-1})},
\end{align}
under a specific weighting scheme, inverse probability weighting. Notably, at every time $t$, the outcome $Y_t$ in each observation $(X_t,A_t,Y_t)$ is adaptively collected upon the decision of action $A_t$. The weight $w_t$ in (\ref{eq:weighted-sgd-intro}) indicates that, at each time $t$, only one $\theta_{a,t}$ between the duplet $\big(\theta_{0,t},\theta_{1,t}\big)$, is updated by SGD. Inverse probability weighting (IPW) is utilized to demonstrate that the weighted stochastic gradient $w_t \nabla \ell\big(\theta_{t-1}; (X_t,A_t,Y_t)\big)$ in (\ref{eq:weighted-sgd-intro}) is an unbiased estimator of the gradient of a deterministic population loss that is independent to the entire the historical information. Precisely in this setting of \eqref{eq:weighted-sgd-intro}, that population loss is indeed the \emph{equal-weighted} combination of the population losses $\frac{1}{|\mathcal{A}|}\sum_{a\in\mathcal A}\mathcal{L}_a$, independent to the historical information. While the unbiasedness and independence properties clear the technical difficulty of theoretical analysis of the asymptotic normality of the IPW-SGD estimator, IPW inflates its asymptotic variance by a factor of order $1/\varepsilon$. This results in highly volatile estimators and excessively wide confidence intervals, compromising the reliability of statistical inference.}

Designing algorithms to ameliorate decision-making and enhance the asymptotic efficiency of estimators remains both challenging and important. In this paper, we allow a general choice of the weighting parameter $w_t$ in \eqref{eq:weighted-sgd-intro}, which admits the IPW weights as a special case and derives the explicit formula for the asymptotic distribution of the generalized-weighting ASGD algorithm, thus provides us a way to compare different choices of $w_t$ and even optimize over $w_t$ for some simple models.  
Our proposed estimator significantly improves the asymptotic efficiency over IPW-ASGD and achieves comparable efficiency if the practitioner picks one arm steadily. This estimator helps construct narrow yet reliable confidence intervals for the underlying parameter of interest. The analysis also reveals a recommendation of optimal choices of weights $w_t$ in certain policies. To overcome the technical challenge raised in history-dependent weighting parameters, we propose a new definition of the loss function, which is different from the loss function used in classical SGD literature (e.g., \citealp{chen2016statistical}) and adaptive SGD literature \citep{chen2021b}. We use two parameters, $\theta$ and $\theta^\prime$, to separate the effect of weighting parameters in SGD and that of decision-making procedures in the local geometric landscape of the loss function.

As a separate interest, this paper establishes a general framework that allows non-smooth loss functions such as quantile loss to estimate conditional quantiles of the reward $Y_t$, 
which finds ubiquitous applications such as operations management of business inventory and risk management of financial assets. Therefore, it is worth exploring the use of quantile-based objective functions in sequential decision-making.

{Additionally, our analysis facilitates both degenerate and non-degenerate models, where the former refers to the same underlying parameter under different actions, i.e., $\theta_{0}^*=\theta_{1}^*$. An important example is a variant of Thompson Sampling in the degenerate model based on the Hodges estimator, as studied for offline M-estimators in \cite{zhang2021statistical}. }

As a summary, we study a general framework of online statistical inference for contextual bandit. {
This paper is considered as an extensive generalization over \cite{chen2021b} from three aspects: weighting schemes; handling non-smooth loss functions via stochastic subgradient; and applicability to wider range of arm selection policies. }We summarize the contribution and emphasize the technical challenges in the following facets.
\begin{itemize}
\item SGD 
with inverse probability weighting (IPW) suffers from an unbounded asymptotic variance when the exploration rate tends to $0$, i.e., the relative efficiency of adaptive models versus non-adaptive models diverges to infinity. Our proposed algorithm features a general policy with a flexible specification of the weights to avoid such deficiency and obtain a bounded relative efficiency. We further provide some practical insights into the optimal weight specification in linear regression that attains the lowest asymptotic covariance matrix among a class of weight specifications.
\item We analyze SGD that features stochastic subgradients under nonsmooth losses. An important example is the quantile regression which can be used for risk-averse or risk-aware decision-making.
Moreover, this example provides robustness to the outliers of the reward due to the fact that the objective function is globally Lipschitz. 

\item Beyond the asymptotic normality of the proposed estimator, we further establish an analysis of the higher-order remainder term in its Bahadur representation. {In classical \emph{i.i.d.} SGD settings, the remainder term has the rate of $\mathcal{O}_p\big(t^{-\alpha+\frac12}+t^{-\frac{\alpha}{2}}+t^{\alpha-1}\big)$.} {On the contrary, under the non-degenerate adaptive setting with two example policies, the reminder term entails a slower rate of $\mathcal{O}_p\big(t^{-\alpha + \frac{1}{2}} + t^{-\frac{\alpha}{4}}+t^{\alpha-1}\big)$.}
{We attribute the  slower rate to the nature of adaptive data collection, which introduces the temporal difference of the gradient noise, unlike the independent structure in the classical SGD}.
\end{itemize}

The remaining of the paper is organized as follows. Section~\ref{sec:setup} introduces the general weighted SGD for contextual bandit and present illustrative examples of the classical regression problems. 
{In Section~\ref{sec:theory}, we formulate the problem under general weighting schemes and policies into stochastic optimization and study the asymptotic distribution of the SGD estimator. }Section~\ref{sec:finite-sample} establishes its Bahadur representation and discuss the optimal choices of the step sizes. Section \ref{sec:online-inf} presents an online inference procedure to construct the confidence intervals. {
In Section~\ref{sec:var-discuss}, we justify our framework for two illustrative regression examples under two specified arm selection policies, and specifically demonstrate the applicability to the degenerate model for a modified $\varepsilon$-greedy. We further present a comparison of the statistical efficiency under different weighting schemes, 
and provide practical implications on weight choices based on an explicit analytic form of the asymptotic covariance matrices. }In Section~\ref{sec:num}, we conduct
simulation studies and real data analyses which lend numerical support to our theoretical claims. A summary of notations throughout the paper is relegated to Section A of the supplementary material.

\subsection{Related works}\label{subsec:related}

\paragraph{Online statistical inference for model parameters in SGD}
The asymptotic distribution of ASGD is first given in \cite{ruppert1988efficient} and \cite{polyak1992acceleration}. Since then, there has been a rapid growth of interest recently in conducting statistical inference for model parameters in stochastic gradient algorithms. \cite{chen2016statistical,chen2024online} proposed two online estimators (plug-in and batch-means) in constructing estimators of limiting covariance matrix of ASGD, of which \cite{zhu2021online} extended the batch-means to overlapped batches. \cite{fang2018online} proposed a perturbation-based resampling procedure to conduct inference for ASGD. \cite{tang2023acceleration} studies a momentum-based variant of SGD. \cite{su2018uncertainty} proposed a tree-structured inference scheme to construct confidence intervals. \cite{wen2023online} studies online inference for tensors. 
\cite{lee2021fast,lee2022fast} generalized to a functional central limit theorem and proposed an online inference procedure called random-scaling for smooth objectives and quantile regression, respectively. 

\paragraph{Statistical inference in online decision-making}
\cite{chen2021a} studied statistical inference under a linear contextual bandit framework. 
\cite{zhang2021statistical,zhang2022statistical} conducted inference for $M$-estimators in contextual bandit and non-Markovian environments. 
\cite{hao2019bootstrapping} used multiplier bootstrap to offer uncertainty quantification for exploration in the bandit settings.
\cite{deshpande2018accurate,khamaru2021near} studied inference for adaptive linear regression where the
vector contexts are correlated over time.  \cite{zhan2021off,hadad2021confidence} employed adaptive weighting of observations during off-policy evaluation and constructed confidence intervals. \cite{chen2021b,han2022online} conducted statistical inference under the contextual bandit settings via SGD. Related statistical inference literature in reinforcement learning as a well-known online decision-making setting also exists. \cite{ramprasad2022online,liu2023online} conducted statistical inference for TD learning. \cite{shi2021statistical} constructed the confidence interval for policy values in Markov decision processes. \cite{shi2022off,chen2022reinforcement} conducted statistical inference for confounded and heterogeneous MDP. 

\section{Problem Setup} \label{sec:setup}

We consider a contextual bandit environment where the observed data at each decision point $t$ is a triplet $\zeta_t = (X_t, A_t, Y_t)$ for all $t \geq 1$, consisting of covariate $X_t$, action $A_t$, and reward $Y_t$. Define $\mathcal{F}_t=\sigma(\{\zeta_s\mid 1\leq s\leq t\})$ is the $\sigma$-algebra of all past triplets up to time $t$. This paper considers a finite action space, i.e., $A_t \in \A$ and $| \A | < \infty$. We assume a stochastic contextual bandit environment in which $\left\{X_{t}, Y_{t}(a): a \in \A\right\} \stackrel{i . i . d}{\sim} \P \in \bp$ for all $t \geq 1$. The contextual bandit environment distribution $\P$ is in a space of possible environment distributions $\bp$. 
{Here $Y_t(a)$, also known as the potential outcome in causal inference \citep{rubin2005causal},  corresponds to the (heuristic) reward $Y_t$ given a fixed action $a$ regardless of the realized action $A_t$.}
Note that $Y_t(a)$ is observed for $A_t=a$ only, but not observed for any other $a'\in\A\backslash \{a$\}.
We define the trajectory until time $t$ as $\H_{t}:=\left\{X_{s}, A_{s}, Y_{s}\right\}_{s=1}^{t}$ for $t \geq 1$ and $\H_{0}:=\emptyset$. {Actions $A_{t} \in \A$ are selected according to some stochastic policy $A_t \sim \pi \left(X_{t}, \H_{t-1}\right)$, which defines a probability distribution over actions, $\Pr(A_t=a\mid X_t,\H_{t-1})$.}
Although the covariate–reward tuples are \emph{i.i.d.}, the observed data $\{X_t, A_t, Y_t\}_{t\ge1}$ are not, because actions are selected adaptively via policies $\pi(X_t,\mathcal H_{t-1})$ that depend on past data $H_{t-1}$, a defining feature of adaptively collected observations.

We are interested in constructing confidence regions for some unknown $\theta_a^* \in \R^{p},a\in\A$. Under the finite action space where $| \A | < \infty$, we can use $\theta^*\in\R^{|\A|p}$ as the concatenated vector of $\theta_a^*$ for all $a \in \A$, that is, $\theta^*=\big(\theta^{*\top}_0,\cdots,\theta^{*\top}_{|\A|-1}\big)^{\top}$, where we assume that $\theta^*_a$ is a conditionally minimizing value of some loss function $\ell(\theta; \zeta)$ for $\P \in \bp$,
\begin{align} \label{eq:oracle}
\theta_a^*(\P) \in \underset{\theta \in \R^p}{\operatorname{argmin}} \, \E\left[\ell\left(\theta; \zeta\right) \mid X, A=a\right].
\end{align}
When there is no ambiguity, we employ the notation $\theta^*_a$ for simplicity.
Note that \eqref{eq:oracle} represents an implicit modeling assumption that such an underlying $\theta_a^*$ does not depend on $X$ for a given loss $\ell(\theta; \zeta)$, which is satisfied in many statistical applications. 
{In the following, we illustrate several  classical regression examples where the loss function $\ell(\theta;\zeta_t)$ is a functional on $X^\top\theta_{A_t}$, where $\theta \in \R^{d}$ is the concatenated vector of $\theta_{A_t} \in \R^p$ for all possible choices of $A_t \in \mathcal{A}$ and $d = p |\mathcal{A}|$, and therefore, under the binary action settings where $\mathcal{A}=\{0,1\}$, the notation $\theta_{[1:p]}$ (and $\theta_{[p+1:2p]}$) is referred to $\theta_{0}$ (and $\theta_{1}$), respectively. All these regression examples satisfies \eqref{eq:oracle}, and we will refer to them throughout the paper.}
\begin{example}[Linear Regression] \label{eg:ls-reg}
Consider a linear contextual bandit problem where $
\E[Y_t \mid X_t,A_t]=X_t^\top\theta^*_{A_t}$, 
and we can further rewrite this as
$    \E[Y_t \mid X_t,A_t]=(1-A_t)\left(X_t^\top\theta^{*}_{[1:p]}\right)+A_t\left(X_t^\top\theta^{*}_{[p+1:2p]}\right),
$ where $\theta^* \in \R^d$ is the concatenated vector of $\theta^{*}_{[1:p]}$ and $\theta^{*}_{[p+1:2p]}$, the contextual bandit environment $\left\{X_{t}, Y_{t}(a): a \in \A\right\}$ $\stackrel{i . i . d}{\sim} \P \in \bp$ for all $t \geq 1$, and $\A=\{0,1\}$. The true reward $Y_t$ is generated by $\E[Y_t \mid A_t, X_t] + \cE_t$ where $\{\cE_t\}$ are \emph{i.i.d.} random error with mean zero and variance $\sigma^2$. A least squares objective is often used in linear regression. In binary-action settings, the loss function $\ell$ is defined as
\begin{align*}
\ell(\theta; \zeta_t) = \frac{1}{2} (1 - A_t)\left(Y_t - X_t^\top \theta_{[1:p]}\right)^2 + \frac{1}{2} A_t\left(Y_{t} - X_t^\top \theta_{[p+1:2p]}\right)^2.
\end{align*}  
\end{example}
To avoid confusion, we refer the term linear regression to the problem of least square regression. In a linear regression, given the covariates $X_t$, the practitioner typically selects the arm $A_t$ favoring a higher expected reward $\E[Y_t|X_t,A_t]$.

\begin{example}[Quantile Regression] \label{eg:quant-reg}
Consider a  linear contextual bandit problem where
\begin{align*}
Y_t = Q_{\tau}(Y_t|X_t,A_t)+\cE_t,\quad \text{where }Q_{\tau}(Y_t|X_t,A_t)=X_t^\top \theta^{*}_{A_t},
\end{align*}
and $\{\cE_t\}$ are \emph{i.i.d.} random noise such that,
$\Pr(\cE_t \leq 0{}\mid X_t,A_t) = \tau$
for some given quantile level $\tau \in (0,1)$.  
In binary-action settings, we have $Y_t = (1-A_t) X_t^\top \theta^{*}_{[1:p]} + A_t X_t^\top \theta^*_{[p+1:2p]}+\cE_t$,
and the loss function can be written as  $
\ell(\theta; \zeta_t) = (1 - A_t)\rho_\tau\left(Y_t - X_t^\top \theta_{[1:p]}\right) + A_t\rho_\tau\left(Y_{t} - X_t^\top \theta_{[p+1:2p]}\right)$, 
where $\rho_\tau(u)= u (\tau-\ID(u<0))$. 
\end{example}
In Example \ref{eg:quant-reg}, the practitioner favors an arm with a higher conditional quantile of reward instead of higher expected rewards. Quantile regression is a statistical technique widely applied in the realm of economics and social sciences, for example, allowing researchers to examine how various factors affect different percentiles of the wage distribution rather than just the average, which provides insights into understanding income disparities affected by demographic characteristics, education levels, and other variables. Example \ref{eg:quant-reg} offers a useful bandit model in risk-averse or risk-aware decision-making, when the attention is given to a certain quantile of a population instead of the mean. It is worthwhile to note that the quantile loss is nonsmooth and often overlooked in bandit literature. 

\begin{example}[Logistic Regression] \label{eg:log-reg}
Consider a two-arm contextual bandit problem under the logistic model with binary rewards where $A_t\in \A=\{0,1\}$, $Y_t\in\{-1,1\}$, where $
\Pr(Y_t\mid X_t,A_t) =  \left(1+\exp\left(-Y_tX_t^\top\theta^*_{A_t}\right)\right)^{-1},$
or in binary-action settings, $
\Pr(Y_t\mid X_t,A_t) =  \left(1+\exp\big(-(1-A_t)Y_tX_t^\top \theta^{*}_{[1:p]} -A_tY_tX_t^\top \theta^*_{[p+1:2p]}\big)\right)^{-1}$. 
We consider the entropy loss
\begin{align*}
\ell(\theta; \zeta_t) = (1 - A_t) \log \left(1 + \exp \left(-Y_t  X_t^\top \theta_{[1:p]}\right) \right) + A_t \log \left(1 + \exp \left(- Y_tX_t^\top \theta_{[p+1:2p]}\right) \right).
\end{align*}
\end{example}
In Example \ref{eg:log-reg}, the reward $Y_t$  is binary and a parametric generalized linear model is assumed for the distribution of $Y_t$ given $X_t$ and $A_t$. The entropy loss is a convex function but not guaranteed strongly convex everywhere.

As illustrated by the above three examples, the data $\zeta_t=(X_t,A_t,Y_t)$ for each iteration $t$ is adaptively collected. Now we consider a generalized version of the classical SGD \eqref{eq:vanilla-sgd} with weights $w_t$ depends only on the triplet $(X_t, A_t, \theta_{t-1})$, as follows,
\begin{align} \label{eq:weighted-sgd}
\theta_t = \theta_{t-1} - \eta_t w_t \nabla \ell(\theta_{t-1}; \zeta_t),
\end{align}
{where $\eta_t = \eta_0 t^{-\alpha}$, $\eta_0>0$ and $\alpha\in(1/2,1)$.}
It is noteworthy to mention that the above updating rule can be considered as a general version of \eqref{eq:weighted-sgd-intro}, which allows arbitrary weight $w_t$ in the SGD updates. Even though our theory allows for pretty general weight specifications of $w_t$, we emphasize three popular choices of weight $w_t$ as examples throughout the discussions of the paper. 
\begin{align}
    \bullet\ &\text{Inverse probability weighting (\ipw)}: 
    \label{eq:ipw-weight}
        w_t(\theta_{t-1};X_t, A_t) = \dfrac{1}{2 \, \Pr(A_t \mid X_t, \theta_{t-1})};\\
    \bullet\ &\text{Square-root\,importance\,weights\,(\sipw):} 
    \label{eq:sqrt-weight}
        w_t(\theta_{t-1};X_t, A_t) = \sqrt{\dfrac{1}{2 \Pr(A_t \mid X_t, \theta_{t-1})}};\\
    \bullet \ &\text{Vanilla weights (\vanilla)}:
    \label{eq:vanilla}
        w_t(\theta_{t-1};X_t, A_t) = 1.
\end{align}

These weighting schemes are well-rooted in literature, for example, \hyperref[eq:ipw-weight]{\ipw} is studied by \cite{chen2021b,han2022online} to correct the action distribution towards a deterministic equal-weighted aggregation over $\A$ in the population, and {\hyperref[eq:sqrt-weight]{\sipw} is used for offline estimation in \cite{hammersley2013monte} and \cite{zhang2021statistical}. It is noteworthy to mention that, our proposed method is not limited to analyzing these weights but applied to general weight specifications. 
Before presenting main results, we revisit the three aforementioned motivating examples with binary action $\A=\{0,1\}$ and illustrate the algorithm for the three models. Note that $\theta_t$ is the concatenated vector of dimension $d=|\A|p=2p$. 

\begin{itemize}
\item Linear regression (Example \ref{eg:ls-reg}). The weighted SGD \eqref{eq:weighted-sgd} is written as
{\small\begin{align*}
\theta_t
&= \theta_{t-1}- \eta_t 
\begin{pmatrix}
w_t\big(\theta_{[1:p],t-1};X_t,0\big)\left(X_t^\top \theta_{[1:p],t-1} - Y_t\right)\ID_{\{A_t = 0\}}\,X_t\\
w_t\big(\theta_{[p+1:2p],t-1};X_t,1\big)  \left(X_t^\top \theta_{[p+1:2p],t-1} - Y_t\right)\ID_{\{A_t = 1\}}\,X_t
\end{pmatrix}.
\end{align*}}
\item Quantile regression (Example \ref{eg:quant-reg}).
{\small\begin{align*}
\theta_t&= \theta_{t-1}- \eta_t 
\begin{pmatrix}
w_t\big(\theta_{[1:p],t-1};X_t,0\big)\left(\tau - \ID(Y_t - X_t^\top \theta_{[1:p],t-1} < 0)\right)\ID_{\{A_t = 0\}}\,(-X_t)\\
w_t\big(\theta_{[p+1:2p],t-1};X_t,1\big)  \left(\tau - \ID(Y_t - X_t^\top \theta_{[p+1:2p],t-1} < 0)\right)\ID_{\{A_t = 1\}}\,(-X_t)
\end{pmatrix}.
\end{align*}}
\item Logistic regression (Example \ref{eg:log-reg}):
{\small\begin{align*}
\theta_t&= \theta_{t-1}- \eta_t 
\begin{pmatrix}
w_t\big(\theta_{[1:p],t-1};X_t,0\big)\left(1 + \exp \big(Y_t X_t^\top \theta_{[1:p],t-1}\big)\right)^{-1} Y_t\ID_{\{A_t = 0\}}\,(-X_t)\\
w_t\big(\theta_{[p+1:2p],t-1};X_t,1\big)  \left(1 + \exp \big(Y_t X_t^\top \theta_{[p+1:2p],t-1}\big)\right)^{-1} Y_t\ID_{\{A_t = 1\}}\,(-X_t)
\end{pmatrix}.
\end{align*}}
\end{itemize}
Given our path of $\{\theta_t\}_{t \geq 1}$, we assume the policy $\pi \left(X_{t}, \H_{t-1}\right)$ depends on the history $\H_{t-1}$ only through $\theta_{t-1}$, our estimator from the latest step, i.e., $A_t \sim \pi \left(X_{t}, \theta_{t-1}\right)$. 
{In the next section, we will demonstrate our main theoretical results under general policy $\pi \left(X_{t}, \theta_{t-1}\right)$ and weighting schemes which satisfy certain conditions.}
\section{Asymptotic Properties of Adaptive Weighted SGD} \label{sec:theory}
To analyze the asymptotic behavior of the weighted SGD update \eqref{eq:weighted-sgd}, we construct the following population objective function, such that \eqref{eq:weighted-sgd} corresponds to iterative updates within the stochastic optimization framework of $\L_{\theta^\prime}(\theta)$,
\begin{align} \label{eq:custom-loss}
\L_{\theta^\prime}(\theta) =  \E_{\P_X}\left[ \E_{\pi(X, \theta^\prime)} \E_{\P_{Y\mid X,A}}\left(w(\theta^\prime; X, A)\ell(\theta; X, A, Y) \mid X,A \right) \right],
\end{align}
where $A \sim \pi(X, \theta^\prime)$, and $\theta^\prime, \theta \in \R^d$. 
Note that the objective $\L_{\theta^\prime}(\theta)$ is a function of $\theta$ with a parameter $\theta'$ corresponding to the current estimate used to select the action. In the practical use of this population objective, we typically let $\theta'=\theta_{t-1}$ at iteration $t$ for on-policy learning. 
Below we will always use the expression $\nabla \L_{\theta^\prime}(\theta)$ to represent the partial gradient of $\L_{\theta^\prime}(\theta)$ with respect to the variable $\theta$, i.e.,
\begin{align*}
\nabla \L_{\theta^\prime}(\theta) = \frac{\partial}{\partial \theta} \L_{\theta^\prime}(\theta) \in \R^d, \; \; \nabla^2 \L_{\theta^\prime}(\theta) = \frac{\partial^2}{\partial \theta^2} \L_{\theta^\prime}(\theta) \in \R^{d \times d}.
\end{align*}
Although this definition of loss may seem complex since it corresponds to two parameters, it remains the desirable property that,  if $\theta^*$ is a minimizer of (\ref{eq:oracle}), then it is also a minimizer of (\ref{eq:custom-loss}), that is, $\nabla \L_{\theta^\prime}(\theta^*)=0$.
We also note that for quantile regression in Example \ref{eg:quant-reg}, even though the individual objective $\ell(\theta; X,Y)$ is non-smooth, the population objective $\L_{\theta'}(\theta)$ is second-order differentiable if one assumes some mild regularity conditions on the error distribution. 
Finally, we denote $\xi_{\theta^\prime}(\theta; \zeta)$ as the difference between the stochastic gradient and population gradient of the loss defined in \eqref{eq:custom-loss}, i.e.,
\begin{align} \label{eq:xi}
\xi_{\theta^\prime}(\theta; \zeta) = w(\theta^\prime; X, A)\nabla \ell(\theta; \zeta) - \nabla \L_{\theta^\prime}(\theta),
\end{align}
By definition, we can easily verify that $w(\theta^\prime; X, A)\nabla \ell(\theta; \zeta)$ is an unbiased estimator of $\nabla \L_{\theta^\prime}(\theta)$, which implies $\E[\xi_{\theta^\prime}(\theta; \zeta)] = 0$. Note that our framework allows general  $w(\theta';X,A)$, while in the work of \cite{chen2021b}, the loss function is defined as
\begin{align} \label{eq:song-loss}
\Tilde{\L}(\theta) = \E_{\P_X} \left[ \E_{\pi_\mathrm{equal}}\E_{\P_{Y\mid X,A}} \left(\ell(\theta; X, A, Y) \mid X,A \right) \right],
\end{align}
where $A \sim \pi_\mathrm{equal}$, which means each arm in $\A$ contributes equally to the population objective. To match the SGD update with the loss function $\tilde{\L}(\cdot)$, they choose the weight $w_t$ to be specifically in the IPW form such that $w_t$ is proportional to $\frac{\pi_\mathrm{equal}}{\pi(X, \theta)}$.  
This weighting scheme corrects the importance of each arm towards a discrete uniform distribution $\pi_\mathrm{equal}$, instead of its own sampling distribution $\pi(X, \theta)$.  However, this definition is limited to such a specific weighting scheme and the resulting asymptotic covariance matrix could be extremely large as Remark~\ref{remark:discuss on gamma} shows. 

{Our theoretical analysis relies heavily on our definition of this population loss function $\L_{\theta^\prime}(\theta)$ in \eqref{eq:custom-loss}. By expressing the loss using two different variables $\theta$ and $\theta^\prime$, we separate the loss $\ell(\theta; \zeta)$ from the policy $\pi(X, \theta^\prime)$ and the weight $w(\theta^\prime; X, A)$, as we have a focus on the local geometry of $\L_{\theta^\prime}(\theta)$ with respect to $\theta$ instead of the geometry with respect to $\theta'$. It is worthwhile noting that that $\theta^*$ is a minimizer of $\L_{\theta^\prime}(\theta)$ regardless of what $\theta^\prime$ is. In Remark \ref{rmk:smooth-clt}, we demonstrate this property in a special case. }

\subsection{Asymptotic normality}
We first introduce some regularity assumptions on the population loss function $\L_{\theta^\prime}(\theta)$, the individual loss function $\ell(\theta; \zeta)$, and the gradient weight $w(\theta^{\prime}; X, A)$.

\begin{assumption} \label{assum:bound}
There exists some constants $\underline{w}, \overline{w}$, such that $0 < \underline{w}< w_t < \overline{w}$ for all $t \geq 1$.
\end{assumption}

\begin{assumption} \label{assum:loss}
The loss function $\L_{\theta^\prime}(\theta)$ is convex with respect to $\theta \in \R^d$, continuously differentiable with respect to $\theta \in \R^d $, and twice continuously differentiable  at $\theta^*$. Moreover, there exists some constants $\delta, \mu > 0$, such that $\langle \nabla \L_{\theta}(\theta), \theta - \theta^* \rangle > 0,~\forall\theta \neq \theta^*$ and
\begin{align*}
\langle \nabla \L_{\theta}(\theta), \theta - \theta^* \rangle &\geq \mu \|\theta - \theta^*\|^2, \; \; \forall \,\theta \in \{\theta: \|\theta - \theta^*\|\leq \delta\}.
\end{align*}
\end{assumption}

\begin{assumption} \label{assum:hessian}
The Hessian matrix $\nabla^2 \L_{\theta^\prime}(\theta) \in \R^{d \times d}$ exists for all $(\theta, \theta^\prime) \in \R^d \times \R^d$ and the Hessian matrix at $(\theta^*, \theta^*)$ is positive definite, i.e., $H \triangleq \nabla^2 \L_{\theta^*}(\theta^*) \succ 0$. Moreover, for large enough $t$, there exists some constant $K>0$, such that
\begin{align}\label{eq:hessian lipschitz}
\left\| \nabla^2 \L_{\theta_{t-1}}(\theta) - \nabla^2 \L_{\theta^*}(\theta^*)\right\| \leq K \| \theta - \theta^* \| + K \|\theta_{t-1} - \theta^* \|,
\end{align}
for all $\|\theta-\theta^*\|\leq\delta$,
where $\theta_{t-1}$ is recursively updated through equation \eqref{eq:weighted-sgd}.
\end{assumption}

\begin{assumption} \label{assum:gram}
For any action $A \in \A$ and covariate $X$, we further assume $\|\nabla \ell(\theta; \zeta)\|^2$ exists almost surely under $\P_{Y\mid X,A}$, and $
\E \left(\|\nabla \ell(\theta; \zeta)\|^2 \mid X, A \right) \leq \phi(X)(1+\|\theta - \theta^*\|^2 ),$
for some function $\phi(\cdot)$ such that $\E [\phi(X)] <\infty$. We also assume the Gram matrix of $\xi_{\theta^\prime}(\theta; \zeta)$ at $(\theta^*; \theta^*)$, $S \triangleq \E[\xi_{\theta^*}(\theta^*; \zeta^*) \xi_{\theta^*}(\theta^*; \zeta^*)^\top]$, exists, where $\zeta^*=(X,A^*,Y(A^*))$ and $A^*\sim\pi(X,\theta^*)$. 
\end{assumption}

\begin{assumption} \label{assum:tv}
Let $\Delta(X, \theta) = \d_{\rm TV}(\pi(X, \theta), \pi(X, \theta^*))$ be the total variation distance of $\pi(X, \theta)$ and $\pi(X, \theta^*)$. For function $\phi(X)$ defined above, $\lim_{\theta \rightarrow \theta^*} \E[\Delta(X, \theta)\phi(X)] = 0$,
\begin{align*}
&\lim_{\theta \rightarrow \theta^*} \E\left[\|\nabla \ell(\theta; \zeta) - \nabla \ell(\theta^*; \zeta)\|^2 \mid X,A \right] = 0,~\lim_{\theta \rightarrow \theta^*} \E\left[|w(\theta; X, A) - w(\theta^*; X, A)|^2 \phi(X) \mid A \right] = 0.
\end{align*}
\end{assumption}

{ Assumption~\ref{assum:bound} is a common assumption on the weights applied to the stochastic gradient, which is used in many adaptive setting literature, e.g., \cite{chen2021a}, \cite{chen2021b}, and \cite{zhang2021statistical}. Specifically, it requires the arm selection probabilities to be  bounded away from zero. }
 The convexity and continuity on the population loss $\L$ in Assumption~\ref{assum:loss} is a standard requirement in classical SGD literature \citep{polyak1992acceleration,chen2016statistical,chen2021b,duchi2021asymptotic}. We can also find similar arguments in the SGD literature mentioned above for Assumption~\ref{assum:loss} to Assumption~\ref{assum:gram}, whereas we generalize the previous assumptions on our loss function $\L_{\theta}(\theta)$ with an extra variable $\theta^\prime$. {Specifically, for Assumption \ref{assum:hessian}, instead of requiring the Lipschitz property of $\nabla^2\L_{\theta^\prime}(\theta)$ for both $\theta$ and $\theta^\prime$ within the  neighborhood of $\theta^*$, we require this property holds only with respect to $\theta^\prime=\theta_{t-1}$. 
}Assumption~\ref{assum:tv} further regularizes the function $\phi(\cdot)$ defined in Assumption~\ref{assum:gram}. { Later, we verify our assumptions for linear and quantile regression examples under the modified $\varepsilon$-greedy and exponential policies (Section~\ref{sec:var-discuss}).
} It is noteworthy to mention that, in Assumption~\ref{assum:gram} and Assumption~\ref{assum:tv}, we only implicitly assume $\nabla \ell$ exists almost surely under $\P_{Y \mid X,A}$. Therefore, our assumption is not restricted to smooth loss function $\ell$, it also covers many non-smooth statistical problems like quantile regression. {
These assumptions can be categorized into those that constrain the data collection algorithm and those that pertain to the data-generating environment. Specifically, Assumption \ref{assum:bound} pertains to the data collection algorithm, while Assumptions \ref{assum:gram}--\ref{assum:tv} relate to the data-generating environment. Assumptions \ref{assum:loss}--\ref{assum:hessian}, however, jointly constrain both. }

We now state our main result that characterizes the limiting distribution of the averaged weighted SGD iterates defined in \eqref{eq:weighted-sgd} under general models.
\begin{theorem} \label{thm:smooth-clt} Under Assumption~\ref{assum:bound} to Assumption~\ref{assum:tv}, {the averaged SGD estimator $\bar{\theta}_t=t^{-1} \sum_{s=0}^{t-1} \theta_s$} converges to $\theta^*$ almost surely when $t \rightarrow \infty$ and 
\begin{align*}
\sqrt{t}(\bar{\theta}_t - \theta^*) \overset{d}{\rightarrow} \mathcal{N}(0, H^{-1}SH^{-1}),
\end{align*}
{where $\theta_s$ is updated in \eqref{eq:weighted-sgd} with step size $\eta_t = \eta_0 t^{-\alpha}$, $\eta_0>0$ and $\alpha\in(1/2,1)$, $H=\nabla^2 \L_{\theta^*}(\theta^*)$ and $S = \E[\xi_{\theta^*}(\theta^*; \zeta^*) \xi_{\theta^*}(\theta^*; \zeta^*)^\top]$. }
\end{theorem}

We relegate the proof to Section~C of the supplement. To emphasize the technical challenge in the theoretical analysis, our loss function $\L$ in \eqref{eq:custom-loss} is not defined by the stable policy as in the prior works \citep{chen2021b}. 
The action $A_{t} \sim \pi(X_t, \theta_{t-1})$ and $A_{t}^* \sim \pi(X_t, \theta^*)$ are no longer in the same probability space, and therefore we specify a coupling between $A_t$ and $A_t^*$ to compare them. A natural choice is the coupling such that
\begin{align}
\label{eq:coupling}
\Delta(X_t, \theta_{t-1}) = d_{\rm TV}(\pi(X_t, \theta_{t-1}), \pi(X_t, \theta^*)) = \frac{1}{2}\sum_{a\in\A}|p^t_a - q^t_a| = \Pr(A_t \neq A_t^*),
\end{align}
where $p^t_a = \Pr(A_t=a), q^t_a = \Pr(A_t^*=a)$, $a\in\A$.

{As demonstrated in Theorem \ref{thm:smooth-clt}, the limiting distribution remains the same across different specifications of the decaying step size sequence $\eta_t$ with $\alpha\in(1/2,1)$. However, $\eta_t$ influences how closely the distribution of $\bar{\theta}_t$ aligns with the limiting Gaussian distribution, as discussed in the next section.}

\subsection{Bahadur representations} \label{sec:finite-sample}

In this section, we further present the Bahadur representation of our weighted SGD under the adaptive data collection environment. Aside from the asymptotic normality result in Theorem \ref{thm:smooth-clt}, the Bahadur representation characterizes the remainder term beyond the normal approximation, which helps conduct a finer convergence analysis of the proposed estimator. The Bahadur representation was first studied in \cite{bahadur1966note} for quantile regression, and generalized to $M$-estimators by \cite{carroll1978almost, he1996general} and many others. 
For the SGD estimator under classical non-adaptive settings \eqref{eq:vanilla-sgd}, the Bahadur representation can be inferred by the proof of \cite[Theorem 2]{polyak1992acceleration} as,
\begin{align}\label{eq:bahadur-rep-classical}
\sqrt{t} \Sigma^{-1/2} (\bar{\theta}_t^{\sgd} - \theta^*) = W + \mathcal{O}_p\big(t^{-\alpha + \frac{1}{2}} + t^{-\frac{\alpha}{2}} + t^{\alpha-1}\big),
\end{align}
where $\Sigma = H^{\sgd-1} S^{\sgd} H^{\sgd -1}$, and $W$ is the leading term as a sum of independent variables that converges to a standard normal distribution as $t \rightarrow \infty$. The other term on the right-hand side is a higher-order remainder term that converges faster than the leading term $W$ under common regularity conditions. In the following theorem, we provide the Bahadur representation of the proposed weighted SGD  \eqref{eq:weighted-sgd} under adaptive settings.

{\begin{theorem} \label{thm:ls-fs}
For any policy and weighting scheme satisfying the conditions in Theorem \ref{thm:smooth-clt}, 
we further assume $\E \left(\|\nabla \ell(\theta_{t-1}; \zeta)-\nabla \ell(\theta^*; \zeta)\|^2 \mid X, A \right) \leq C\|\theta_{t-1} - \theta^*\|^2$, 
and 
\begin{enumerate}[(a)]
  \item \label{a}Given $\theta^*$, 
  the following inequalities hold for some constant $\beta_1,\beta_2>0$,
$\E\left[ \Delta(X,\theta_{t-1})\phi(X) \mid \right] \leq C t^{-\beta_1}$, $\E\left[|w(\theta_{t-1}; X, A) - w(\theta^*; X, A)|^2 \phi(X) \mid A \right] \leq C t^{-\beta_2}$, 
where $\phi$ is defined in Assumption~\ref{assum:gram};
\item \label{b}For any action $A \in \A$ and covariate $X$, assume $\|\nabla \ell(\theta; \zeta)\|^4$ exists almost surely under $\P_{Y\mid X,A}$ and $\E \left(\|\nabla \ell(\theta; \zeta)\|^4 \mid X, A \right) \leq C(1+\|\theta - \theta^*\|^4 )$ for some positive constant $C$. 
\end{enumerate}
We have for $\bar{\theta}_t$ is identically defined as in Theorem \ref{thm:smooth-clt}, 
\begin{align} \label{eq:bahadur-rep}
\sqrt{t}\Sigma^{-1/2}(\bar{\theta}_t - \theta^*) &= \underbrace{\frac{1}{\sqrt{t}} \sum_{i=1}^{t-1} \Sigma_t^{-1/2} Q_i^t \xi_{\theta^*}(\theta^*; \zeta_i^*)}_{W} + \mathcal{O}_p\big(t^{-\alpha + \frac{1}{2}} + t^{-\frac{1}{2}\min\{\alpha,\beta_1,\beta_2\}} + t^{\alpha-1}\big),
\end{align}
where $\Sigma=H^{-1}SH^{-1}$, $\Sigma_t= \frac{1}{t}\sum_{i=1}^{t-1} Q_i^t S Q_i^t$, and $Q_i^t = \eta_i \sum_{j = i}^{t-1} \prod_{k = i+1}^j (I_d - \eta_k H)$ for $t>0$. For the main term on the right-hand side of \eqref{eq:bahadur-rep}, we have $\E[W]=0$, $\E[WW^\top]=I_d$. 

\end{theorem}}

We defer the proof details to Section~F of the supplement, where we decompose the remainder term into four quantities and provide their upper bounds. To derive the above decomposition, we need a certain level of continuity of the distribution of covariate $X$. 
{The assumption (a) can be verified under different policies with various combination of $\beta_1$ and $\beta_2$. Details for verifying (a) are provided in Sections~F.1 and F.2 of the supplement.} 
The assumption (b) is a fourth moment condition that strengthens Assumption~\ref{assum:gram}, enabling the use of the Martingale central limit theorem. To study the Bahadur representation, we require a generalization of the coupling we defined in \eqref{eq:coupling}. Consider the $(|\A|-1)$-simplex $S = \{(x_1, \dots, x_{|\A|}) \mid x_i \geq 0, \sum x_i = 1\}$. It has $|\A|$ vertices given by $V_i = (0, \dots, 0, 1, 0, \dots)$ where $1$ is in the $i$-th coordinate. Pick a point $P$ uniformly from $S$. For any categorical distribution with probability $(p_1, \dots, p_{|\A|})$, define $K = (p_1, \dots, p_{|\A|})$. The probability that $P$ lies in the sub-simplex with vertices $\{V_1, \dots, V_{i-1},V_{i+1} \dots , V_{|\A|}, K\}$ ($V_i$ is deleted) is exactly $p_i$. Thus, $K$ gives a partition of $S$ that has the required categorical distribution and we can use this to define the action $A$. Furthermore, given two different distributions $K, K^\prime$, the quantity $\P(A \neq A^\prime)$ is bounded by $C d_{\rm TV}(K, K^\prime)$, where $C$ is some constant which only depends on $|\A|$.

\begin{remark}
{The decay rate of the remainder term in \eqref{eq:bahadur-rep} explicitly demonstrate the effect of step size $\eta_t$ on how fast the limiting distribution of $\bar{\theta}_t$ converges to its limiting normal distribution. }Given the Bahadur representation of $\bar{\theta}_t$, we now emphasize the difference in the convergence rate of the adaptive SGD and the classical SGD results. 
{This remainder rate \eqref{eq:bahadur-rep} generally exhibits a slower rate compared to the i.i.d.~settings \eqref{eq:bahadur-rep-classical}, due to a critical term $t^{-\frac12\min\{\alpha,\beta_1,\beta_2\}}$. However, since $\beta_1$ and $\beta_2$ can be regarded as arbitrarily large in i.i.d.~settings, \eqref{eq:bahadur-rep} effectively generalizes \eqref{eq:bahadur-rep-classical}, recovering the remainder of the classical SGD in non-adaptive environments.} 
In Section F.1 of the supplement, we further establish a lower bound. 
In practice, we are confronted with the challenge that the underlying distribution for $(X,A,Y)$ is unknown. Therefore, we must rely on policies such as $\varepsilon$-greedy to learn from past observations, which inevitably leads to a slower rate of convergence in adaptive settings. This phenomenon is not unique to the $\varepsilon$-greedy policy and also arises under other policies, including exponential policies. In the next section, we study the remainder rates of the modified $\varepsilon$-greedy policy \eqref{eq:modified eps-greedy} and the exponential policy \eqref{eq:exp3} in linear regression, as established in Corollaries~\ref{corr:bahadur eps} and~\ref{corr:bahadur exp3}.

\end{remark}
\subsection{Online statistical inference} \label{sec:online-inf}
To provide statistical inference for the model parameter, we need to estimate the variance of $\hat\theta_t$, which is $H^{-1}SH^{-1}$, as we established in Theorem \ref{thm:smooth-clt}, in a fully online fashion. A few options have been provided from SGD inference literature, e.g., the plug-in estimator  \citep{chen2016statistical, chen2021b}, the batch-means estimator  \citep{chen2016statistical,zhu2021online}, the bootstrap estimator  \citep{fang2018online}. Among the above, the plug-in estimator is expected to achieve a very good numerical behavior as evident from classical SGD approaches. In this paper, we use the plug-in estimator for smooth loss functions $\ell$, and leave the other methods as interesting future work. In adaptive settings, the  online plug-in estimators for $S$ and $H$ are given by, 
\[
\hat S_n=\frac1n\sum_{t=1}^nw_t^2 \nabla \ell(\theta_{t-1}; \zeta_t)\nabla \ell(\theta_{t-1}; \zeta_t)^\top,\quad \hat H_n=\frac1n\sum_{t=1}^nw_t\nabla^2 \ell(\theta_{t-1}; \zeta_t).
\]
With the plug-in estimators $(\hat S_t, \hat H_t)$, an online plug-in inference procedure can be provided by replacing $S$ and $H$ in the asymptotic covariance matrix in Theorem \ref{thm:smooth-clt} with $(\hat S_t, \hat H_t)$.
In this section, we demonstrate the online plug-in inference procedure based on the limiting distribution of our proposed estimator $\bar{\theta}_t$ in Theorem \ref{thm:smooth-clt}. 

We establish the consistency of the plug-in estimator under the following Assumption~\ref{assum:plugin}, 
{which is simply a repetition of Assumption~\ref{assum:gram} and Assumption~\ref{assum:tv} with $\phi$ replaced by $\psi$ and with gradient replaced by Hessian.}
The proof is presented in Section~H of the supplement.

\begin{assumption}\label{assum:plugin}
For any action $A \in \A$ and covariate $X$, we assume that $\nabla^2 \ell(\theta; \zeta)$ exists and $\E \left(\|\nabla^2 \ell(\theta; \zeta)\|^2 \mid X, A \right)$ is bounded by $\psi(X)(1+\|\theta - \theta^*\|^2 )$ 
for some function $\psi(\cdot)$ such that $\E [\psi(X)] < \infty$. In addition, we have $\lim_{\theta \rightarrow \theta^*} \E[\Delta(X, \theta)\psi(X)] = 0$ where $\Delta(X,\theta)$ is defined in Assumption \ref{assum:gram}, and 
\begin{align*}
\lim_{\theta \rightarrow \theta^*} \E\left[\|\nabla^2 \ell(\theta; \zeta) - \nabla^2 \ell(\theta^*; \zeta)\|^2 | X,A \right] = 0, ~
\lim_{\theta \rightarrow \theta^*} \E\left[|w(\theta; X, A) - w(\theta^*; X, A)|^2 \psi(X) | A \right] = 0.
\end{align*}
\end{assumption}

\begin{proposition} \label{thm:plugin}
Under Assumption~\ref{assum:bound} to Assumption~\ref{assum:plugin}, the plug-in estimators are consistent, i.e., $\hat S_n \to S$ and $\hat H_n \to H$ in probability. 

\end{proposition}
{For constructing confidence intervals, we estimate the limiting covariance matrix $H^{-1}SH^{-1}$, for which Proposition~\ref{thm:plugin} establishes the consistency of the plug-in estimator $\hat H_n^{-1}\hat S_n\hat H_n^{-1}$. To avoid possible singularity of $\hat H_n$ in finite samples, we adopt a thresholded version: let $\hat H_n=U\hat\Lambda_nU^\top$ be its eigenvalue decomposition, and define $\widetilde H_n=U\widetilde\Lambda_nU^\top$ with $\widetilde\Lambda_{n,kk}=\max\{\kappa_1,\hat\Lambda_{n,kk}\}$ for $k=1,\ldots,|\A|p$, where $\kappa_1<\mu$ as $\mu$ defined in Assumption~\ref{assum:loss}.
 By construction, $\widetilde H_n$ is positive definite and consistent. Hence, for any $c\in\mathbb R^d$, a confidence interval for $c^\top\theta^*$ is obtained by projecting $\bar\theta_t$ and $\widetilde H_t^{-1}\hat S_t\widetilde H_t^{-1}$ onto $c$, yielding an asymptotically exact interval at level $q$ with $z$-score $z_{q/2}$, as stated in the corollary below.
}

{\begin{corollary}
    Under Assumption~\ref{assum:bound} to Assumption~\ref{assum:plugin}, as $t\rightarrow\infty$,
\begin{equation*}
    \mathbb{P}\left\{c^{\top} \theta^* \in\left[c^{\top} \bar{\theta}_t-\frac{z_{q / 2}}{\sqrt{t}} \sqrt{c^{\top} \widetilde{H}_{t}^{-1} \hat{S}_{t} \widetilde{H}_{t}^{-1} c}, \ 
    c^{\top} \bar{\theta}_t+\frac{z_{q / 2}}{\sqrt{t}} \sqrt{c^{\top} \widetilde{H}_{t}^{-1} \hat{S}_{t} \widetilde{H}_{t}^{-1} c}\right]\right\} \rightarrow 1-q.
\end{equation*}
\end{corollary}}

\section{Practical Examples} \label{sec:var-discuss}

{In this section, we instantiate our general theoretical framework using two distinct policies: a \emph{modified} $\varepsilon$-greedy policy and an exponential policy. We apply these policies to the linear regression (Example \ref{eg:ls-reg}) and quantile regression (Example \ref{eg:quant-reg}) settings introduced earlier. We verify Assumptions \ref{assum:bound}--\ref{assum:tv}, and derive explicit analytic forms of the asymptotic covariance matrices under Gaussian covariates and discuss the choice of weighting schemes}. {Due to space constraints, we detail the results for linear regression in the main text and relegate the derivation and verification for quantile regression to Section E of the supplement.}
The verification of the logistic regression (Example \ref{eg:log-reg}) follows from a similar procedure.
\subsection{Modified $\varepsilon$-greedy policy}
{In this section, we present the main results under a modified $\varepsilon$-greedy policy instead of  its original version. The original $\varepsilon$-greedy policy assigns the probability of selecting an action $A_t$ to be}
\begin{align} \label{eq:eps-greedy}
\Pr(A_t = 0 \mid X_t, \theta_{t-1}) = (1-\varepsilon)\ID \left\{X_t^\top \theta_{0,t-1}> X_t^\top \theta_{1,t-1}\right\} + \frac{\varepsilon}{2},
\end{align}
for some constant $\varepsilon \in (0,1)$. Here, the $\varepsilon$ is a pre-specified constant that helps address the exploration-and-exploitation dilemma, which is often set as some small constant close to zero. 
{However, under the degenerate model where $\theta^*_{0}=\theta^*_{1}$, the $\varepsilon$-greedy policy \eqref{eq:eps-greedy} introduces a discontinuity at $\theta_{t-1}=\theta^*$. Additionally, Assumption \ref{assum:hessian} does not hold, as the Hessian $\nabla^2\L_{\theta^\prime}(\theta)$ is discontinuous in $\theta'$ near $\theta^\prime=\theta^*$. Consequently, the asymptotic normality results and the inference procedure are invalidated. This discontinuity arises because, even though $\theta_{t-1}$ may converge to $\theta^*$, under \eqref{eq:eps-greedy}, $\theta'=\theta_{t-1}\neq \theta^*$ corresponds to the non-degenerate objective $\L_{\theta^\prime}(\theta)$ while $\theta'=\theta^*$ corresponds to the degenerate one. As a result, their respective landscapes and asymptotic properties differ by nature.

Many related work has discussed the inference problems in this degenerate model. For example, \cite{zhang2021statistical} discusses how the quality of the Gaussian approximation degrades as the true data generating
process gets closer to the “degenerate” model.
\cite{luedtke2016statistical} addresses the challenge posed by ``exceptional laws'', where treatment effects are either zero or the model is non-unique, complicating the construction of pathwise differentiable estimators and valid inference in such degenerate settings.

The aforementioned challenges necessitate a modified $\varepsilon$-greedy policy applicable to both degenerate and non-degenerate models. Specifically, this policy is designed to facilitate an asymptotic transition to the degenerate regime when the model parameters are identical. Drawing inspiration from the Hodges estimator, we propose the following modified policy
\begin{align} \label{eq:modified eps-greedy}
\Tilde{\Pr}(A_t = 0 \mid X_t, \theta_{t-1}) 
&= \Pr(A_t = 0 \mid X_t, \theta_{t-1})\cdot\ID\left\{\|\theta_{0,t-1}-\theta_{1,t-1}\|>t^{-\frac{\alpha}{4}}\right\}\nonumber\\
\quad&+\frac{1}{2}\cdot\ID\left\{\|\theta_{0,t-1}-\theta_{1,t-1}\|\leq t^{-\frac{\alpha}{4}}\right\},
\end{align}
where $\Pr(A_t = 0 \mid X_t, \theta_{t-1})$ corresponds to the original $\varepsilon$-greedy policy defined in \eqref{eq:eps-greedy}.
This modified policy employs a thresholding rule on $\|\theta_{0,t-1}-\theta_{1,t-1}\|$ to distinguish between the degenerate and non-degenerate regimes. A formal theoretical demonstration is provided in Lemmas J.3 and J.4 of the supplement.
}

This setting can be relaxed to a deterministic sequence $\{\varepsilon_t\}$ which converges to some constant $\varepsilon_\infty \in (0,1)$, and we defer the technical details of $\varepsilon_t$ to Section G in the supplement. 
In the current work, we focus on policies that depend only on $X_t,\theta_{t-1}$ for simplicity. It may be relaxed to $A_t \sim \pi \left(X_{t}, \Phi_{t-1}\right)$ for other statistics $\Phi_{t-1}$ relying on the history $\theta_0, \theta_1, \cdots, \theta_{t-1}$, e.g., the running average of the $\{\theta_{s}\}_{s=0}^{t-1}$, which we leave for future works.

We now use the linear regression model in Example \ref{eg:ls-reg} with random design as a special case of our main result that has been presented in Theorem \ref{thm:smooth-clt}. We specify $w_t$ as a pre-specified function of $\Pr(A_t \mid X_t, \theta_{t-1})$, i.e., $w_t(\theta_{t-1};X_t, A_t) = \varphi(\Pr(A_t \mid X_t, \theta_{t-1}))$. The following Theorem \ref{prop:ls-reg} provides a new way to determine further the optimal weighting scheme to minimize the asymptotic variance of the weighted ASGD estimator.

To further illustrate our assumptions and central limit theorem result in Theorem~\ref{thm:smooth-clt}, we validate them under two examples we mentioned above, i.e., linear regression (Example~\ref{eg:ls-reg}) and quantile regression (Example~\ref{eg:quant-reg}), with the modified $\varepsilon$-greedy in \eqref{eq:modified eps-greedy}. As a result, Theorem~\ref{thm:smooth-clt} holds for these two examples. 
{Here we only demonstrate the results for linear regression and relegate the results for quantile regression and their verification to Section E of the supplement.} 

In Corollary~\ref{corr:ls-reg} below, we demonstrate that Assumptions~\ref{assum:bound}--\ref{assum:tv} are quite natural and can be satisfied by the linear regression example. Before this,
in order to describe the decaying rate of a probability density function and its (sub)gradients, we use the definition of \emph{rapidly decreasing} functions in the above corollary, which is also known as Schwartz functions. The definition captures the properties of the functions whose derivatives of any order decrease faster than any reciprocal power of $x$ as $x$ tends to infinity. Detailed definitions and discussions are relegated to Section~J of the supplement. 
 {The following corollary shows the asymptotic normality of Theorem \ref{thm:smooth-clt} can be applied to linear regression under the modified $\varepsilon$-greedy policy. The similar results for quantile regression are demonstrated in Corollary~E.1 of the supplement.}

\begin{corollary} \label{corr:ls-reg}
{Using the modified $\varepsilon$-greedy policy in \eqref{eq:modified eps-greedy}, for the linear regression example we used in Example~\ref{eg:ls-reg}, assume that the covariate $X$ is sub-Gaussian} and $\E [X X^\top] \succ 0$. Further assume that the probability density function of $X$, $p(x)$, is smooth and rapidly decreasing, and the weight $
  w_t(\theta_{t-1};X_t, A_t) = \varphi(\Pr(A_t \mid X_t, \theta_{t-1}))$ where the function $\varphi(\cdot): (0,1) \mapsto \R^+$ is continuous. 
  Under the above conditions,
Assumptions~\ref{assum:bound}--\ref{assum:tv} are satisfied and therefore Theorem~\ref{thm:smooth-clt} holds. 
\end{corollary}\begin{remark}
{We emphasize that, under the modified $\varepsilon$-greedy policy, the asymptotic normality established in Theorem \ref{thm:smooth-clt} is pointwise but not uniform across all underlying distributions $\P \in \mathbf{P}$. This limitation is a consequence of the superefficiency inherent in the construction of the Hodges estimator. Specifically, there exists no universal threshold $t_0$ such that the approximation error of $\sqrt{t}(\overline{\theta}_{t}-\theta^{*})$ to its limiting distribution remains uniformly bounded for all $t > t_0$ across the entire parameter space. We provide a detailed discussion and illustration of this phenomenon in Section I.1 of the supplementary material.
For a contrast, we refer readers to Remark \ref{rem:eps}, where we demonstrate that the exponential policy admits uniform asymptotic normality. 
}
\end{remark}

More specifically, when $X$ is Gaussian, we can derive  an explicit analytic forms of the Hessian matrix $H$ and Gram matrix $S$ that match their definitions in Theorem~\ref{thm:smooth-clt}.

\begin{theorem} \label{prop:ls-reg}
In the linear regression Example~\ref{eg:ls-reg} with the modified $\varepsilon$-greedy policy in \eqref{eq:modified eps-greedy}, assume that $\varphi(\cdot): (0,1) \mapsto \R^+$ is continuous, and $X_t \sim \N(\mu, I_p)$. {The ASGD estimator $\bar{\theta}_t=t^{-1} \sum_{s=0}^{t-1} \theta_s$ converges to $\theta^*$ almost surely}
and, as $ t\rightarrow\infty$, 
\begin{align*}
&\sqrt{t}(\bar{\theta}_t - \theta^*) \overset{d}{\rightarrow} \mathcal{N}(0, H^{-1}SH^{-1}),\;
\text{where }&S& = \begin{bmatrix}
S_0& 0 \\
0 & S_1
\end{bmatrix},~H = \begin{bmatrix}
H_0& 0 \\
0 & H_1
\end{bmatrix}.
\end{align*} 
\begin{itemize}
    \begin{small}
    \item Non-degenerate model:
        \begin{align*}
        S_0 &= \sigma^2\left(\Big(1-\frac{\varepsilon}{2}\Big) \varphi^2\Big(1-\frac{\varepsilon}{2}\Big) G_0^* + \frac{\varepsilon}{2} \varphi^2\Big(\frac{\varepsilon}{2}\Big)  G_1^* \right), &S_1& = \sigma^2\left(\frac{\varepsilon}{2} \varphi^2\Big(\frac{\varepsilon}{2}\Big) G_0^* + \Big(1-\frac{\varepsilon}{2}\Big) \varphi^2\Big(1-\frac{\varepsilon}{2}\Big) G_1^*\right), \\
H_0 &= \Big(1-\frac{\varepsilon}{2}\Big) \varphi\Big(1-\frac{\varepsilon}{2}\Big) G_0^* + \frac{\varepsilon}{2} \varphi\Big(\frac{\varepsilon}{2}\Big)  G_1^* , \; \; &H_1& = \frac{\varepsilon}{2} \varphi\Big(\frac{\varepsilon}{2}\Big) G_0^* + \Big(1-\frac{\varepsilon}{2}\Big) \varphi\Big(1-\frac{\varepsilon}{2}\Big) G_1^*,\\
G_0^* &= {\Phi \left(a^*\right) \left(I_p+\mu\mu^\top\right) + \frac{1}{\sqrt{2\pi}}a^* e^{-\frac{a^{*2}}{2}}\nu^*\nu^{*\top}},&G_1^*&= \left(1 - \Phi \left(a^*\right)\right) \left(I_p+\mu\mu^\top\right) - \frac{1}{\sqrt{2\pi}}a^* e^{-\frac{a^{*2}}{2}}\nu^*\nu^{*\top},
    \end{align*}
    {and $\nu^*=(\theta^*_{0} - \theta^*_{1})/\|\theta^*_{0} - \theta^*_{1}\|$, $a^* = \mu^\top \nu^*$, and $\Phi$ is the cumulative distribution function of standard normal distribution. }

    \item Degenerate model: $
        S_0=S_1=\frac{\sigma^2}{2}\varphi^2\Big(\frac{1}{2}\Big)\left(I_p+\mu\mu^\top\right)$, ~$ 
        H_0=H_1=\frac{1}{2}\varphi\Big(\frac{1}{2}\Big)\left(I_p+\mu\mu^\top\right).$
    \end{small}
\end{itemize}
\end{theorem}

The proof of Theorem \ref{prop:ls-reg} is provided in Section~D.2 of the supplement, by verifying the assumptions and calculating the covariance matrices. Before we discuss its implications, we first illustrate the definition of $\L_{\theta^\prime}(\theta)$ in \eqref{eq:custom-loss} under the special case of Theorem \ref{prop:ls-reg}. 
\begin{remark}\label{rmk:smooth-clt}
{Under the stated conditions in Theorem~\ref{prop:ls-reg}}, we have for any $\varepsilon$ defined in the modified $\varepsilon$-greedy policy in (\ref{eq:modified eps-greedy}),
\begin{align*}
\L_{\theta^\prime}(\theta) &= (\theta^*-\theta)^\top {G_{\theta^\prime} }(\theta^*-\theta)+\frac{\sigma^2}{2} \left(\big(1-\frac{\varepsilon}{2}\big) \varphi\big(1-\frac{\varepsilon}{2}\big) + \frac{\varepsilon}{2} \varphi\big(\frac{\varepsilon}{2}\big)+\frac{1}{2} \varphi\big(\frac{1}{2}\big)\right),
\end{align*}
where $w_t(\theta_{t-1};X_t, A_t) = \varphi(\Pr(A_t \mid X_t, \theta_{t-1}))$,
{and $G_{\theta^\prime}$ is a positive definite matrix which is determined by $\theta^\prime$.
Therefore, we can clearly see that $\theta^*$ is a minimizer of $\L_{\theta^\prime}(\theta)$ no matter what $\theta^\prime$ is}, which is the desirable property we mentioned before. 
\end{remark}
{In light of Theorem \ref{prop:ls-reg}, we specifically consider a certain class of modified $\varepsilon$-greedy policies with weighting schemes characterized by $\varphi$, to compare the corresponding asymptotic covariance matrices in the degenerate and non-degenerate models.} 
We specify  $\varphi_{\gamma}(p) = (|\A|p)^{\gamma}$ as a class of power functions parameterized by a constant $\gamma$. This class of weights covers the following three popular weighting schemes: {\hyperref[eq:ipw-weight]{\ipw}} as $\gamma = -1$, {\hyperref[eq:sqrt-weight]{\sipw}} as $\gamma=-1/2$, and {\hyperref[eq:vanilla]{\vanilla}} as $\gamma = 0$, up to some constants. 

{For the degenerate model, we notice that 
the explicit form of covariance matrix $H^{-1}SH^{-1}$ in Theorem \ref{prop:ls-reg} is not related to the weighting scheme $\varphi$, implying the three popular weighting schemes all have the same covariance matrix. For the non-degenerate model,} 
the explicit form of the covariance matrix in Theorem \ref{prop:ls-reg} appears to be complicated at first sight, which explains why the literature mainly focuses on {\hyperref[eq:ipw-weight]{\ipw}} and {\hyperref[eq:sqrt-weight]{\sipw}} that either keeps the Hessian matrix $H$ as constant or stabilizes the estimator by keeping the Gram matrix $S$ as constant, respectively. Notably, the behavior of general $\gamma$ can be analyzed once we notice that $H_0, H_1, G_0, G_1$ all have the form $\tilde b(I+\mu\mu^\top) + \tilde c\nu^*\nu^{*\top}$ for some constants $\tilde b$ and $\tilde c$, as thus they can be simultaneously diagonalized. With details due in Section~D.3  of the supplement, we can explicitly perform the eigendecomposition of the asymptotic covariance matrix $H^{-1}SH^{-1}$. We can also show that when varying $\gamma$, 
{the eigenvectors stay fixed and each eigenvalue exhibits the following form with some $b \in (0,B)$ where $B$ is some constant,\begin{align}
g(\gamma) = \frac{(1-\varepsilon/2)^{1+2\gamma}b + (\varepsilon/2)^{1+2\gamma}(B-b)}{\big((1-\varepsilon/2)^{1+\gamma}b +(\varepsilon/2)^{1+\gamma}(B-b)\big)^2}. \label{eq:cov-poly-main}
\end{align}}
{Based on the settings of Theorem \ref{prop:ls-reg}, we have already derived the analytic closed form of the asymptotic distribution with the explicit expression of asymptotic covariance and its eigenvalue decomposition. Building on this result,}
 we will discuss the impact of different $\gamma$, especially for the three choices of weight $w_t$ we mentioned before, and the impact of different $\varepsilon$ which measures the trade-off between exploration and exploitation, respectively.

\begin{remark}[Discussion on $\gamma$ in the non-degenerate model]\label{remark:discuss on gamma}
    In practice, for the $\varepsilon$-greedy policy, one specifies $\varepsilon$ as some small constant. When $\varepsilon$ gets close to $0$, it can be inferred from \eqref{eq:cov-poly-main} that $\gamma \geq - 1/2$ leads to a finite covariance matrix; {this includes {\hyperref[eq:vanilla]{\vanilla}} as $\gamma=0$ and {\hyperref[eq:sqrt-weight]{\sipw}} as $\gamma=-1/2$ but excludes {\hyperref[eq:ipw-weight]{\ipw}} as $\gamma=-1$.  Meanwhile, $\gamma<-1/2$ leads to an infinite covariance matrix. Furthermore, the minimum of \eqref{eq:cov-poly-main} is obtained at $\gamma = 0$ for all $b \in (0,B)$. }
Therefore, under the settings in Theorem~\ref{prop:ls-reg}, {\hyperref[eq:vanilla]{\vanilla}} has an asymptotic covariance matrix that is dominated by any other asymptotic covariance matrix obtained from a power-law weighted scheme, $\varphi_{\gamma}(p) = (|\A|p)^{\gamma}$. The following Corollary \ref{prop:vanilla-cov-mat} concludes the above discussion, which is proved in Section~D.4 of the supplement. 
\end{remark}
\begin{corollary}[Optimal weights in non-degenerate linear regression] \label{prop:vanilla-cov-mat} Under the assumptions of Theorem \ref{prop:ls-reg},  the \vanilla~SGD  has the optimal asymptotic covariance matrix in the linear regression setting, i.e., $\Sigma_{\mathrm{vnl}} \preceq \tilde{\Sigma}$, 
where $\Sigma_{\mathrm{vnl}}$ is the asymptotic covariance matrix of vanilla SGD and $\tilde{\Sigma}$ is the asymptotic covariance matrix under any other weighting function $\varphi$ where $  w_t(\theta_{t-1};X_t, A_t) = \varphi(\Pr(A_t \mid X_t, \theta_{t-1}))$. 
\end{corollary}
{The above Remark \ref{remark:discuss on gamma} and Corollary \ref{prop:vanilla-cov-mat} both suggest that, when applying the modified $\varepsilon$-greedy policy in linear regression with normally distributed covariates $X$, {\hyperref[eq:vanilla]{\vanilla}} and {\hyperref[eq:sqrt-weight]{\sipw}} are preferred over {\hyperref[eq:ipw-weight]{\ipw}}. Now we can further demonstrate the Bahadur representations on linear regression under the modified $\varepsilon$-greedy.}
\begin{corollary}\label{corr:bahadur eps}
    Under the modified $\varepsilon$-greedy policy and the conditions in Corollary \ref{corr:ls-reg}, the rate of the remainder term in \eqref{eq:bahadur-rep} is upper bounded by $\mathcal{O}_p\big(t^{-\alpha + \frac{1}{2}} + t^{-\frac{\alpha}{4}} + t^{\alpha-1}\big)$, which is slower than \eqref{eq:bahadur-rep-classical} in the \emph{i.i.d.} settings.
    If we minimize the order of the rate over $\alpha\in(\frac12,1)$, we have that the optimal convergence rate of the remainder term is $\mathcal{O}_p(t^{-0.2})$ with $\alpha=0.8$. 
    {A corresponding simulation is conducted in Figure B.11 of the supplement.
    Moreover, a matching lower bound for the remainder is established in Section F.1 of the supplement.}
\end{corollary}

\subsection{Exponential  policy and other policies}
Our analysis is not restricted to an $\varepsilon$-greedy policy, but indeed encompasses general randomized policies including $\varepsilon_t$-greedy policy where $\varepsilon_t\rightarrow 0$, Thompson Sampling, exponential policies (Boltzmann exploration), etc. For $\varepsilon_t$-greedy, we defer additional discussion and technical details to Section~G of the supplement. Other than that, exponential policies employ a softmax function, and the practitioner chooses action by
\begin{equation}\label{eq:exp3}
    \Pr(A_t=a\mid X_t,\theta_{t-1})=\frac{e^{ X_t^\top\theta_{a,t-1}}}{\sum_{a^\prime\in\A}e^{ X_t^\top\theta_{a^\prime,t-1}}}.
\end{equation}
Such exponential weighting mechanism is often considered in adversarial bandit and multinomial logit bandit modeling (see, e.g., LinEXP3 in \citealp{lattimore2020bandit}). 
{We adopt a clipping $\text{clip}_{\delta}(p_a) = \max\{\delta, p_a\}$ to ensure Assumption \ref{assum:bound} holds, that is, bound the selection probability of action $a$ away from 0, where $\delta>0$ and $p_a$ refers to the arm selection probability defined in \eqref{eq:exp3}. 
In the following Corollary  \ref{corr:exp3 ls-reg}, we use this exponential policy to illustrate our main results under the linear regression (Example \ref{eg:ls-reg}) and relegate the results for quantile regression (Example \ref{eg:quant-reg}) to Section E of the supplement.
\begin{corollary}\label{corr:exp3 ls-reg}
    Under the stated conditions in Corollary~\ref{corr:ls-reg} with the exponential policy in \eqref{eq:exp3} applied to linear regression (Example~\ref{eg:ls-reg}), we further assume the weight $
  w_t(\theta_{t-1};X_t, A_t) = \varphi(\Pr(A_t \mid X_t, \theta_{t-1}))$ where the function $\varphi(\cdot): (0,1) \mapsto \R^+$ is  differentiable, and $\varphi^\prime(\epsilon)$ is bounded for $\epsilon\in[\delta_0,1)$ where $\delta_0$ is the clipping parameter applied to $\text{clip}_{\delta_0}$. 
Under the above conditions, 
Assumptions~\ref{assum:bound}--\ref{assum:tv} are satisfied and therefore Theorem~\ref{thm:smooth-clt} holds. 
\end{corollary}
\begin{remark}\label{rem:eps}
{
In contrast to $\varepsilon$-greedy strategies, the asymptotic normality established in Corollary~\ref{corr:exp3 ls-reg} is uniform over the class of distributions $\mathbf{P}$. This uniformity arises because the exponential policy ensures that the joint density of $(X, A, Y)$ is differentiable in quadratic mean, thereby satisfying the Local Asymptotic Normality (LAN) property (see Theorem 7.2 in \citealp{van2000asymptotic}). 
In contrast, $\varepsilon$-greedy type policies violate this property due to their inherent discontinuity and therefore precludes uniform asymptotic normality. 
A rigorous justification for this uniformity is provided in Section~I.2 of the supplementary material.}
\end{remark}

In addition to the modified $\varepsilon$-greedy, we have already demonstrated  the exponential policy is also a candidate policy satisfying Assumptions \ref{assum:bound}--\ref{assum:tv} in some application scenarios. The similar results for quantile regression are demonstrated in Corollary~E.2 of the supplement. Now, we will further illustrate our Bahadur representation results for this policy under the linear regression setting (Example \ref{eg:ls-reg}).

\begin{corollary}\label{corr:bahadur exp3}
    Under the exponential  policy and the conditions in Corollary \ref{corr:exp3 ls-reg}, the rate of the remainder term in \eqref{eq:bahadur-rep} is $\mathcal{O}_p\big(t^{-\alpha + \frac{1}{2}} + t^{-\frac{\alpha}{4}} + t^{\alpha-1}\big)$ 
     which holds for both the degenerate and non-degenerate models.
\end{corollary}
}
}

\section{Simulations and Real Data Analysis}\label{sec:num}
In this section, we investigate the empirical performance of the proposed estimators on normal approximation. We further construct the confidence intervals using a plug-in estimator of the asymptotic covariance matrices and report their coverage rates. Lastly, we validate the performance of the proposed estimator and inference procedure on a logistic regression of a real dataset. Due to the space limitation, 
{we will demonstrate the results for linear regression in the main text, while report the results for quantile regression and other tables and figures in Section B of the supplement.}

\subsection{Normal approximation with modified $\varepsilon$-greedy}\label{sec:normal-approximation}

We verify Theorem~\ref{thm:smooth-clt} under linear regression and quantile regression (Example~\ref{eg:ls-reg} and Example~\ref{eg:quant-reg}). For both examples, the true parameter $\theta^* \in \R^{20}$ and
\begin{align*}
Y_t = (1-A_t) X_t^\top \theta^{*}_{[1:10]} + A_t X_t^\top \theta^*_{[11:20]} + \cE_t.
\end{align*}
In the numerical experiments below, we fix the sample size as $200,000$. The covariate $X_t \sim \N(0, I_{10})$ and the noise $\{\cE_s\}_{s=1}^t$ is \emph{i.i.d.} with standard deviation $\sigma = 0.1$. We use $\varepsilon$-greedy policy~\eqref{eq:modified eps-greedy} to select actions, and set $\varepsilon = 0.02$. For the SGD update~\eqref{eq:weighted-sgd}, we specify the step sizes as $\eta_t = \eta\cdot\max(t, 300)^{-\alpha}$. As indicated in Theorem~\ref{thm:ls-fs}, we set $\alpha = 0.8$ for both linear regression and quantile regression. We compare three weighting schemes below, {\hyperref[eq:ipw-weight]{\ipw}, {\hyperref[eq:sqrt-weight]{\sipw}, {\hyperref[eq:vanilla]{\vanilla} under the modified $\varepsilon$-greedy \eqref{eq:modified eps-greedy} policy. 
\begin{figure}[!t]
    \centering
    \subfigure[\vanilla, Arm $0$]{
    \includegraphics[width=0.31\textwidth]{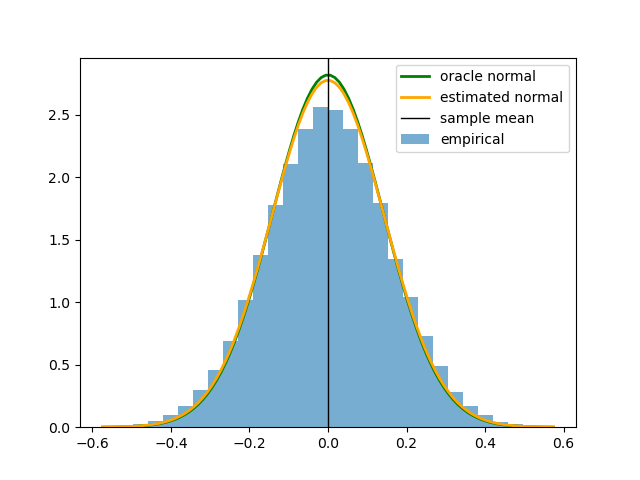}}
    \subfigure[\sipw, Arm $0$]{
    \includegraphics[width=0.31\textwidth]{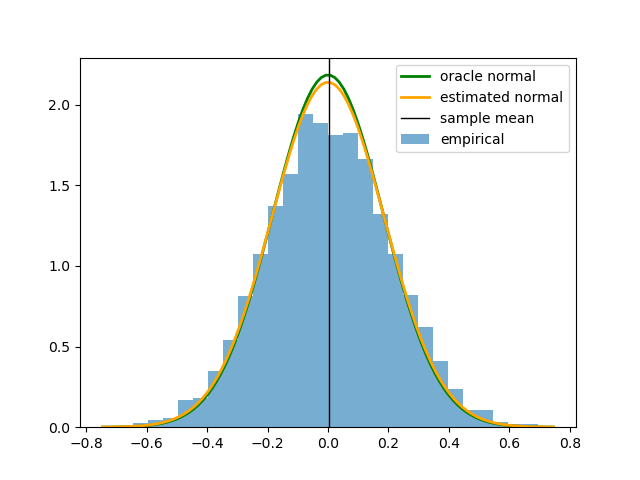}}
    \subfigure[\ipw, Arm $0$]{
    \includegraphics[width=0.31\textwidth]{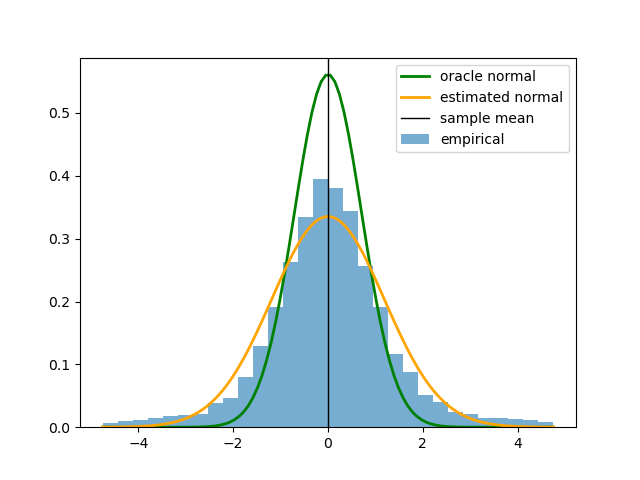}}\\
    \subfigure[\vanilla, Arm $1$]{
    \includegraphics[width=0.31\textwidth]{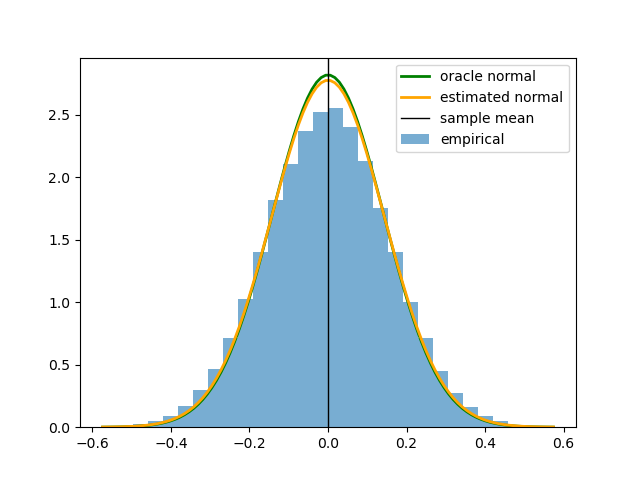}}
    \subfigure[\sipw, Arm $1$]{
    \includegraphics[width=0.31\textwidth]{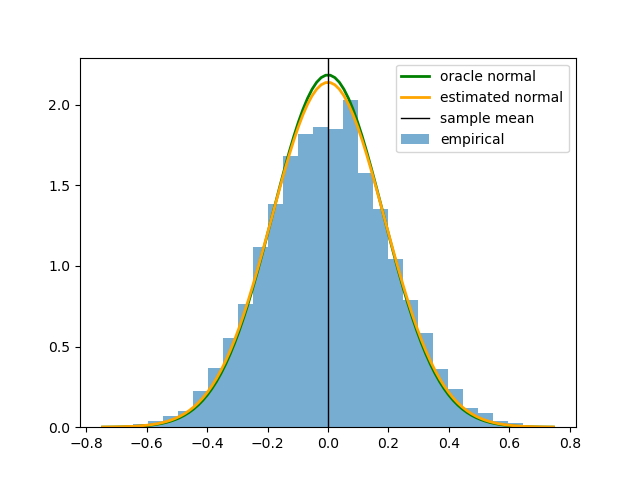}}
    \subfigure[\ipw, Arm $1$]{
    \includegraphics[width=0.31\textwidth]{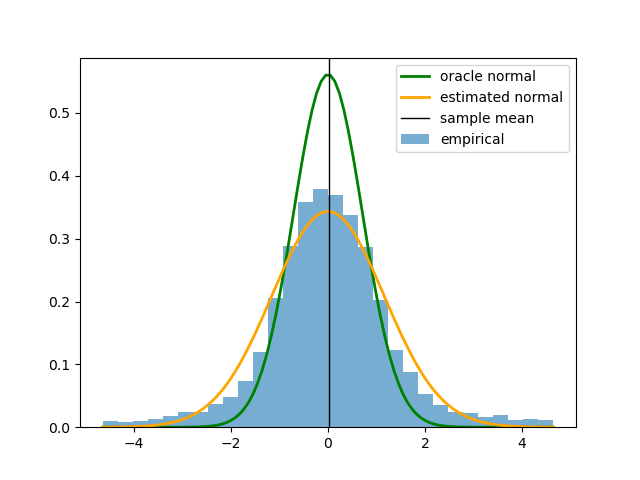}}\\
    \caption{SGD on a non-degenerate linear regression model with the modified $\varepsilon$-greedy policy using different weight schemes. We report the empirical distribution of each action's first dimension of $\sqrt{t}(\bar{\theta}_t - \theta^*)$ based on $10,000$ Monte-Carlo simulations.
    We plot the density of a zero-mean normal distribution that matches the second-order moments.}
    \label{fig-app-linear}
\end{figure}
\begin{figure}[!t]
    \centering
    \subfigure[\vanilla, Arm $0$]{
    \includegraphics[width=0.31\textwidth]{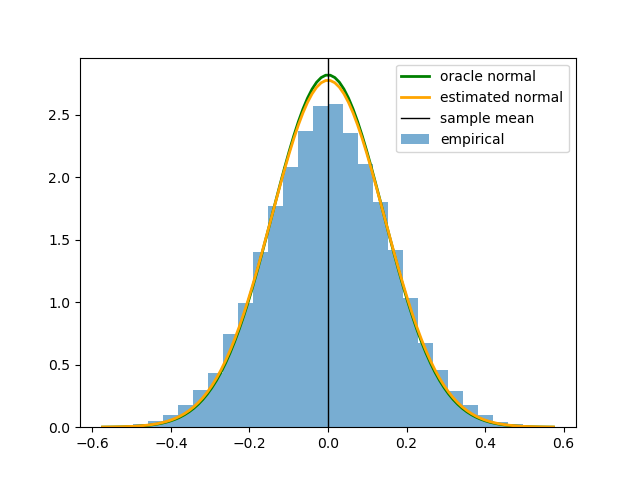}}
    \subfigure[\sipw, Arm $0$]{
    \includegraphics[width=0.31\textwidth]{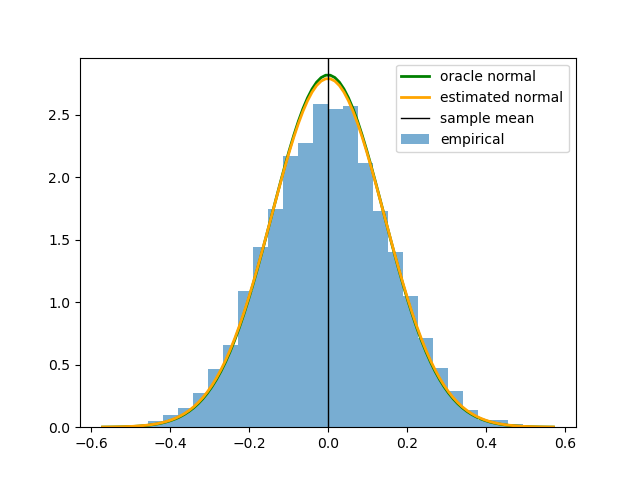}}
    \subfigure[\ipw, Arm $0$]{
    \includegraphics[width=0.31\textwidth]{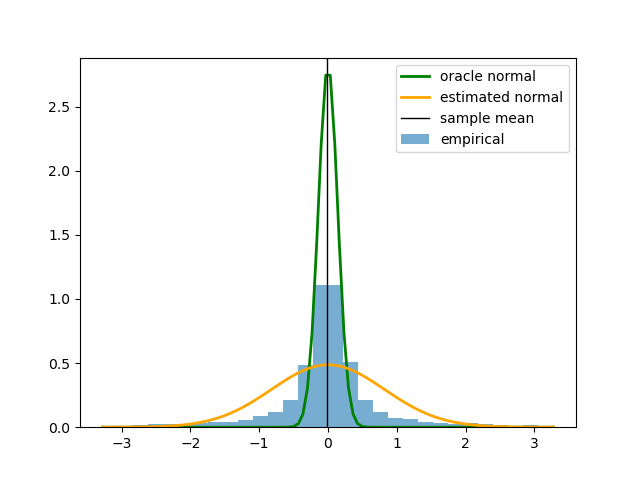}}\\
    \subfigure[\vanilla, Arm $1$]{
    \includegraphics[width=0.31\textwidth]{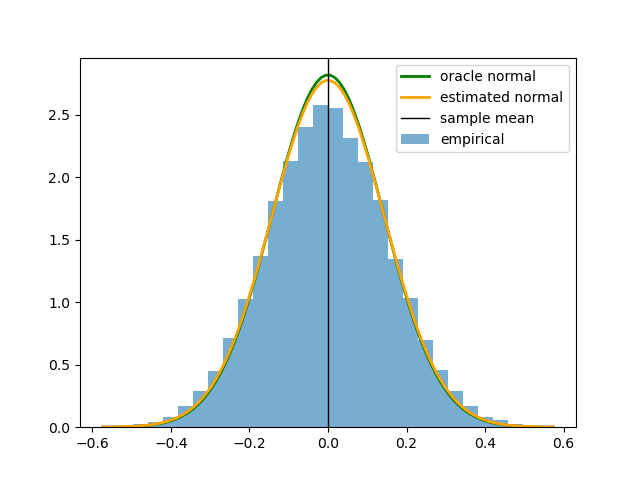}}
    \subfigure[\sipw, Arm $1$]{
    \includegraphics[width=0.31\textwidth]{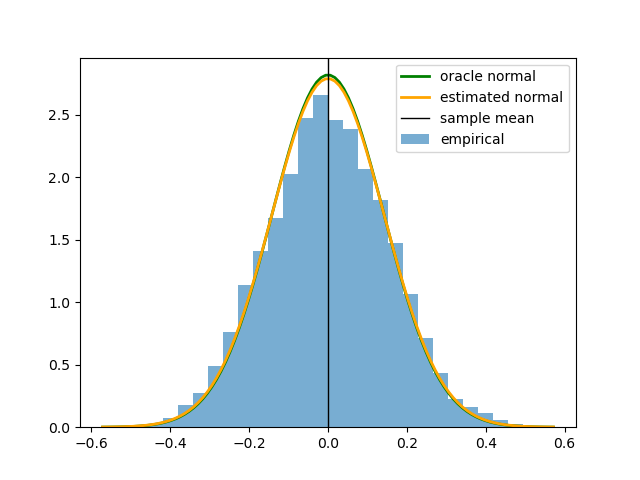}}
    \subfigure[\ipw, Arm $1$]{
    \includegraphics[width=0.31\textwidth]{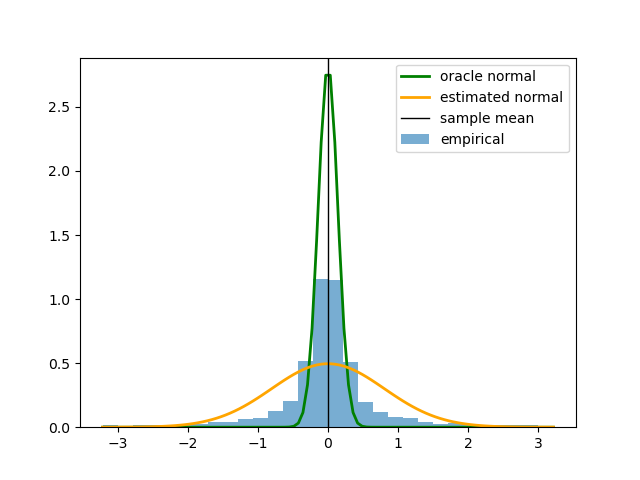}}\\
    \caption{SGD on a degenerate linear regression model with the modified $\varepsilon$-greedy policy using different weight schemes. We report the empirical distribution of each action's first dimension of $\sqrt{t}(\bar{\theta}_t - \theta^*)$ based on  $10,000$ Monte-Carlo simulations.
    We plot the density of a zero-mean normal distribution that matches the second-order moments.}
    \label{fig-app-linear1}
\end{figure}
\begin{figure}[!t]
    \centering
    \subfigure{
    \includegraphics[width=1\textwidth]{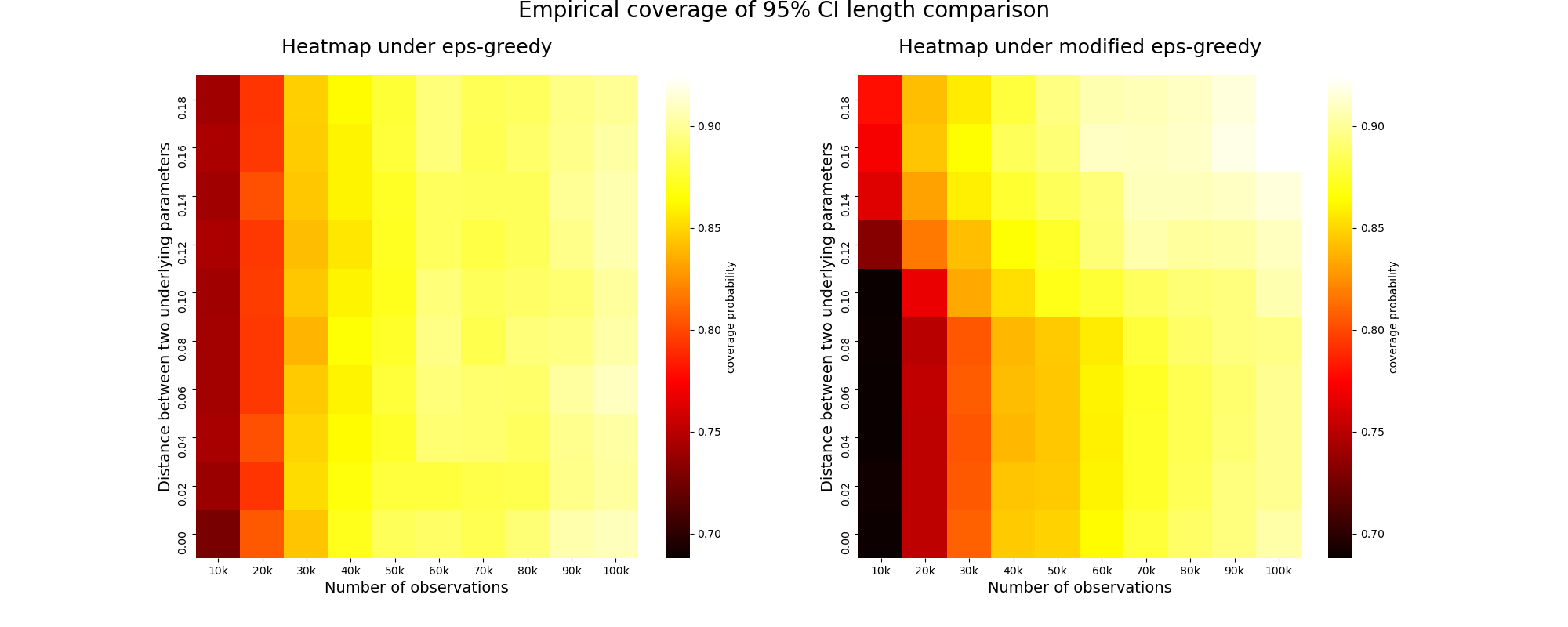}}
    \subfigure{
    \includegraphics[width=1\textwidth]{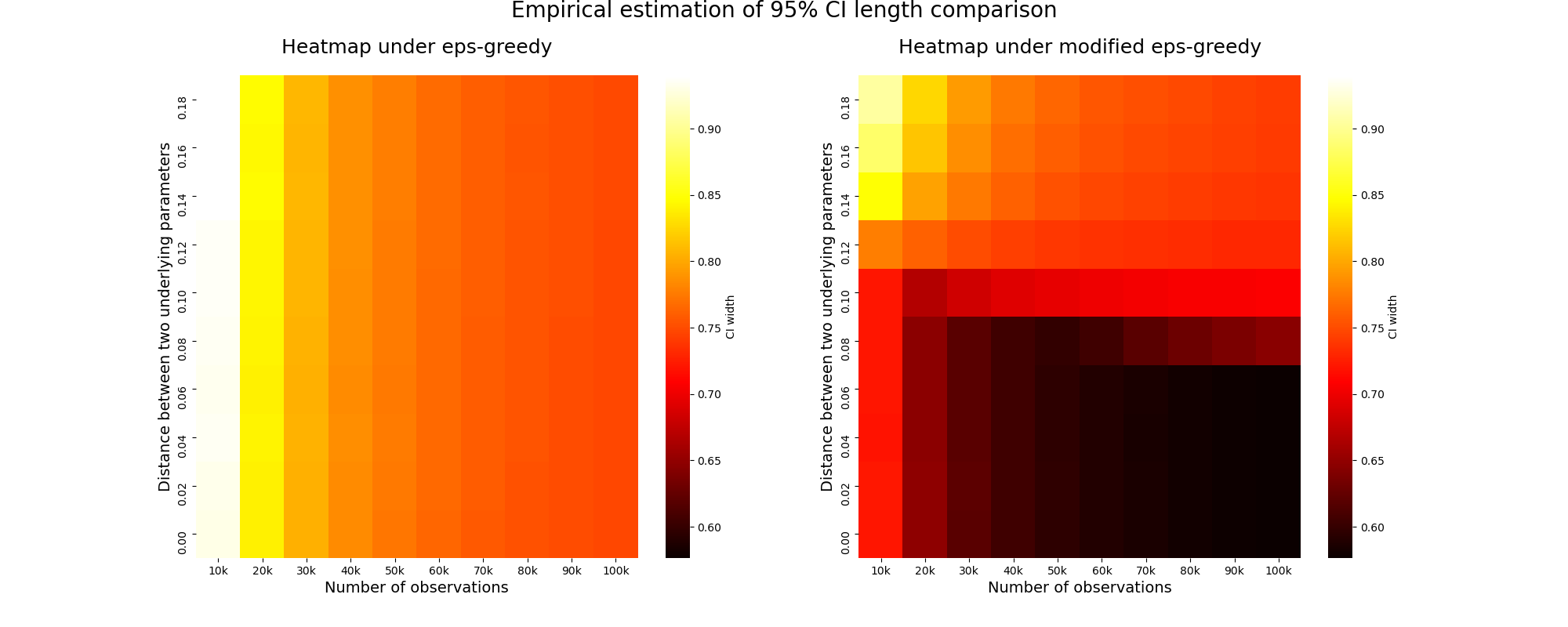}}
    \caption{SGD on linear regression with \sipw \ in the near-degenerate model. We report the empirical coverage rate and its corresponding 95\% CI length.}
    \label{fig-heat-map}
\end{figure}
\subsubsection{Non-degenerate model}
We first present the results for linear regression in the non-degenerate model. In Figure~\ref{fig-app-linear}, we plot the empirical distribution of each action's first dimension of $\sqrt{t}(\bar{\theta}_t - \theta^*)$ using $10,000$ Monte-Carlo simulations. {As can be inferred from the plots, the vanilla SGD and the sqrt-IPW SGD have much smaller standard deviation compared with \ipw~SGD, which matches our discussions in Section~\ref{sec:var-discuss}, and they also exhibit better normal approximation than \ipw. We present studentized
statistics and compare their histograms with a standard normal distribution in Section B.1 of the supplementary material.}

 Under the same setting as in Section \ref{sec:normal-approximation}, we compare the inference results for three candidate weighted-SGD schemes under 
 {non-degenerate linear regression}
 in Table~B.1 of the supplement. Both {\hyperref[eq:vanilla]{\vanilla}} and {\hyperref[eq:sqrt-weight]{\sipw}} provide a valid conference interval, while {\hyperref[eq:ipw-weight]{\ipw}} provides a much wider confidence interval than its oracle. 

{To visualize the evolution of the empirical distribution over $T$, Figure~\ref{fig:identity_transition} presents the normal approximations for $T = 50{,}000$, $100{,}000$, and $150{,}000$. As $T$ increases, the empirical histogram aligns more closely with the theoretical Gaussian density, confirming the convergence. 
The corresponding total variation distances  are further reported in Figure B.12 of the supplement, further illustrating the rate of convergence.

\begin{figure}[!t]
    \centering
    \subfigure[$T=50{,}000$]{\includegraphics[width=0.32\textwidth]{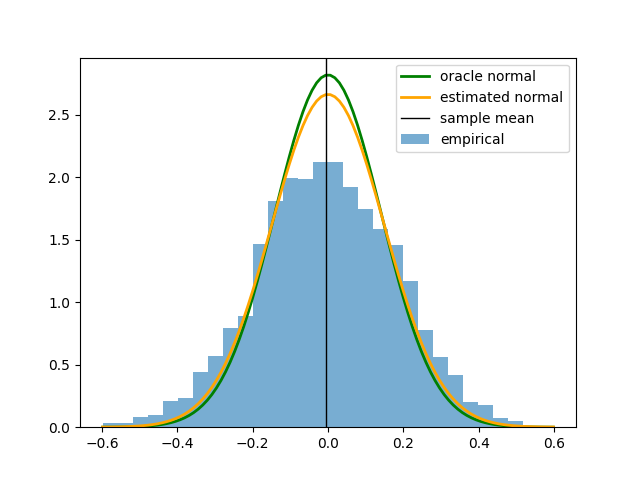}
        \label{fig:sub1}}
   \subfigure[$T=100{,}000$]{\includegraphics[width=0.32\textwidth]{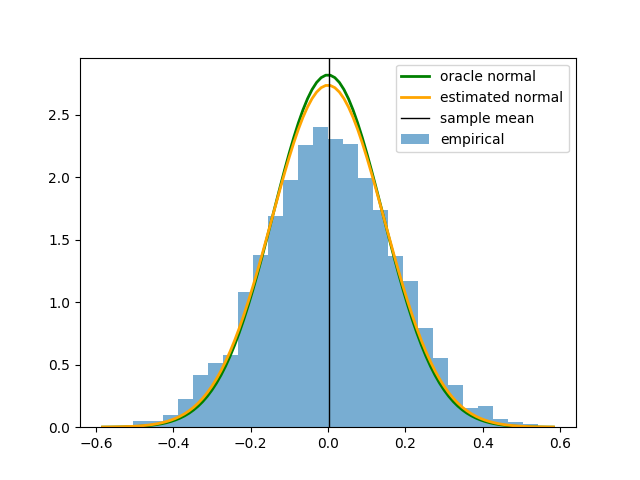}}
    \subfigure[$T=150{,}000$]{\includegraphics[width=0.32\textwidth]{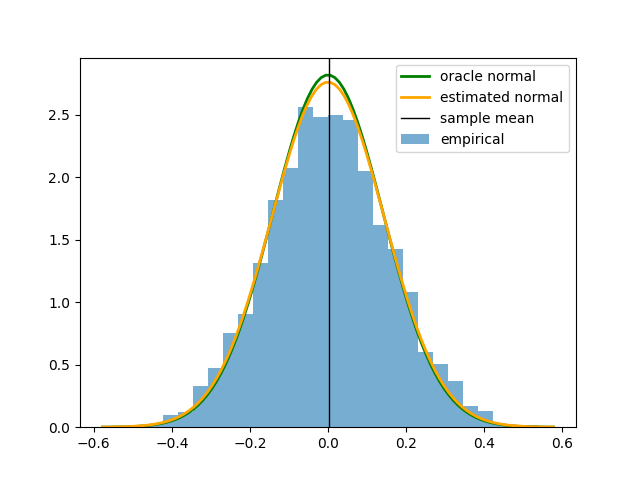}}

    \caption{Empirical distribution transitions across different $T$ values for \vanilla, Arm 0.}
    \label{fig:identity_transition}
\end{figure}
}

 \subsubsection{Degenerate model}
We now present the results for linear regression in the degenerate model. In Figure~\ref{fig-app-linear1}, we plot the empirical distribution of each action's first dimension of $\sqrt{t}(\bar{\theta}_t - \theta^*)$ using $10,000$ Monte-Carlo simulations. 
{It shows that {\hyperref[eq:ipw-weight]{\ipw}} exhibits the same issue as in the non-degenerate model, as early-stage triggering of the non-degenerate criterion under the modified $\varepsilon$-greedy policy leads to inflated variance estimates that propagate throughout the SGD process. We also present studentized statistics in Figure~B.2 of the supplementary material.
 }

Similar to Section \ref{sec:normal-approximation}, we compare the inference results for three candidate weighting schemes under degenerate linear regression models in Table~B.2. Both {\hyperref[eq:vanilla]{\vanilla}} and {\hyperref[eq:sqrt-weight]{\sipw}} provide comparable conference intervals, while {\hyperref[eq:ipw-weight]{\ipw}} provides much wider ones. 

In Figure \ref{fig-heat-map}, we compare the empirical cover rate and its corresponding 95\% confidence interval length under the classical $\varepsilon$-greedy \eqref{eq:eps-greedy} policy with its modified version \eqref{eq:modified eps-greedy}, when $\theta^*_{[1:10]}$ deviates a little from $\theta^*_{[11:20]}$ (near-degenerate model). We find that their empirical coverage rates are generally close after 70,000 SGD iterations. However, the empirical CI length of the modified $\varepsilon$-greedy policy is shorter than the other when the distance between $\theta^*_{[1:10]}$ and $\theta^*_{[11:20]}$ is less than 0.08. 
}

\subsection{Normal approximation with exponential policy}
{Under analogous settings to Section \ref{sec:normal-approximation}, we report and compare the performance of~the three weighting schemes for the exponential policy \eqref{eq:exp3} in Figures B.6--B.7  of the supplementary material for non-degenerate models and Figures B.8--B.9 for degenerate ones.}

\subsection{Real data analysis}

In this section, we apply our online estimation and inference framework to Yahoo! Today module user click-log dataset and conduct statistical inference for model parameters. We use the news recommendation and user response records on May $1^{\mathrm{st}}$, 2009. On this day, we consider the two most recommended ($405,888$ times) articles, No.109510 and No.109520 for analysis.
We follow the experiment settings in \cite{chen2021b}. The action $A_t$ is specified to be $1$ when Article No.109510 is recommended and $A_t = 0$ when Article No.109520 is recommended. The original user features have six covariates, where the first five sum up to one, and the sixth is constant $1$. In our experiments below, we keep the second to fifth in the original features as $X_{[2:5]}$ and specify $X_{[1]} = 1$ as the intercept. As the reward $Y_t$ is binary, we consider a logistic regression model and set $Y_t = 1$ if the user clicks on the article link and $Y_t = -1$ if not. We use the $\varepsilon$-greedy algorithm~\eqref{eq:eps-greedy}. To match the process with our offline dataset, we keep the entry if the recorded offline action matches the action given by our online $\varepsilon$-greedy algorithm with two specifications of $\varepsilon\in\{0.2,0.02\}$.

We use the same specifications as above, $300$-step meltdown and $\alpha = 0.8$, and compare three weighting schemes, {\hyperref[eq:vanilla]{\vanilla}}, {\hyperref[eq:sqrt-weight]{\sipw}}, and {\hyperref[eq:ipw-weight]{\ipw}. Tables~B.5 and B.6 in the supplement present the result for $\varepsilon=0.2$ and $\varepsilon=0.02$. {Our results and findings for {\hyperref[eq:ipw-weight]{\ipw}} align with those in \cite{chen2021b}, while  {\hyperref[eq:vanilla]{\vanilla}} and {\hyperref[eq:sqrt-weight]{\sipw}} have smaller standard errors and $p$-values, matching our discussion regarding the different weight schemes.}

}
\newpage
\bibliographystyle{chicago}
\bibliography{refs}

\begin{thebibliography}{}

\bibitem[\protect\citeauthoryear{Bahadur}{Bahadur}{1966}]{bahadur1966note}
Bahadur, R.~R. (1966).
\newblock A note on quantiles in large samples.
\newblock {\em Ann. Math. Stat.\/}~{\em 37\/}(3).

\bibitem[\protect\citeauthoryear{Carroll}{Carroll}{1978}]{carroll1978almost}
Carroll, R.~J. (1978).
\newblock On almost sure expansions for {$M$}-estimates.
\newblock {\em Ann. Stat.\/}~{\em 6\/}(2), 314--318.

\bibitem[\protect\citeauthoryear{Chen, Song, and Jordan}{Chen
  et~al.}{2024}]{chen2022reinforcement}
Chen, E.~Y., R.~Song, and M.~I. Jordan (2024).
\newblock Reinforcement learning in latent heterogeneous environments.
\newblock {\em J. Am. Stat. Assoc.\/}~{\em 119\/}(548), 3113--3126.

\bibitem[\protect\citeauthoryear{Chen, Lu, and Song}{Chen
  et~al.}{2021a}]{chen2021a}
Chen, H., W.~Lu, and R.~Song (2021a).
\newblock Statistical inference for online decision making: In a contextual
  bandit setting.
\newblock {\em J. Am. Stat. Assoc.\/}~{\em 116\/}(533), 240--255.

\bibitem[\protect\citeauthoryear{Chen, Lu, and Song}{Chen
  et~al.}{2021b}]{chen2021b}
Chen, H., W.~Lu, and R.~Song (2021b).
\newblock Statistical inference for online decision making via stochastic
  gradient descent.
\newblock {\em J. Am. Stat. Assoc.\/}~{\em 116\/}(534), 708--719.

\bibitem[\protect\citeauthoryear{Chen, Lai, Li, and Zhang}{Chen
  et~al.}{2024}]{chen2024online}
Chen, X., Z.~Lai, H.~Li, and Y.~Zhang (2024).
\newblock Online statistical inference for stochastic optimization via
  {K}iefer-{W}olfowitz methods.
\newblock {\em J. Am. Stat. Assoc.\/}~{\em 119\/}(548), 2972--2982.

\bibitem[\protect\citeauthoryear{Chen, Lee, Tong, and Zhang}{Chen
  et~al.}{2020}]{chen2016statistical}
Chen, X., J.~D. Lee, X.~T. Tong, and Y.~Zhang (2020).
\newblock Statistical inference for model parameters in stochastic gradient
  descent.
\newblock {\em Ann. Stat.\/}~{\em 48\/}(1), 251--273.

\bibitem[\protect\citeauthoryear{Deshpande, Mackey, Syrgkanis, and
  Taddy}{Deshpande et~al.}{2018}]{deshpande2018accurate}
Deshpande, Y., L.~Mackey, V.~Syrgkanis, and M.~Taddy (2018).
\newblock Accurate inference for adaptive linear models.
\newblock In {\em International Conference on Machine Learning}.

\bibitem[\protect\citeauthoryear{Duchi and Ruan}{Duchi and
  Ruan}{2021}]{duchi2021asymptotic}
Duchi, J.~C. and F.~Ruan (2021).
\newblock Asymptotic optimality in stochastic optimization.
\newblock {\em Ann. Stat.\/}~{\em 49\/}(1), 21--48.

\bibitem[\protect\citeauthoryear{Fang, Xu, and Yang}{Fang
  et~al.}{2018}]{fang2018online}
Fang, Y., J.~Xu, and L.~Yang (2018).
\newblock Online bootstrap confidence intervals for the stochastic gradient
  descent estimator.
\newblock {\em J. Mach. Learn. Res.\/}~{\em 19\/}(1), 3053--3073.

\bibitem[\protect\citeauthoryear{Hadad, Hirshberg, Zhan, Wager, and
  Athey}{Hadad et~al.}{2021}]{hadad2021confidence}
Hadad, V., D.~A. Hirshberg, R.~Zhan, S.~Wager, and S.~Athey (2021).
\newblock Confidence intervals for policy evaluation in adaptive experiments.
\newblock {\em Proc. Natl. Acad. Sci.\/}~{\em 118\/}(15).

\bibitem[\protect\citeauthoryear{Hammersley}{Hammersley}{2013}]{hammersley2013monte}
Hammersley, J. (2013).
\newblock {\em Monte carlo methods}.
\newblock Springer Science \& Business Media.

\bibitem[\protect\citeauthoryear{Han, Sun, and Zhang}{Han
  et~al.}{2025}]{han2022online}
Han, Q., W.~W. Sun, and Y.~Zhang (2025).
\newblock Online statistical inference in decision-making with matrix context.
\newblock {\em Ann. Stat.\/}~{\em 53\/}(5), 1963--1986.

\bibitem[\protect\citeauthoryear{Hao, Abbasi~Yadkori, Wen, and Cheng}{Hao
  et~al.}{2019}]{hao2019bootstrapping}
Hao, B., Y.~Abbasi~Yadkori, Z.~Wen, and G.~Cheng (2019).
\newblock Bootstrapping upper confidence bound.
\newblock {\em Neural Information Processing Systems\/}.

\bibitem[\protect\citeauthoryear{He and Shao}{He and
  Shao}{1996}]{he1996general}
He, X. and Q.-M. Shao (1996).
\newblock A general bahadur representation of {M}-estimators and its
  application to linear regression with nonstochastic designs.
\newblock {\em Ann. Stat.\/}~{\em 24\/}(6), 2608--2630.

\bibitem[\protect\citeauthoryear{Khamaru, Deshpande, Mackey, and
  Wainwright}{Khamaru et~al.}{2025}]{khamaru2021near}
Khamaru, K., Y.~Deshpande, L.~Mackey, and M.~J. Wainwright (2025).
\newblock Near-optimal inference in adaptive linear regression.
\newblock {\em Ann. Stat.\/}~{\em 53\/}(6), 2329--2355.

\bibitem[\protect\citeauthoryear{Lattimore and Szepesv{\'a}ri}{Lattimore and
  Szepesv{\'a}ri}{2020}]{lattimore2020bandit}
Lattimore, T. and C.~Szepesv{\'a}ri (2020).
\newblock {\em Bandit algorithms}.
\newblock Cambridge University Press.

\bibitem[\protect\citeauthoryear{Lee, Liao, Seo, and Shin}{Lee
  et~al.}{2022}]{lee2021fast}
Lee, S., Y.~Liao, M.~H. Seo, and Y.~Shin (2022).
\newblock Fast and robust online inference with stochastic gradient descent via
  random scaling.
\newblock In {\em AAAI Conference on Artificial Intelligence}.

\bibitem[\protect\citeauthoryear{Lee, Liao, Seo, and Shin}{Lee
  et~al.}{2025}]{lee2022fast}
Lee, S., Y.~Liao, M.~H. Seo, and Y.~Shin (2025).
\newblock Fast inference for quantile regression with millions of observations.
\newblock {\em J. Econom.\/}~{\em 249}, 105673.

\bibitem[\protect\citeauthoryear{Liu, Tu, Chen, and Zhang}{Liu
  et~al.}{2025}]{liu2023online}
Liu, W., J.~Tu, X.~Chen, and Y.~Zhang (2025).
\newblock Online estimation and inference for robust policy evaluation in
  reinforcement learning.
\newblock {\em Ann. Stat.\/}~{\em 53\/}(5), 2128--2152.

\bibitem[\protect\citeauthoryear{Luedtke and van~der Laan}{Luedtke and van~der
  Laan}{2016}]{luedtke2016statistical}
Luedtke, A.~R. and M.~J. van~der Laan (2016).
\newblock Statistical inference for the mean outcome under a possibly
  non-unique optimal treatment strategy.
\newblock {\em Ann. Stat.\/}~{\em 44\/}(2), 713.

\bibitem[\protect\citeauthoryear{Polyak and Juditsky}{Polyak and
  Juditsky}{1992}]{polyak1992acceleration}
Polyak, B.~T. and A.~B. Juditsky (1992).
\newblock Acceleration of stochastic approximation by averaging.
\newblock {\em SIAM J. Control. Optim.\/}~{\em 30\/}(4), 838--855.

\bibitem[\protect\citeauthoryear{Ramprasad, Li, Yang, Wang, Sun, and
  Cheng}{Ramprasad et~al.}{2023}]{ramprasad2022online}
Ramprasad, P., Y.~Li, Z.~Yang, Z.~Wang, W.~W. Sun, and G.~Cheng (2023).
\newblock Online bootstrap inference for policy evaluation in reinforcement
  learning.
\newblock {\em J. Am. Stat. Assoc.\/}~{\em 118\/}(544), 2901--2914.

\bibitem[\protect\citeauthoryear{Robbins}{Robbins}{1952}]{robbins1952some}
Robbins, H. (1952).
\newblock Some aspects of the sequential design of experiments.
\newblock {\em Bulletin of the American Mathematical Society\/}~{\em 58\/}(5),
  527--535.

\bibitem[\protect\citeauthoryear{Robbins and Monro}{Robbins and
  Monro}{1951}]{robbins1951stochastic}
Robbins, H. and S.~Monro (1951).
\newblock A stochastic approximation method.
\newblock {\em Ann. Math. Stat.\/}~{\em 22\/}(3), 400--407.

\bibitem[\protect\citeauthoryear{Rubin}{Rubin}{2005}]{rubin2005causal}
Rubin, D.~B. (2005).
\newblock Causal inference using potential outcomes: Design, modeling,
  decisions.
\newblock {\em J. Am. Stat. Assoc.\/}~{\em 100\/}(469), 322--331.

\bibitem[\protect\citeauthoryear{Ruppert}{Ruppert}{1988}]{ruppert1988efficient}
Ruppert, D. (1988).
\newblock Efficient estimations from a slowly convergent robbins-monro process.
\newblock Technical report, Cornell University ORIE.

\bibitem[\protect\citeauthoryear{Shao and Zhang}{Shao and
  Zhang}{2022}]{shao2022berry}
Shao, Q.-M. and Z.-S. Zhang (2022).
\newblock Berry--{E}sseen bounds for multivariate nonlinear statistics with
  applications to {M}-estimators.
\newblock {\em Bernoulli\/}~{\em 28\/}(3), 1548--1576.

\bibitem[\protect\citeauthoryear{Shi, Zhang, Lu, and Song}{Shi
  et~al.}{2022}]{shi2021statistical}
Shi, C., S.~Zhang, W.~Lu, and R.~Song (2022).
\newblock Statistical inference of the value function for reinforcement
  learning in infinite-horizon settings.
\newblock {\em J. R. Stat. Soc. Ser. B\/}~{\em 84\/}(3), 765--793.

\bibitem[\protect\citeauthoryear{Shi, Zhu, Ye, Luo, Zhu, and Song}{Shi
  et~al.}{2024}]{shi2022off}
Shi, C., J.~Zhu, S.~Ye, S.~Luo, H.~Zhu, and R.~Song (2024).
\newblock Off-policy confidence interval estimation with confounded markov
  decision process.
\newblock {\em J. Am. Stat. Assoc.\/}~{\em 119\/}(545), 273--284.

\bibitem[\protect\citeauthoryear{Su and Zhu}{Su and
  Zhu}{2023}]{su2018uncertainty}
Su, W. and Y.~Zhu (2023).
\newblock {HiGrad}: Uncertainty quantification for online learning and
  stochastic approximation.
\newblock {\em J. Mach. Learn. Res.\/}~{\em 24\/}(124).

\bibitem[\protect\citeauthoryear{Tang, Liu, Zhang, and Chen}{Tang
  et~al.}{2023}]{tang2023acceleration}
Tang, K., W.~Liu, Y.~Zhang, and X.~Chen (2023).
\newblock Acceleration of stochastic gradient descent with momentum by
  averaging: finite-sample rates and asymptotic normality.
\newblock {\em arXiv preprint arXiv:2305.17665\/}.

\bibitem[\protect\citeauthoryear{van~der Vaart and Wellner}{van~der Vaart and
  Wellner}{2013}]{van2013weak}
van~der Vaart, A. and J.~Wellner (2013).
\newblock {\em Weak Convergence and Empirical Processes: With Applications to
  Statistics}.
\newblock Springer Series in Statistics. Springer New York.

\bibitem[\protect\citeauthoryear{van~der Vaart}{van~der
  Vaart}{2000}]{van2000asymptotic}
van~der Vaart, A.~W. (2000).
\newblock {\em Asymptotic statistics}, Volume~3.
\newblock Cambridge university press.

\bibitem[\protect\citeauthoryear{Wen, Sun, and Zhang}{Wen
  et~al.}{2023}]{wen2023online}
Wen, X., W.~W. Sun, and Y.~Zhang (2023).
\newblock Online tensor inference.
\newblock {\em arXiv preprint arXiv:2312.17111\/}.

\bibitem[\protect\citeauthoryear{Zhan, Hadad, Hirshberg, and Athey}{Zhan
  et~al.}{2021}]{zhan2021off}
Zhan, R., V.~Hadad, D.~A. Hirshberg, and S.~Athey (2021).
\newblock Off-policy evaluation via adaptive weighting with data from
  contextual bandits.
\newblock In {\em ACM SIGKDD Conference on Knowledge Discovery \& Data Mining}.

\bibitem[\protect\citeauthoryear{Zhang, Janson, and Murphy}{Zhang
  et~al.}{2021}]{zhang2021statistical}
Zhang, K., L.~Janson, and S.~Murphy (2021).
\newblock Statistical inference with {M}-estimators on adaptively collected
  data.
\newblock {\em Neural Information Processing Systems\/}.

\bibitem[\protect\citeauthoryear{Zhang, Janson, and Murphy}{Zhang
  et~al.}{2022}]{zhang2022statistical}
Zhang, K.~W., L.~Janson, and S.~A. Murphy (2022).
\newblock Statistical inference after adaptive sampling in non-markovian
  environments.
\newblock {\em arXiv preprint arXiv:2202.07098\/}.

\bibitem[\protect\citeauthoryear{Zhu, Chen, and Wu}{Zhu
  et~al.}{2023}]{zhu2021online}
Zhu, W., X.~Chen, and W.~B. Wu (2023).
\newblock Online covariance matrix estimation in stochastic gradient descent.
\newblock {\em J. Am. Stat. Assoc.\/}~{\em 118\/}(541), 393--404.

\end{thebibliography}

\setcounter{section}{0}
\renewcommand\thesection{\Alph{section}}
\appendix

\numberwithin{figure}{section}
\numberwithin{table}{section}
\numberwithin{equation}{section}
\counterwithin{assumption}{section}
\counterwithin{theorem}{section}
\newpage
\setcounter{page}{1}

\section{Notations}\label{supp:notation}
We first introduce some notations in our paper. For any pair of positive integers $m < n$, we use $[m:n]$ as a shorthand for the discrete set of $\{m, m+1, \ldots, n\}$. For any vector $\theta \in \R^d$, we use $\theta_{[m:n]}$ to denote the vector consisting of the $m$-th to $n$-th coordinates of $\theta$. Similarly, $\theta_{[m:n],t}$ is the corresponding subvector of $\theta_t$. For a set of random variables $X_n$ and a corresponding set of constants $a_n$, $X_n$ = $\mathcal{O}_p(a_n)$ means that $X_n/a_n$ is stochastically bounded and $X_n = o_p(a_n)$ means that $X_n/a_n$ converges to zero in probability as $n$ goes to infinity. We denote $\overset{p}{\rightarrow}$ and $\overset{d}{\rightarrow}$ as convergence in probability and convergence in distribution, respectively.  For convenience, let $\| \cdot \|$ denote the standard Euclidean norm for vectors and the spectral norm for matrices. We use the standard Loewner order notation $\Sigma \succeq 0$ if a matrix $\Sigma$ is positive semi-definite. Denote $I_d$ as the identity matrix in $\R^{d\times d}$. For any square matrix $\Sigma$, $\lambda_{\min}(\Sigma)$ and $\lambda_{\max}(\Sigma)$ represent the smallest and the largest eigenvalues, respectively. We also introduce $\ID(\cdot)$ for the indicator function, and $\lesssim$ is used for inequalities with omitted constants.

\section{Figures and tables}\label{subsec:figuretable}In this section, we present the additional figures and tables relegated from the main text. 
\subsection{Results for studentized statistics}\label{subsec:stu}
\begin{figure}[H]
    \centering
    \subfigure[\vanilla, Arm $0$]{
    \includegraphics[width=0.22\textwidth]{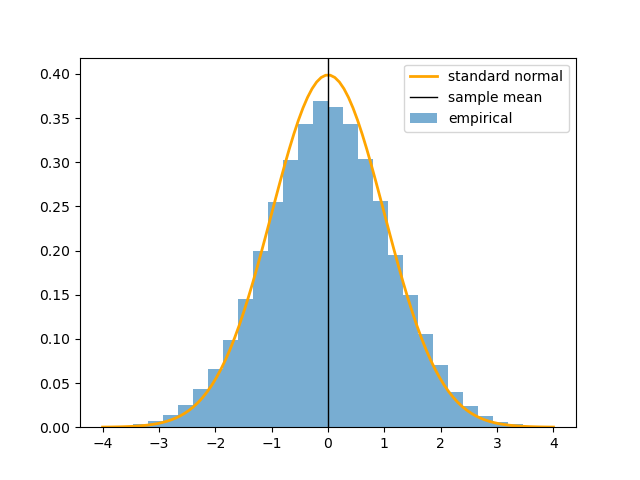}}
    \subfigure[\sipw, Arm $0$]{
    \includegraphics[width=0.22\textwidth]{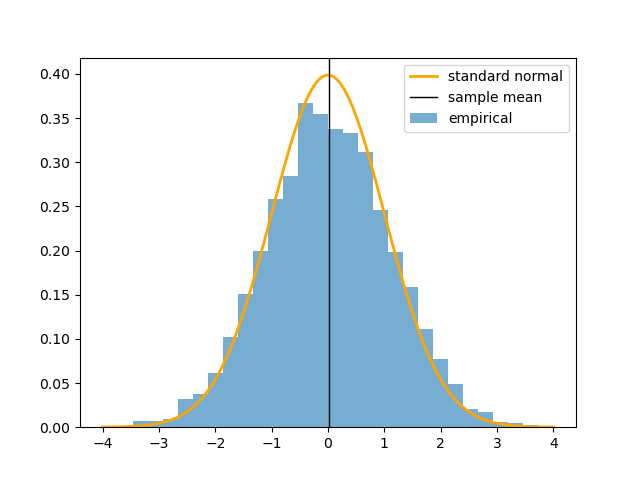}}
    \subfigure[\ipw, Arm $0$]{
    \includegraphics[width=0.22\textwidth]{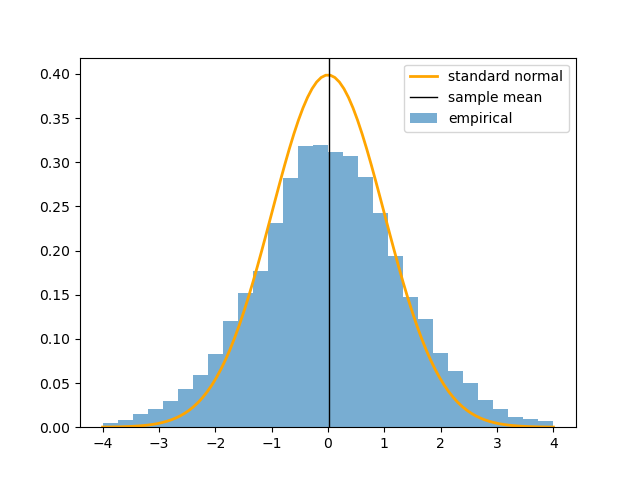}}\\
    \subfigure[\vanilla, Arm $1$]{
    \includegraphics[width=0.22\textwidth]{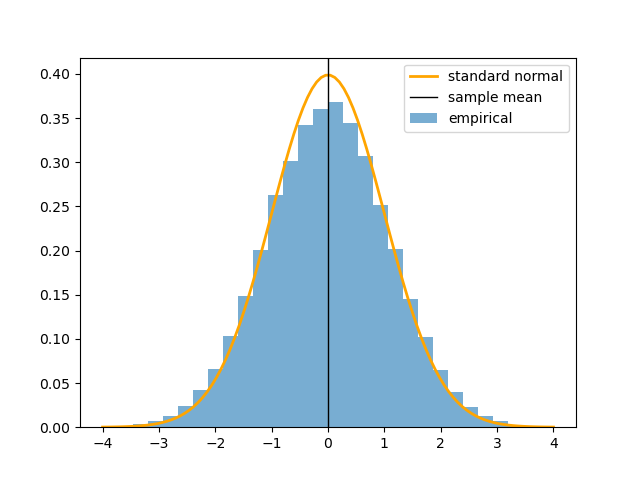}}
    \subfigure[\sipw, Arm $1$]{
    \includegraphics[width=0.22\textwidth]{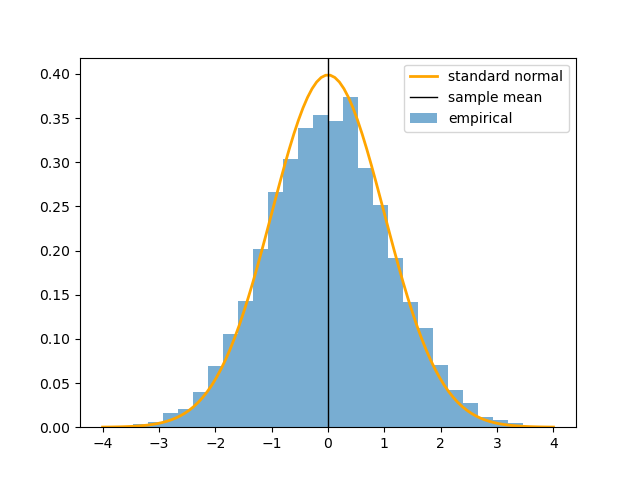}}
    \subfigure[\ipw, Arm $1$]{
    \includegraphics[width=0.22\textwidth]{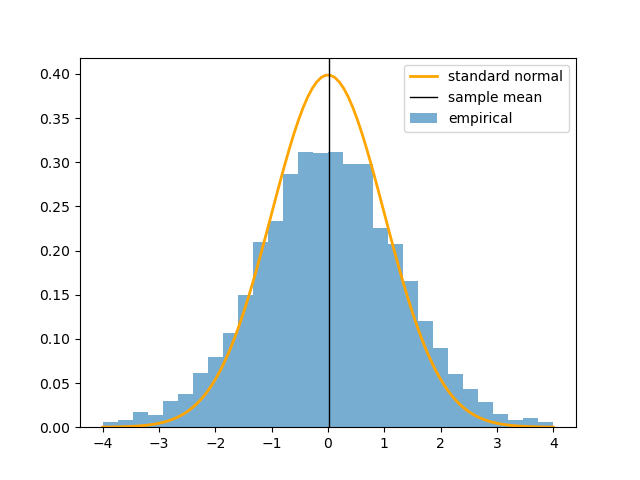}}\\
    \caption{SGD on linear regression with modified $\varepsilon$-greedy and different weights in the non-degenerate model. We report the empirical distribution of each action's first dimension of $\sqrt{t}\hat{S}_t^{-1/2}\hat{H}_t(\bar{\theta}_t - \theta^*)$ for $10,000$ Monte-Carlo simulations.}
    \label{fig-app-linear2}
\end{figure}

\begin{figure}[H]
    \centering
    \subfigure[\vanilla, Arm $0$]{
    \includegraphics[width=0.22\textwidth]{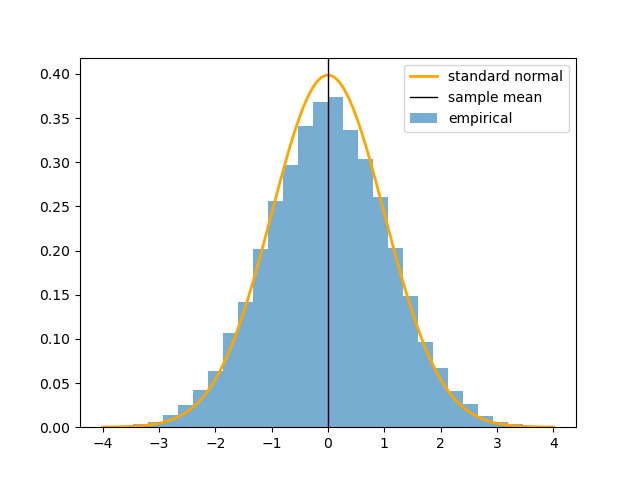}}
    \subfigure[\sipw, Arm $0$]{
    \includegraphics[width=0.22\textwidth]{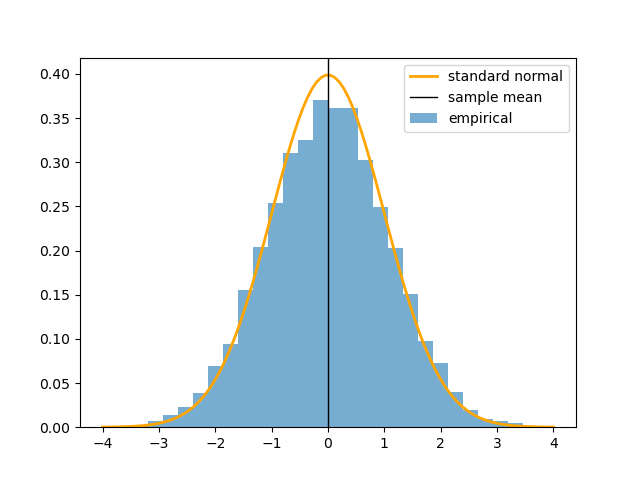}}
    \subfigure[\ipw, Arm $0$]{
    \includegraphics[width=0.22\textwidth]{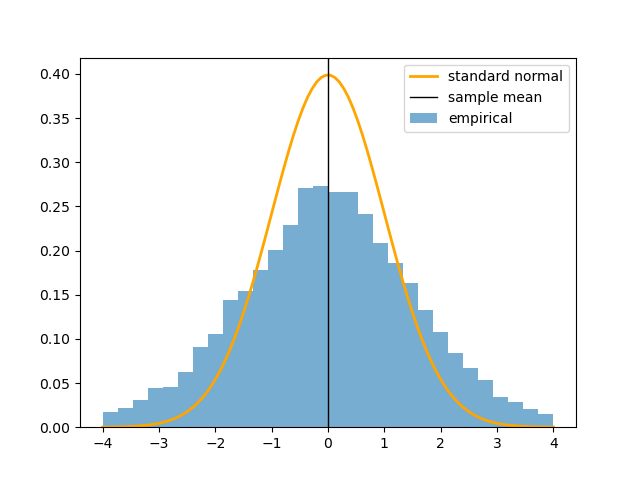}}\\
    \subfigure[\vanilla, Arm $1$]{
    \includegraphics[width=0.22\textwidth]{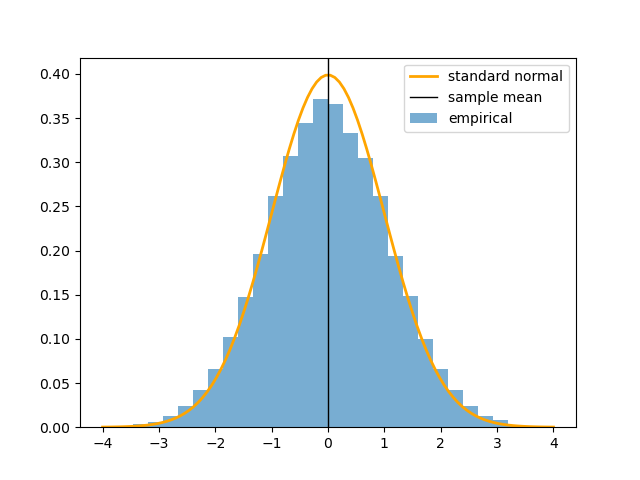}}
    \subfigure[\sipw, Arm $1$]{
    \includegraphics[width=0.22\textwidth]{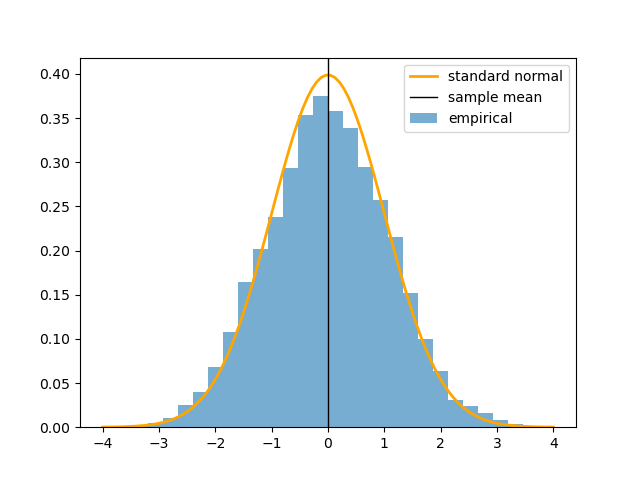}}
    \subfigure[\ipw, Arm $1$]{
    \includegraphics[width=0.22\textwidth]{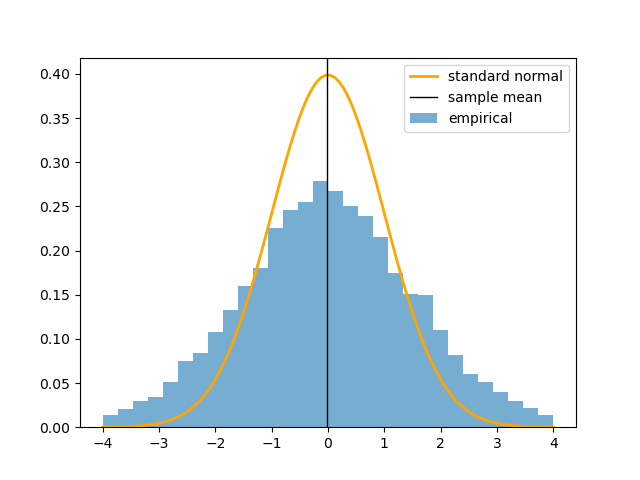}}\\
    \caption{SGD on linear regression with modified $\varepsilon$-greedy and different weights in the degenerate model. We report the empirical distribution of each action's first dimension of $\sqrt{t}\hat{S}_t^{-1/2}\hat{H}_t(\bar{\theta}_t - \theta^*)$ for $10,000$ Monte-Carlo simulations.}
    \label{fig-app-linear3}
\end{figure}

\subsection{Results for quantile regression}
The setting here is the same as Section \ref{sec:normal-approximation} and we use the modified $\varepsilon$-greedy policy. We conduct simulations on quantile regression with quantile level $\tau = 0.75$, and report in Figure~\ref{fig-qtl-non} for the non-degenerate model and Figure \ref{fig-qtl-de} for the degenerate model.

\vspace{-1em}
\begin{figure}[H]
    \centering
    \begin{minipage}{0.22\textwidth}
        \centering
        \includegraphics[width=\textwidth]{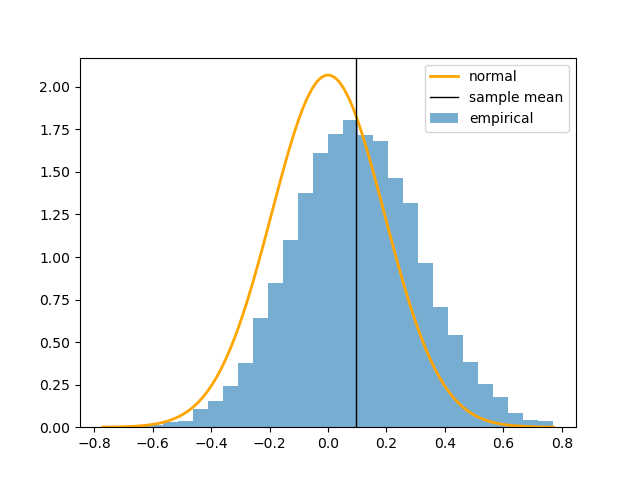}
        \caption*{\vanilla, Arm $0$}
    \end{minipage}
    \begin{minipage}{0.22\textwidth}
        \centering
        \includegraphics[width=\textwidth]{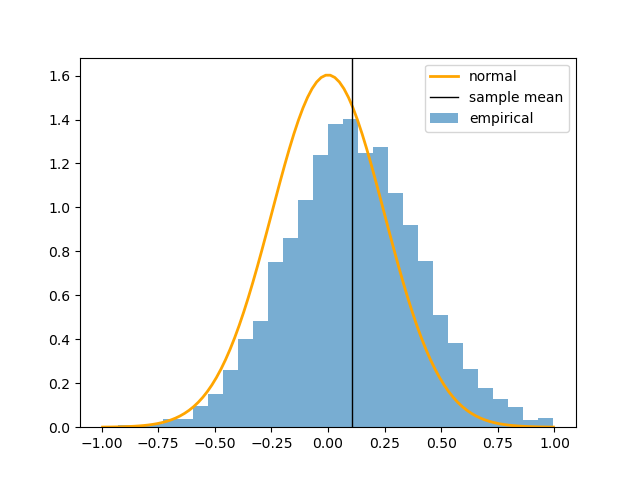}
        \caption*{\sipw, Arm $0$}
    \end{minipage}
    \begin{minipage}{0.22\textwidth}
        \centering
        \includegraphics[width=\textwidth]{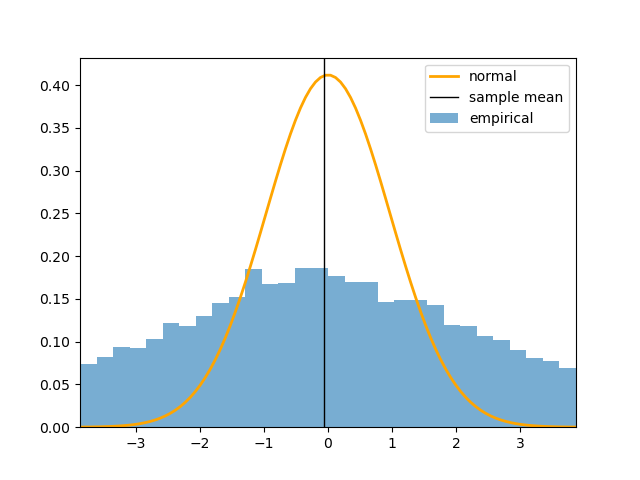}
        \caption*{\ipw, Arm $0$}
    \end{minipage}\\
    \begin{minipage}{0.22\textwidth}
        \centering
        \includegraphics[width=\textwidth]{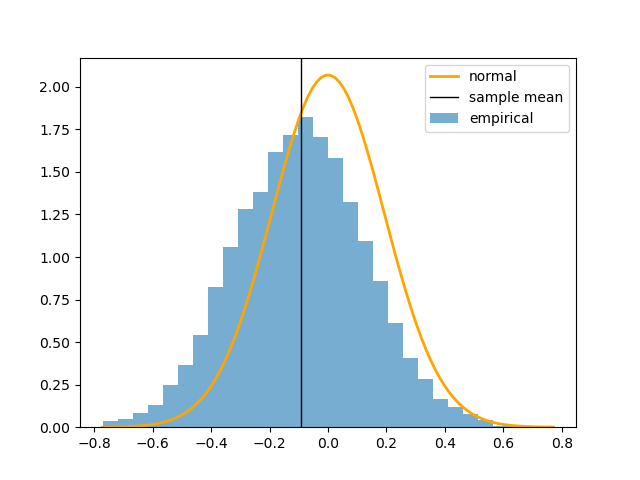}
        \caption*{\vanilla, Arm $1$}
    \end{minipage}
    \begin{minipage}{0.22\textwidth}
        \centering
        \includegraphics[width=\textwidth]{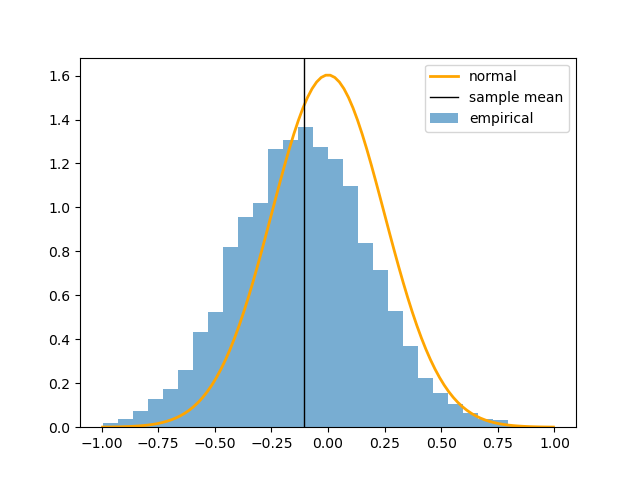}
        \caption*{\sipw, Arm $1$}
    \end{minipage}
    \begin{minipage}{0.22\textwidth}
        \centering
        \includegraphics[width=\textwidth]{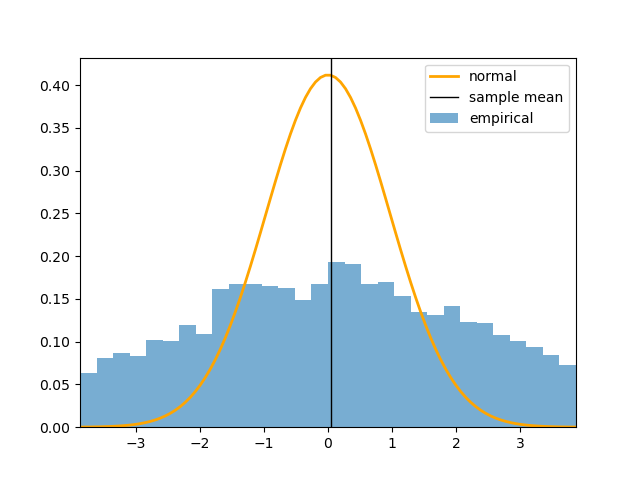}
        \caption*{\ipw, Arm $1$}
    \end{minipage}
    \caption{SGD on quantile regression with modified $\varepsilon$-greedy policy and different weights in the non-degenerate model. We report the empirical distribution of each action's first dimension of $\sqrt{t}(\bar{\theta}_t - \theta^*)$ for $10,000$ Monte-Carlo simulations {with $t=200,000$.}}
    \label{fig-qtl-non}
\end{figure}

\vspace{-1em}
\begin{figure}[H]
    \centering
    \begin{minipage}{0.2\textwidth}
        \centering
        \includegraphics[width=\textwidth]{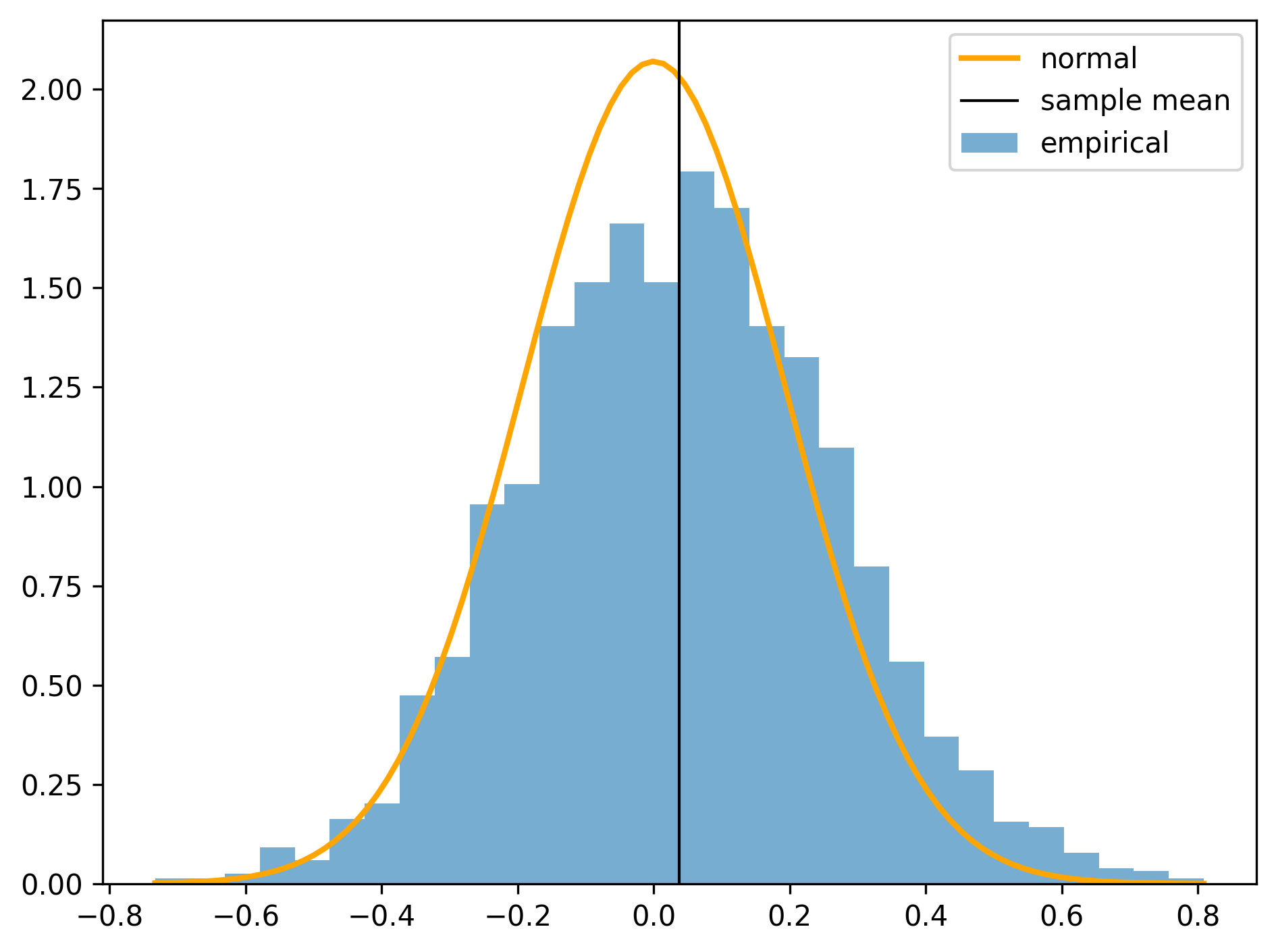}
        \caption*{\vanilla, Arm $0$}
    \end{minipage}
    \begin{minipage}{0.2\textwidth}
        \centering
        \includegraphics[width=\textwidth]{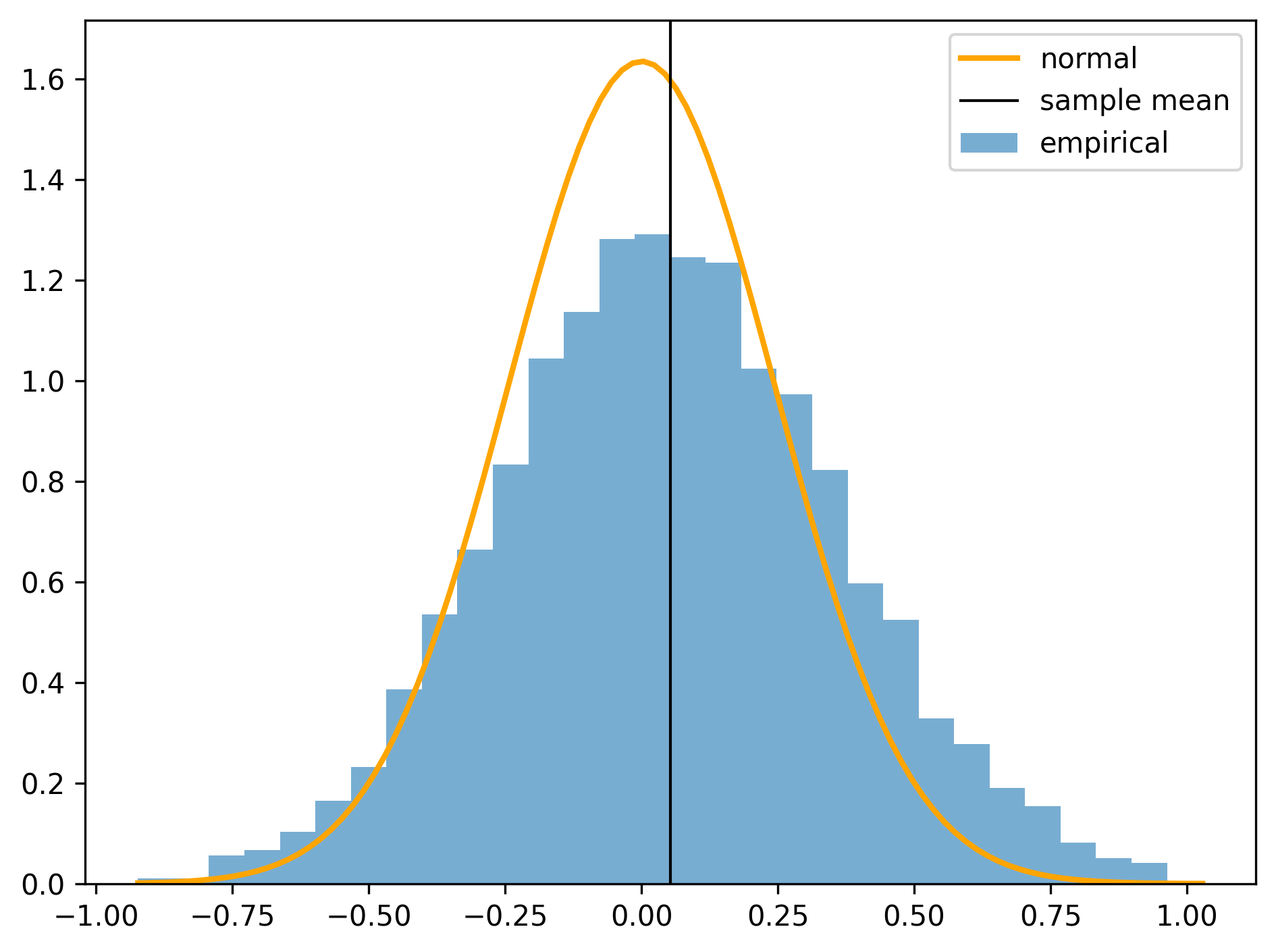}
        \caption*{\sipw, Arm $0$}
    \end{minipage}
    \begin{minipage}{0.2\textwidth}
        \centering
        \includegraphics[width=\textwidth]{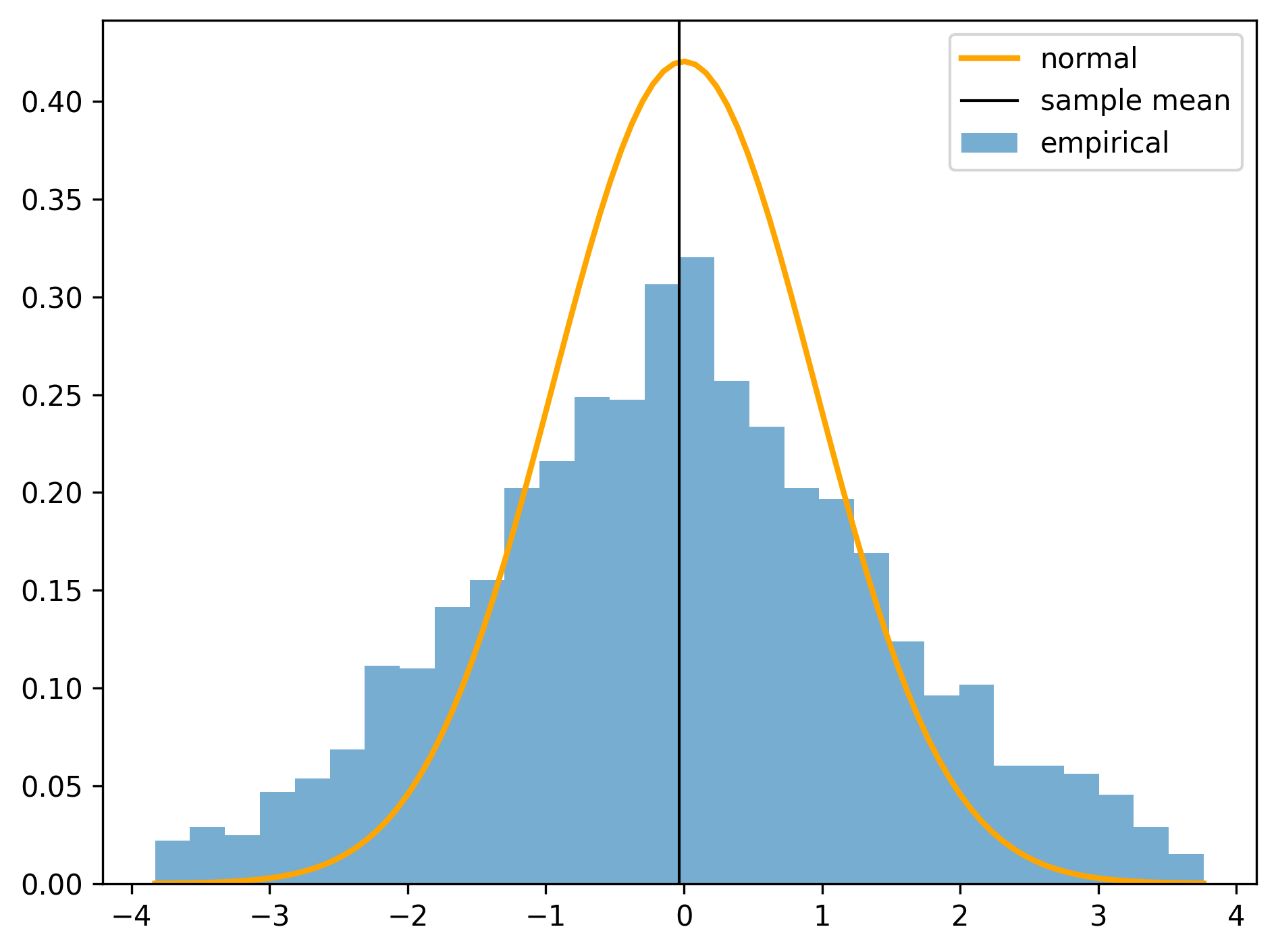}
        \caption*{\ipw, Arm $0$}
    \end{minipage}\\
    \begin{minipage}{0.2\textwidth}
        \centering
        \includegraphics[width=\textwidth]{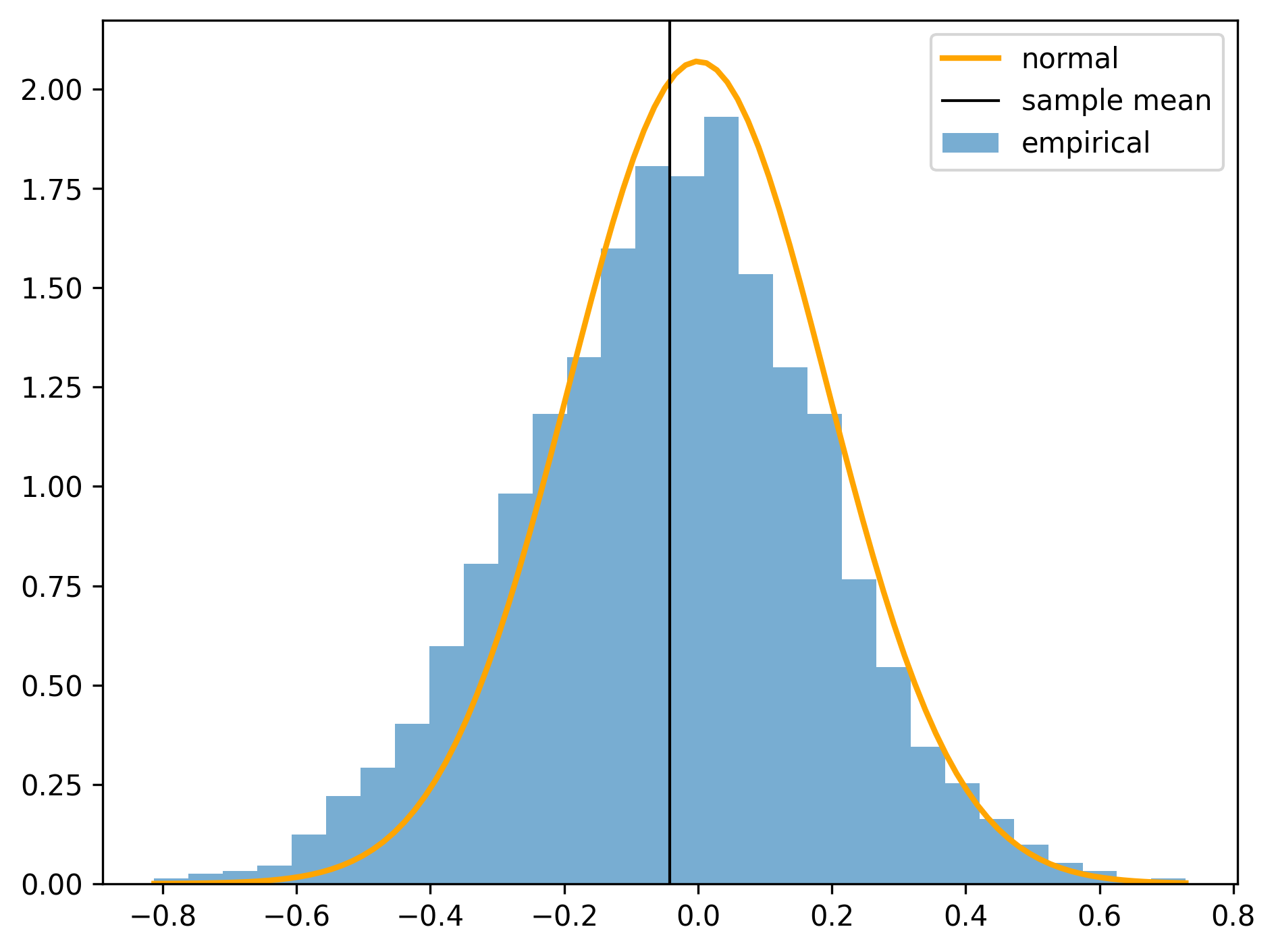}
        \caption*{\vanilla, Arm $1$}
    \end{minipage}
    \begin{minipage}{0.2\textwidth}
        \centering
        \includegraphics[width=\textwidth]{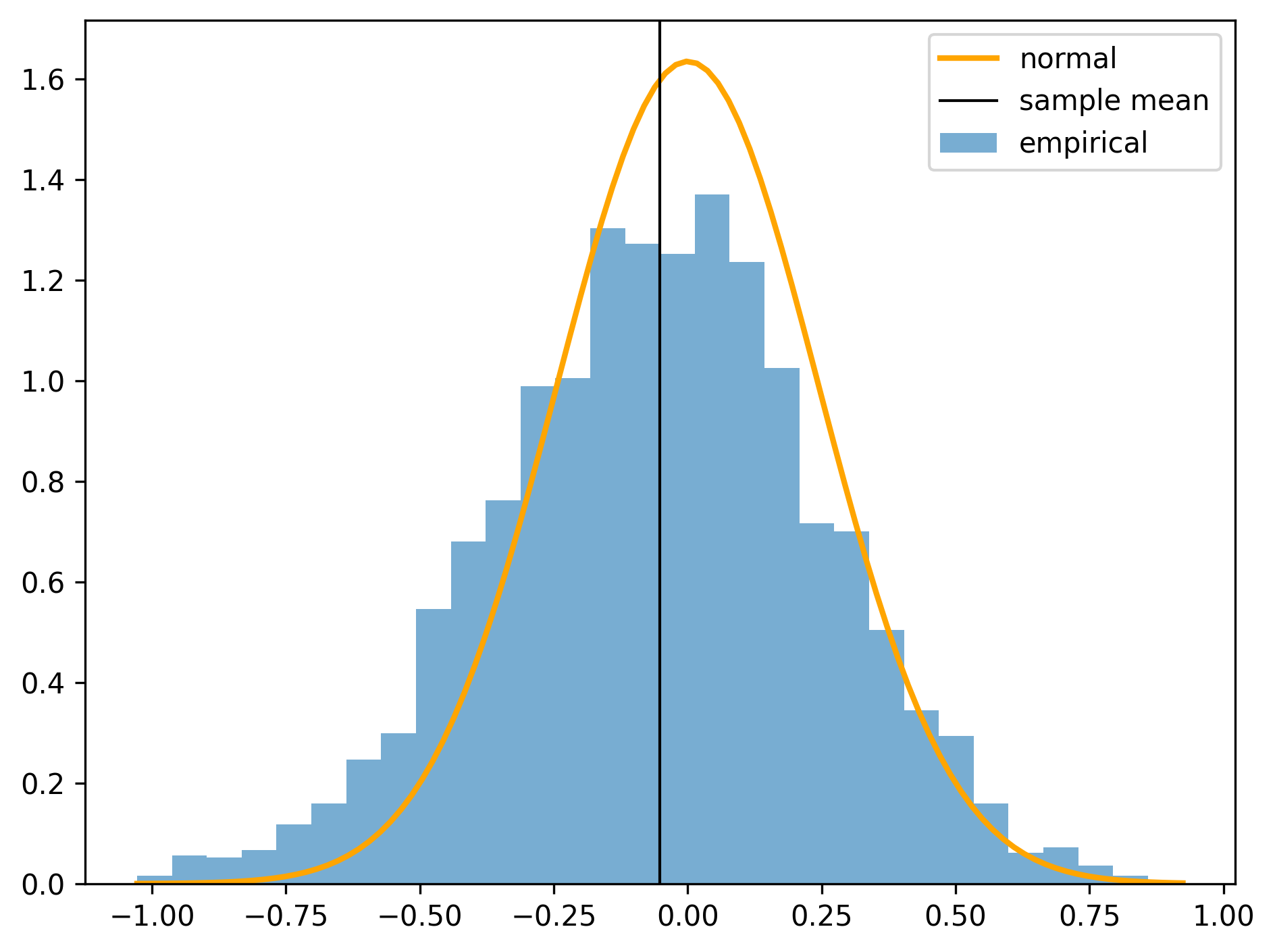}
        \caption*{\sipw, Arm $1$}
    \end{minipage}
    \begin{minipage}{0.2\textwidth}
        \centering
        \includegraphics[width=\textwidth]{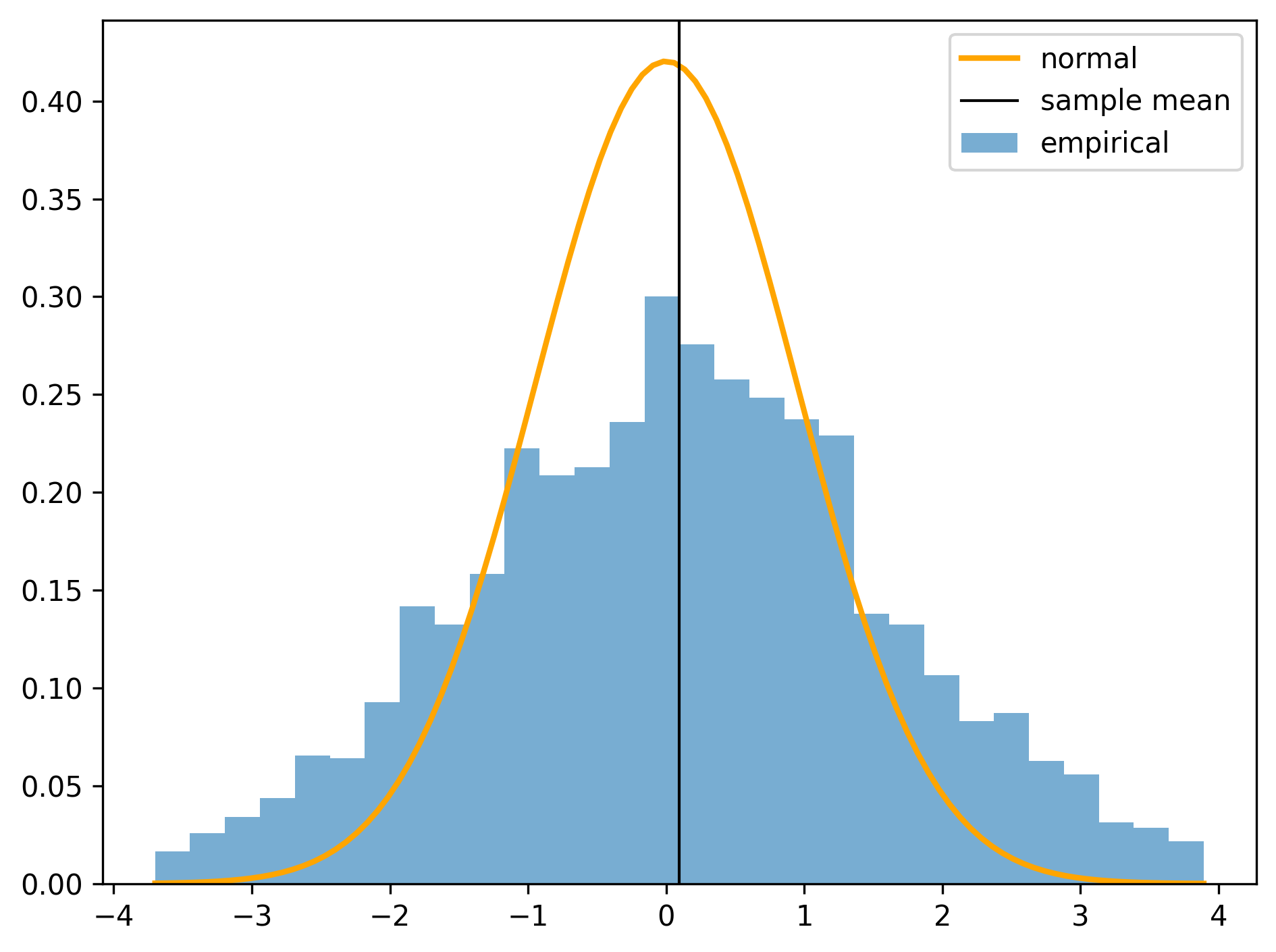}
        \caption*{\ipw, Arm $1$}
    \end{minipage}
    \caption{SGD on quantile regression with modified $\varepsilon$-greedy policy and different weights in the non-degenerate model. We report the empirical distribution of each action's first dimension of $\sqrt{t}(\bar{\theta}_t - \theta^*)$ for $10,000$ Monte-Carlo simulations with $t=2,000,000$.}
    \label{fig-qtl-non1}
\end{figure}

\vspace{-1em}
\begin{figure}[H]
    \centering
    \begin{minipage}{0.22\textwidth}
        \centering
        \includegraphics[width=\textwidth]{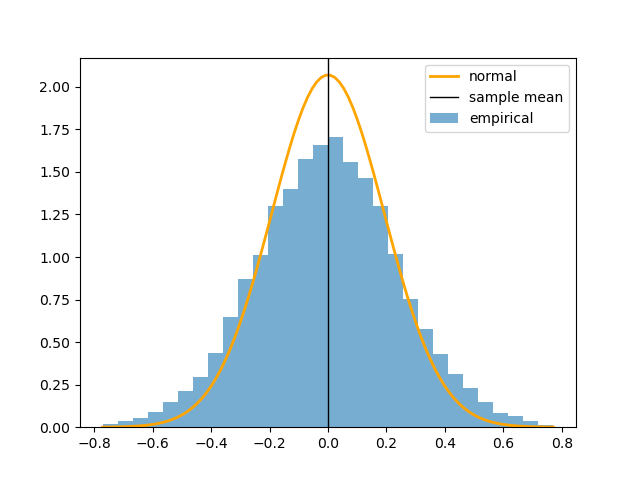}
        \caption*{\vanilla, Arm $0$}
    \end{minipage}
    \begin{minipage}{0.22\textwidth}
        \centering
        \includegraphics[width=\textwidth]{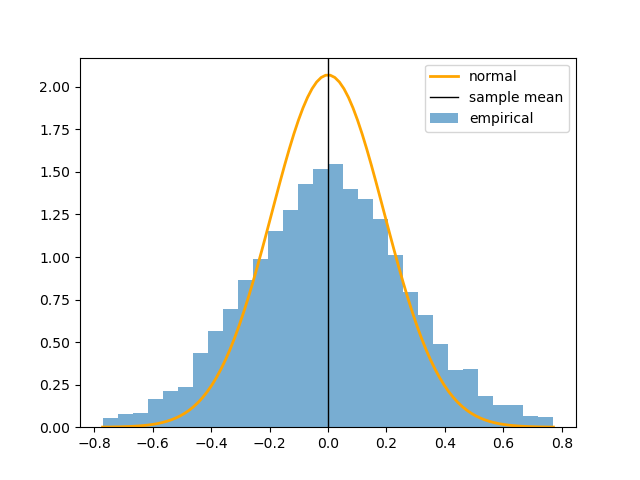}
        \caption*{\sipw, Arm $0$}
    \end{minipage}
    \begin{minipage}{0.22\textwidth}
        \centering
        \includegraphics[width=\textwidth]{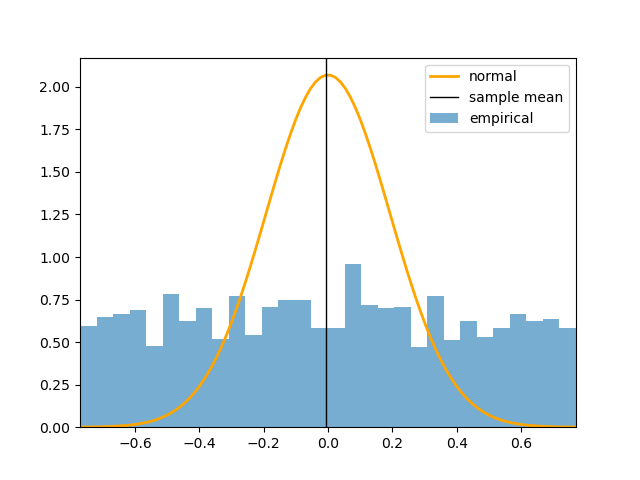}
        \caption*{\ipw, Arm $0$}
    \end{minipage}\\
    \begin{minipage}{0.22\textwidth}
        \centering
        \includegraphics[width=\textwidth]{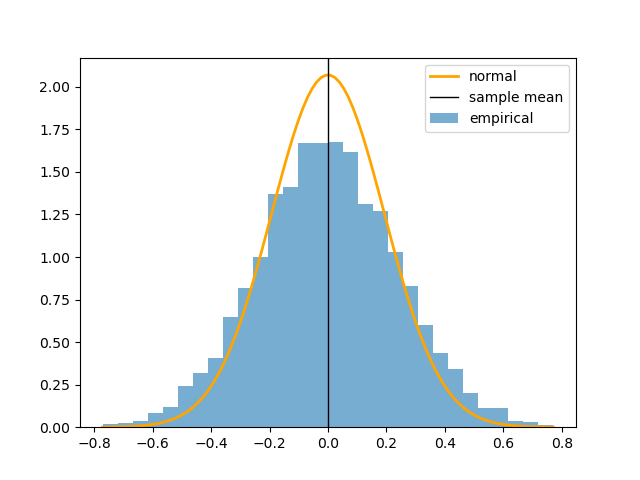}
        \caption*{\vanilla, Arm $1$}
    \end{minipage}
    \begin{minipage}{0.22\textwidth}
        \centering
        \includegraphics[width=\textwidth]{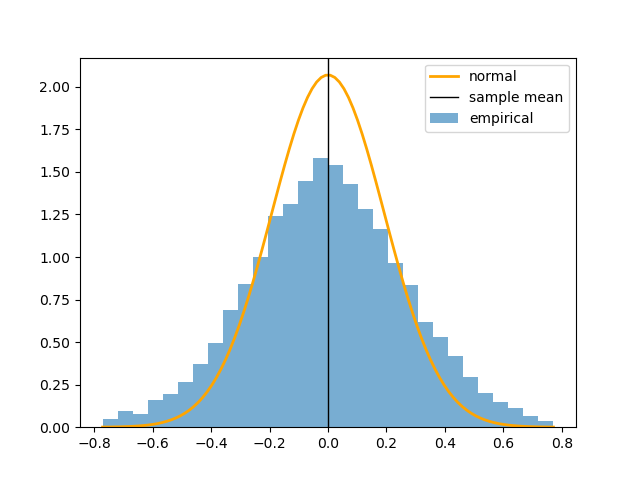}
        \caption*{\sipw, Arm $1$}
    \end{minipage}
    \begin{minipage}{0.22\textwidth}
        \centering
        \includegraphics[width=\textwidth]{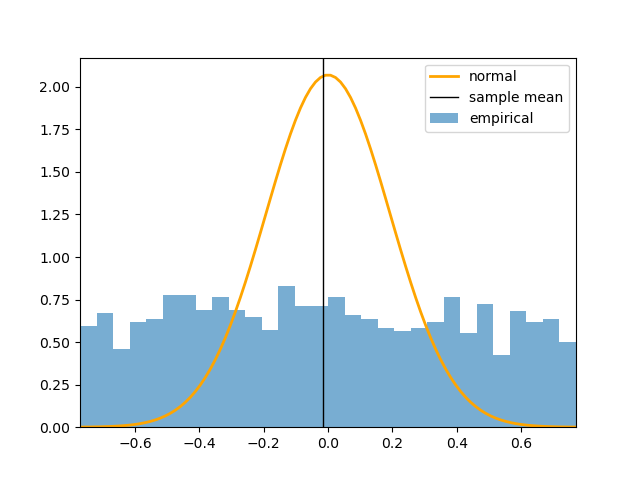}
        \caption*{\ipw, Arm $1$}
    \end{minipage}
    \caption{SGD on quantile regression with modified $\varepsilon$-greedy policy and different weights in the degenerate model. We report the empirical distribution of each action's first dimension of $\sqrt{t}(\bar{\theta}_t - \theta^*)$ for $10,000$ Monte-Carlo simulations.}
    \label{fig-qtl-de}
\end{figure}

\subsection{Results for exponential policy}
\vspace{-1em}
\begin{figure}[H]
    \centering
    \begin{minipage}{0.22\textwidth}
        \centering
        \includegraphics[width=\textwidth]{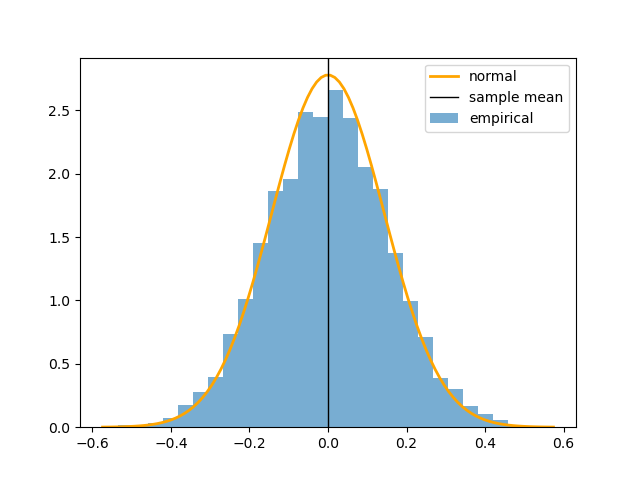}
        \caption*{\vanilla, Arm $0$}
    \end{minipage}
    \begin{minipage}{0.22\textwidth}
        \centering
        \includegraphics[width=\textwidth]{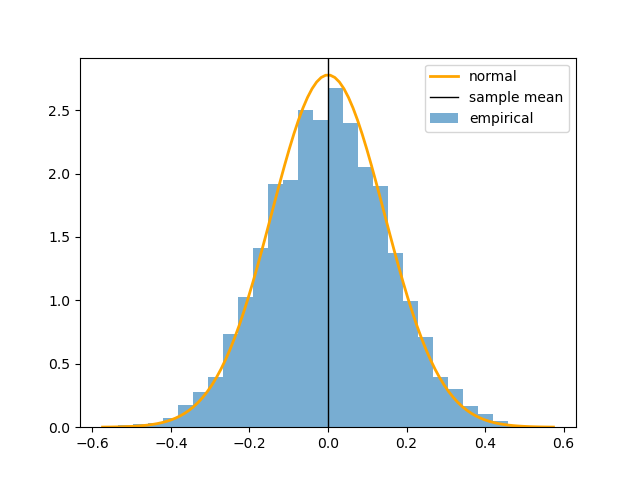}
        \caption*{\sipw, Arm $0$}
    \end{minipage}
    \begin{minipage}{0.22\textwidth}
        \centering
        \includegraphics[width=\textwidth]{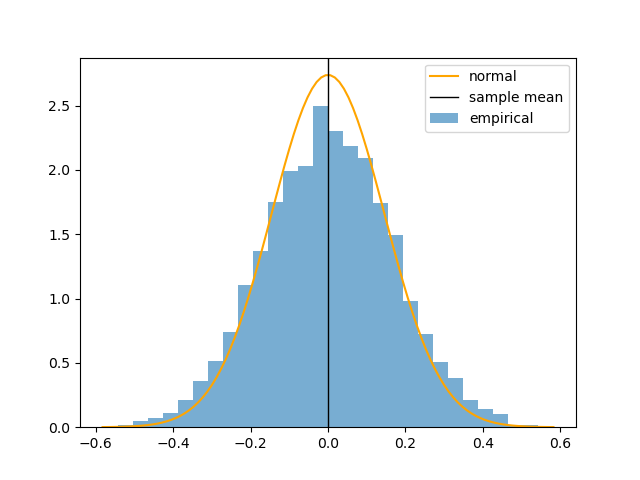}
        \caption*{\ipw, Arm $0$}
    \end{minipage}\\
    \begin{minipage}{0.22\textwidth}
        \centering
        \includegraphics[width=\textwidth]{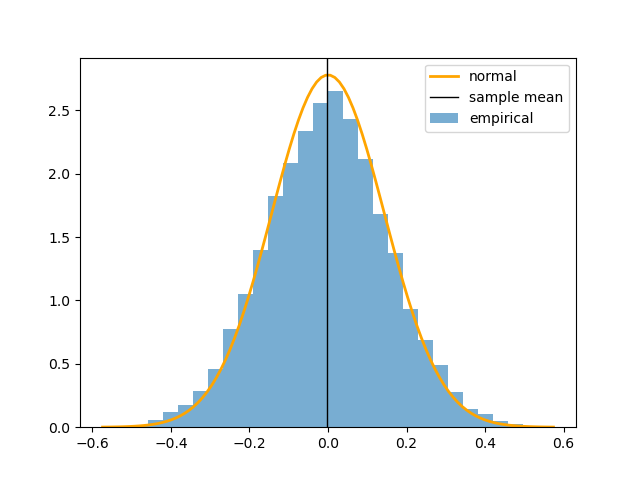}
        \caption*{\vanilla, Arm $1$}
    \end{minipage}
    \begin{minipage}{0.22\textwidth}
        \centering
        \includegraphics[width=\textwidth]{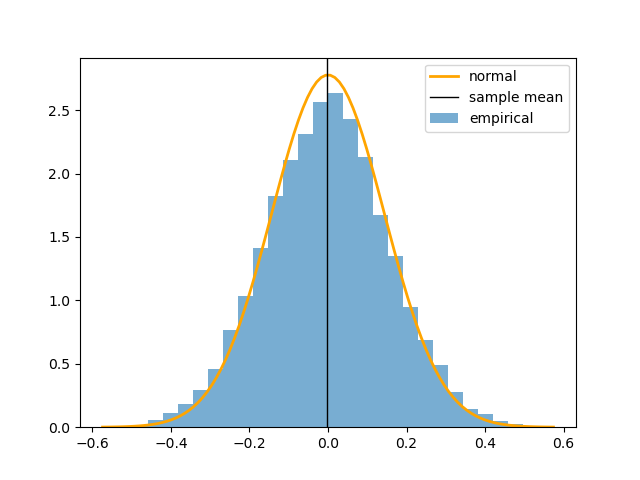}
        \caption*{\sipw, Arm $1$}
    \end{minipage}
    \begin{minipage}{0.22\textwidth}
        \centering
        \includegraphics[width=\textwidth]{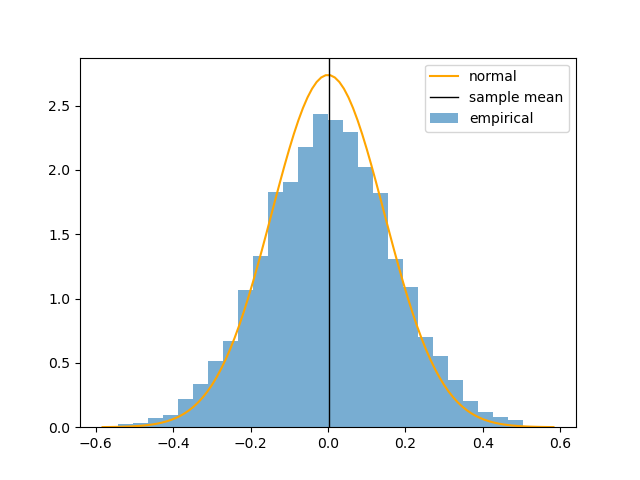}
        \caption*{\ipw, Arm $1$}
    \end{minipage}
    \caption{SGD on linear regression with exponential policy and different weights in the non-degenerate model. We report the empirical distribution of each action's first dimension of $\sqrt{t}(\bar{\theta}_t - \theta^*)$ for $10,000$ Monte-Carlo simulations.}
    \label{fig-exp3-non}
\end{figure}

\begin{figure}[H]
    \centering
    \begin{minipage}{0.22\textwidth}
        \centering
        \includegraphics[width=\textwidth]{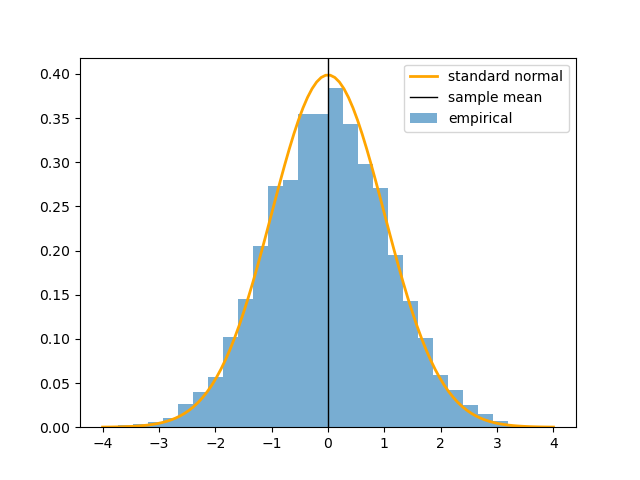}
        \caption*{\vanilla, Arm $0$}
    \end{minipage}
    \begin{minipage}{0.22\textwidth}
        \centering
        \includegraphics[width=\textwidth]{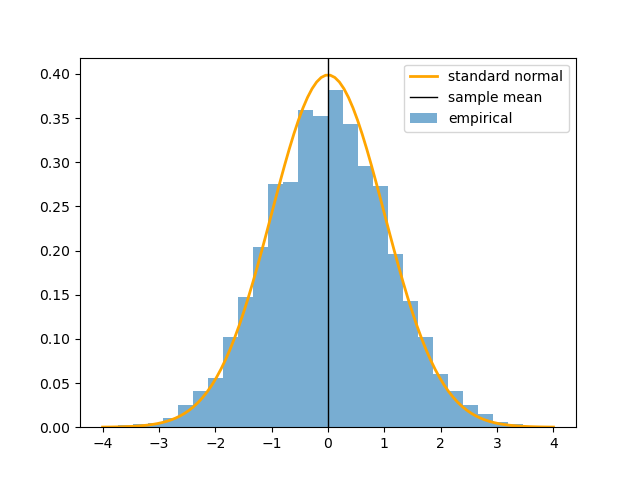}
        \caption*{\sipw, Arm $0$}
    \end{minipage}
    \begin{minipage}{0.22\textwidth}
        \centering
        \includegraphics[width=\textwidth]{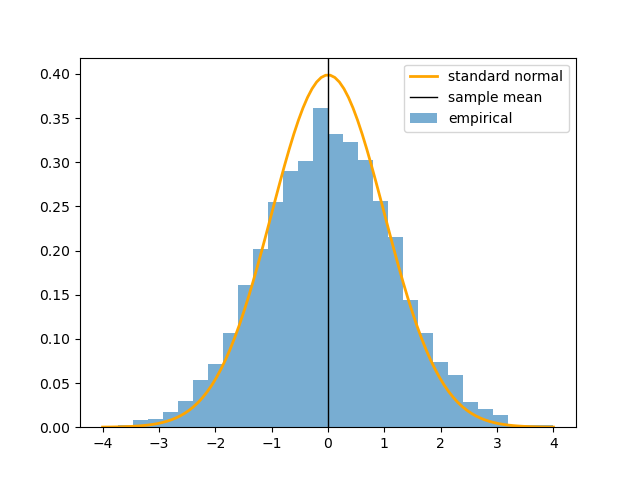}
        \caption*{\ipw, Arm $0$}
    \end{minipage}\\
    \begin{minipage}{0.22\textwidth}
        \centering
        \includegraphics[width=\textwidth]{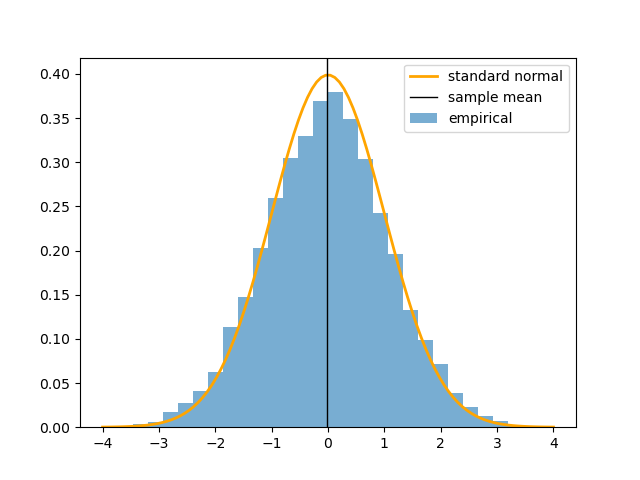}
        \caption*{\vanilla, Arm $1$}
    \end{minipage}
    \begin{minipage}{0.22\textwidth}
        \centering
        \includegraphics[width=\textwidth]{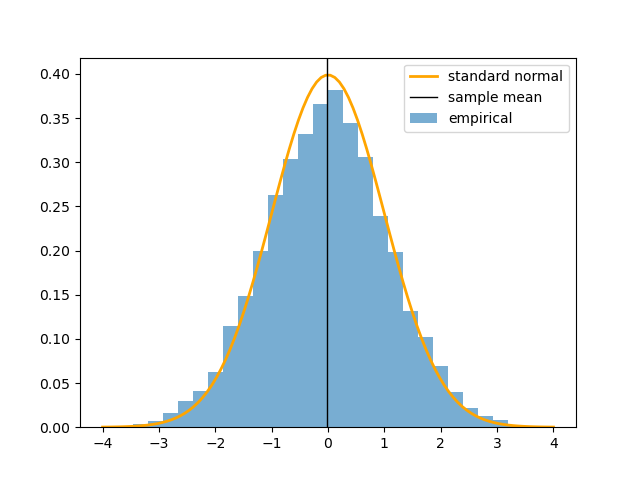}
        \caption*{\sipw, Arm $1$}
    \end{minipage}
    \begin{minipage}{0.22\textwidth}
        \centering
        \includegraphics[width=\textwidth]{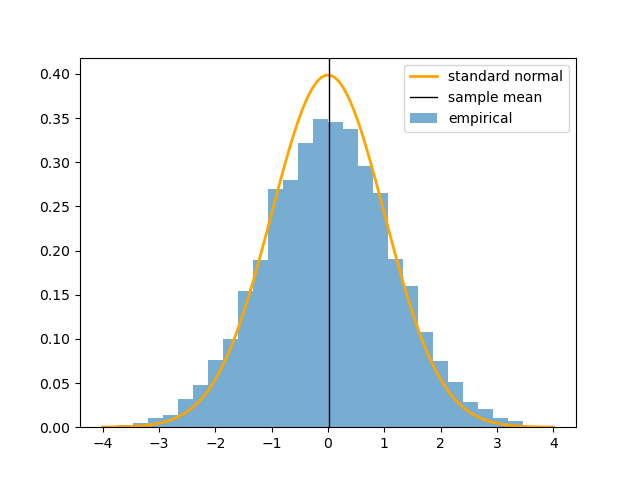}
        \caption*{\ipw, Arm $1$}
    \end{minipage}
    \caption{SGD on linear regression with exponential policy and different weights in the non-degenerate model. We report the empirical distribution of each action's first dimension of $\sqrt{t}\hat{S}_t^{-1/2}\hat{H}_t(\bar{\theta}_t - \theta^*)$ for $10,000$ Monte-Carlo simulations.}
    \label{fig-exp3-std-non}
\end{figure}

\begin{figure}[H]
    \centering
    \begin{minipage}{0.22\textwidth}
        \centering
        \includegraphics[width=\textwidth]{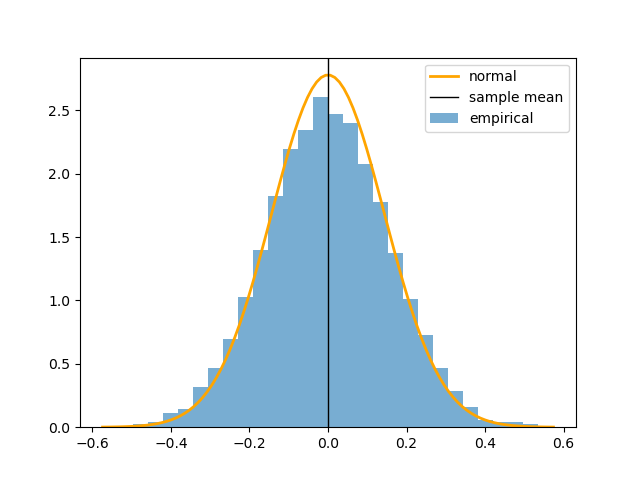}
        \caption*{\vanilla, Arm $0$}
    \end{minipage}
    \begin{minipage}{0.22\textwidth}
        \centering
        \includegraphics[width=\textwidth]{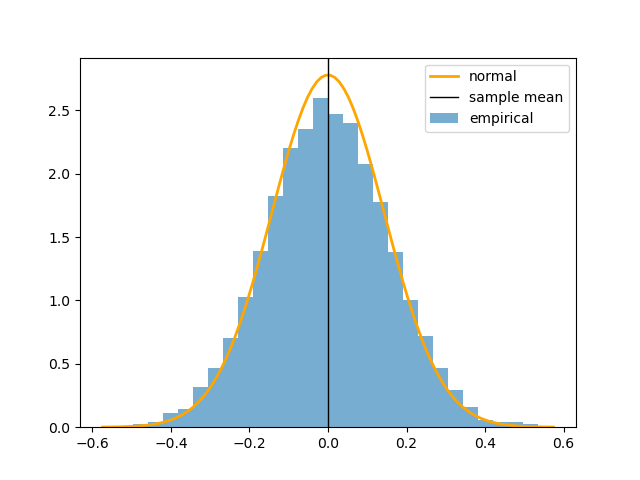}
        \caption*{\sipw, Arm $0$}
    \end{minipage}
    \begin{minipage}{0.22\textwidth}
        \centering
        \includegraphics[width=\textwidth]{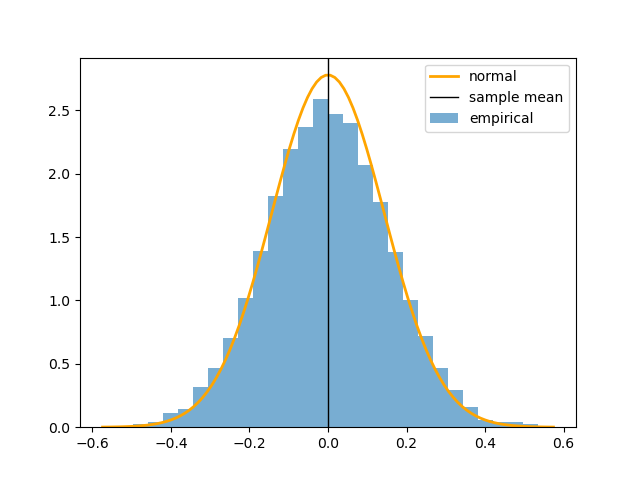}
        \caption*{\ipw, Arm $0$}
    \end{minipage}\\
    \begin{minipage}{0.22\textwidth}
        \centering
        \includegraphics[width=\textwidth]{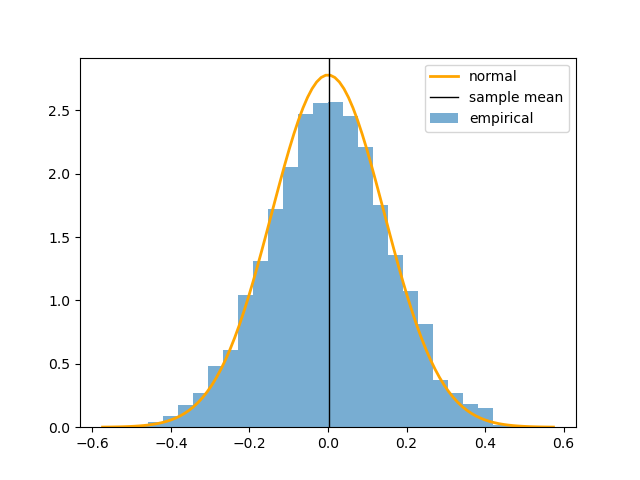}
        \caption*{\vanilla, Arm $1$}
    \end{minipage}
    \begin{minipage}{0.22\textwidth}
        \centering
        \includegraphics[width=\textwidth]{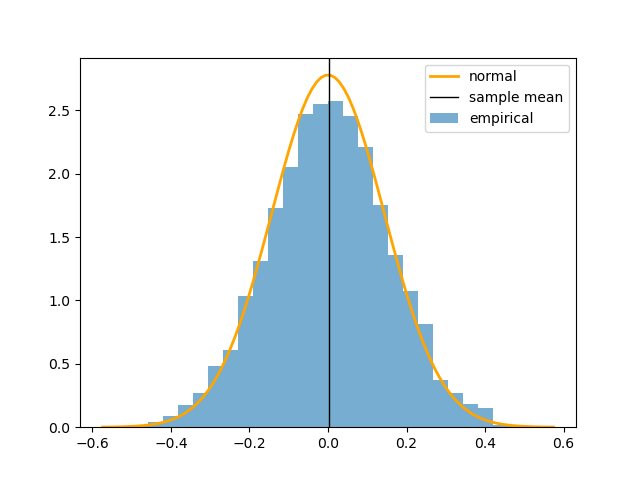}
        \caption*{\sipw, Arm $1$}
    \end{minipage}
    \begin{minipage}{0.22\textwidth}
        \centering
        \includegraphics[width=\textwidth]{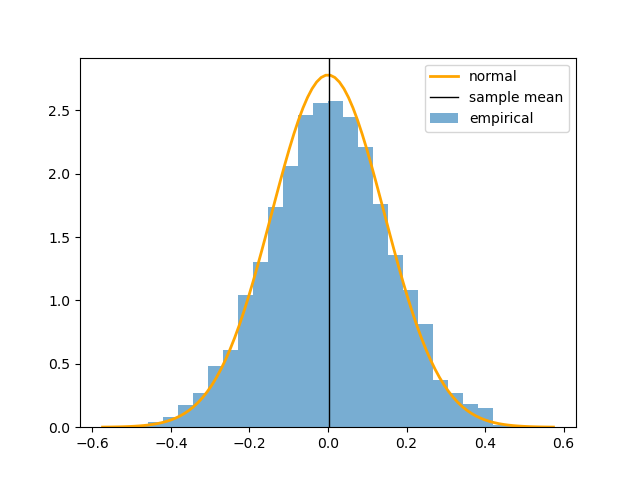}
        \caption*{\ipw, Arm $1$}
    \end{minipage}
    \caption{SGD on linear regression with exponential policy and different weights in the degenerate model. We report the empirical distribution of each action's first dimension of $\sqrt{t}(\bar{\theta}_t - \theta^*)$ for $10,000$ Monte-Carlo simulations.}
    \label{fig-exp3-de}
\end{figure}

\begin{figure}[H]
    \centering
    \begin{minipage}{0.22\textwidth}
        \centering
        \includegraphics[width=\textwidth]{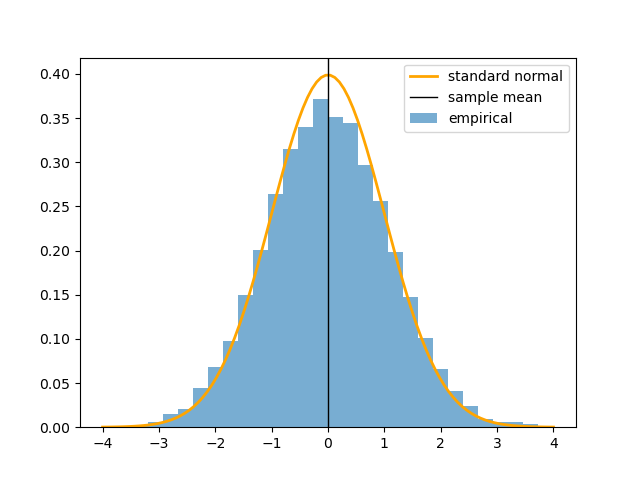}
        \caption*{\vanilla, Arm $0$}
    \end{minipage}
    \begin{minipage}{0.22\textwidth}
        \centering
        \includegraphics[width=\textwidth]{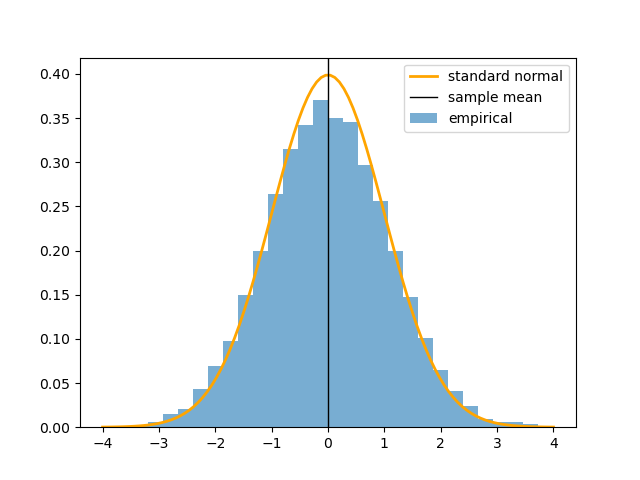}
        \caption*{\sipw, Arm $0$}
    \end{minipage}
    \begin{minipage}{0.22\textwidth}
        \centering
        \includegraphics[width=\textwidth]{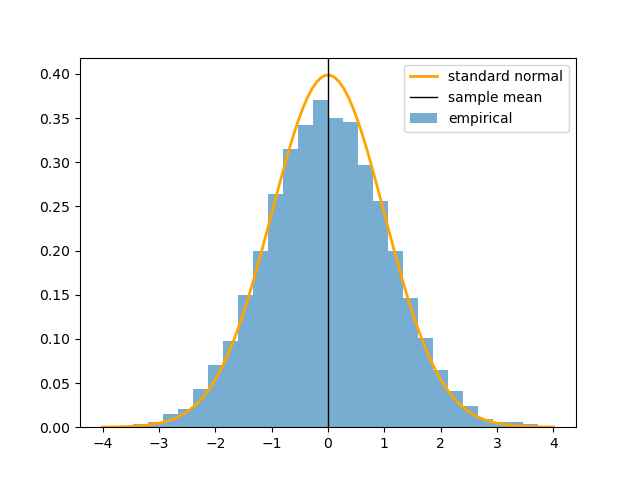}
        \caption*{\ipw, Arm $0$}
    \end{minipage}\\
    \begin{minipage}{0.22\textwidth}
        \centering
        \includegraphics[width=\textwidth]{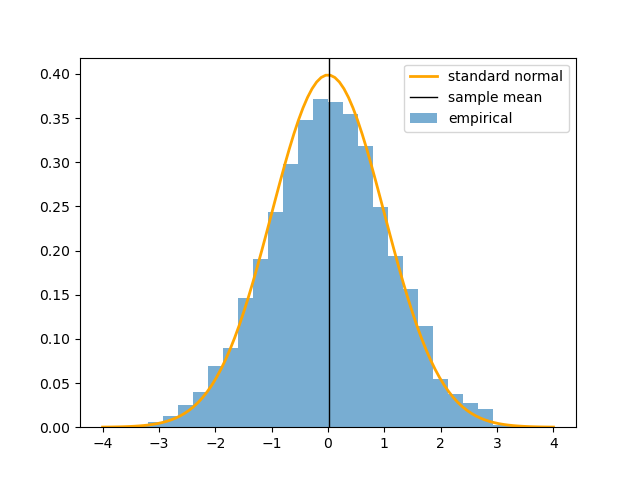}
        \caption*{\vanilla, Arm $1$}
    \end{minipage}
    \begin{minipage}{0.22\textwidth}
        \centering
        \includegraphics[width=\textwidth]{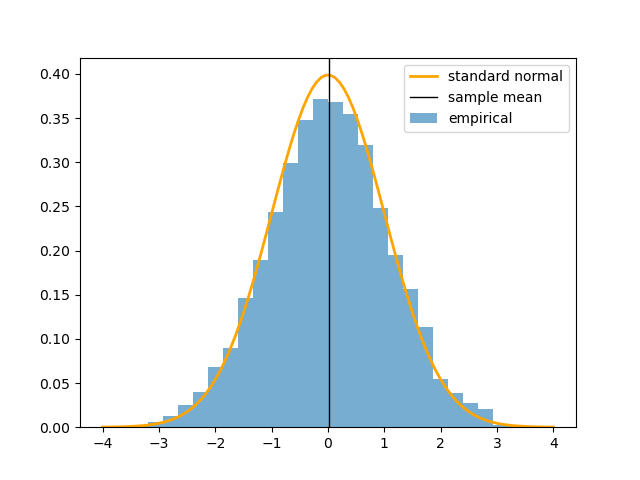}
        \caption*{\sipw, Arm $1$}
    \end{minipage}
    \begin{minipage}{0.22\textwidth}
        \centering
        \includegraphics[width=\textwidth]{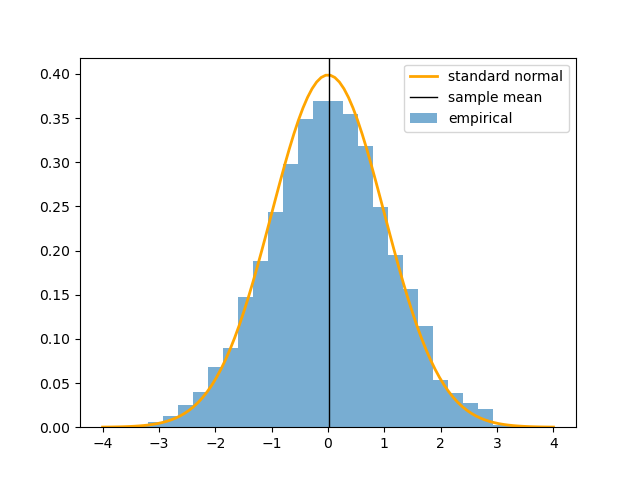}
        \caption*{\ipw, Arm $1$}
    \end{minipage}
    \caption{SGD on linear regression with exponential policy and different weights in the degenerate model. We report the empirical distribution of each action's first dimension of $\sqrt{t}\hat{S}_t^{-1/2}\hat{H}_t(\bar{\theta}_t - \theta^*)$ for $10,000$ Monte-Carlo simulations.}
    \label{fig-exp3-std-de}
\end{figure}
{
\subsection{Addtional figures}
\vspace{-1em}
\begin{figure}[H]
    \centering
    \includegraphics[width=1\linewidth]{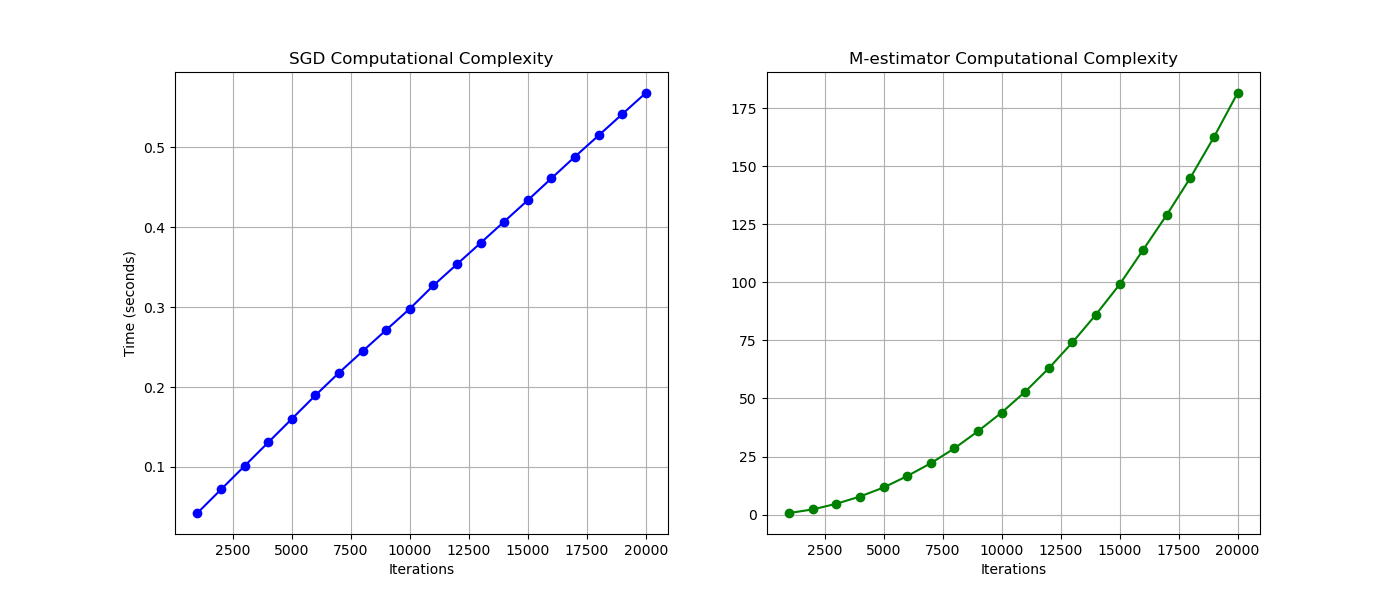}
    \caption{Comparison of running time between SGD and M-estimators in the linear regression settings. The parameter dimension is fixed at 10, with time recorded every 1000 iterations up to a total of 20,000 iterations.}
    \label{fig:Time}
\end{figure}

\begin{figure}[H]
    \centering
    \includegraphics[width=0.85\linewidth]{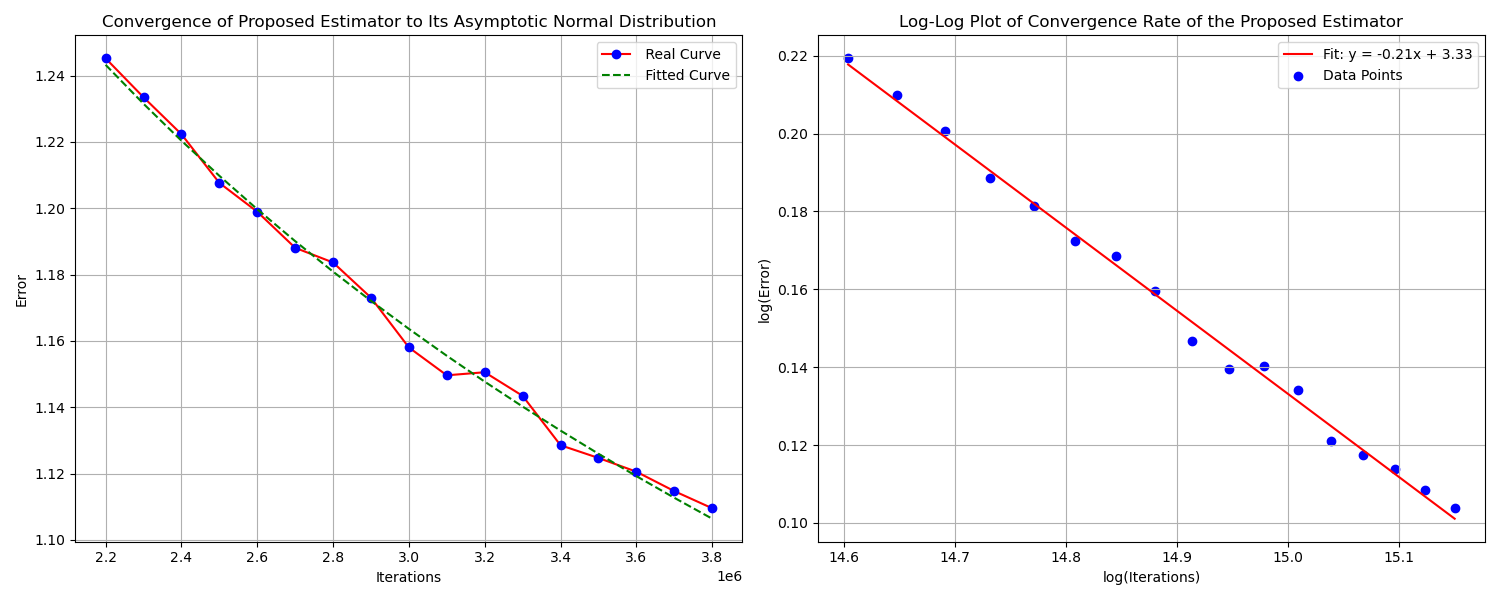}
    \caption{(Left) The average absolute error $ | \sqrt{t} \Sigma^{-1/2} (\bar{\theta}_t - \theta^*) - W |$, plotted on a log-log scale, based on 3000 Monte Carlo simulations with the where $t$ ranges from $2.2$ to $3.8$ million. (Right) A linear regression fit applied to the data presented in the left panel, resulting in $y=-0.21x+3.33$. }

    \label{fig:-0.2}
\end{figure}

\begin{figure}[H]
    \centering
    \includegraphics[width=0.6\linewidth]{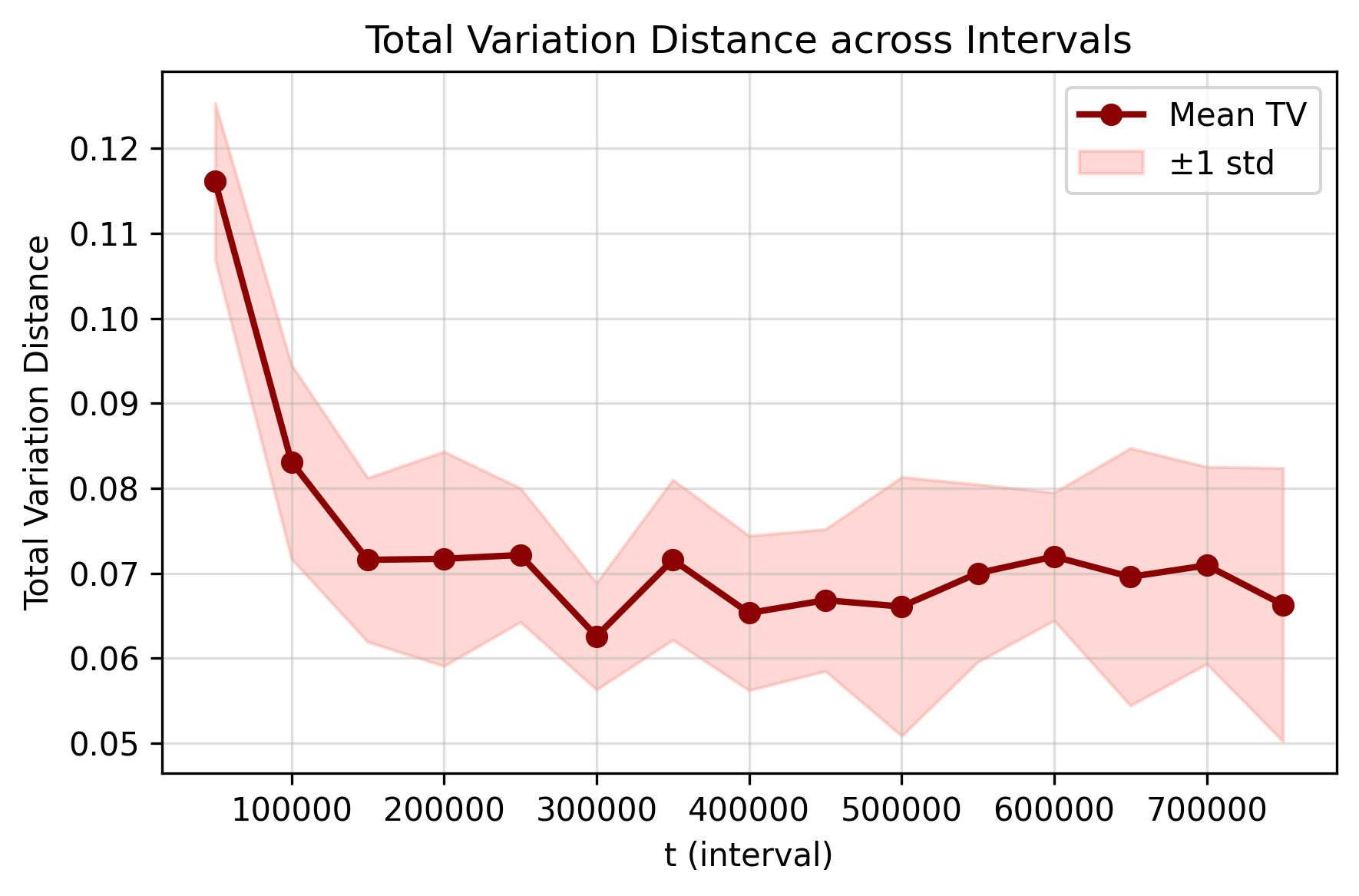}
    \caption{Evolution of total variation distance across time intervals for linear regression with modified $\varepsilon$-greedy in the degenerate model.}

    \label{fig:avg tv}
\end{figure}

}

\subsection{Additional tables}
\begin{table}[H]
  \centering
  \small
  \begin{tabular}{c|ccccc}
    \toprule
    \multirow{2}{*}{Weight \& Arm} & Sample size & Plug-in  & Oracle  & Plug-in & Oracle \\ 
     & Size & Coverage  & Coverage  & Lengths & Lengths \\ \hline
    & $1\times 10^5$ & 0.916 (0.090) & 0.907 (0.094) & 0.570 (0.007) & 0.554 \\
    \multirow{-2}{*}{\vanilla, Arm 0} & $2\times 10^5$ & 0.934 (0.076) & 0.930 (0.080) & 0.563 (0.004) & 0.554 \\
    & $1\times 10^5$ & 0.916 (0.088) & 0.907 (0.093) & 0.571 (0.007) & 0.554 \\
    \multirow{-2}{*}{\vanilla, Arm 1} & $2\times 10^5$ & 0.929 (0.080) & 0.926 (0.087) & 0.563 (0.004) & 0.554 \\ \hline
    & $1\times 10^5$ & 0.910 (0.091) & 0.900 (0.097) & 0.741 (0.029) & 0.716 \\
    \multirow{-2}{*}{\sipw, Arm 0} & $2\times 10^5$ & 0.929 (0.084) & 0.923 (0.084) & 0.729 (0.017) & 0.716 \\
    
    & $1\times 10^5$ & 0.917 (0.087) & 0.906 (0.091) & 0.744 (0.041) & 0.716 \\
    
    \multirow{-2}{*}{\sipw, Arm 1} & $2\times 10^5$ & 0.926 (0.086) & 0.922 (0.086) & 0.730 (0.020) & 0.716 \\ \hline
    & $1\times 10^5$ & 0.867 (0.131) & 0.654 (0.318) & 36.898 (529.276) & 2.786 \\
    \multirow{-2}{*}{\ipw, Arm 0} & $2\times 10^5$ & 0.881 (0.122) & 0.696 (0.315) & 14.047 (51.389) & 2.786 \\
    & $1\times 10^5$ & 0.878 (0.129) & 0.674 (0.322) & 26.785 (169.067) & 2.786 \\
    \multirow{-2}{*}{\ipw, Arm 1} & $2\times 10^5$ & 0.882 (0.131) & 0.718 (0.314) & 21.004 (145.044) & 2.786 \\
    \bottomrule
  \end{tabular}
  \caption{{Inference results of non-degenerate linear regression with modified $\varepsilon$-greedy and different weighting schemes.} Averaged coverage rates and average lengths of the confidence intervals are reported for plug-in estimator and oracle estimator {with 95\% confidence level}. We also include standard error in the parentheses.}
  \label{table-cov-linear}
\end{table}

\begin{table}[H]
  \centering
  \small
  \begin{tabular}{c|ccccc}
    \toprule
    \multirow{2}{*}{Weight \& Arm} & Sample size & Plug-in  & Oracle  & Plug-in & Oracle \\ 
     & Size & Coverage  & Coverage  & Lengths & Lengths \\ \hline
    & $1\times 10^5$ & 0.915 (0.093) & 0.906 (0.097) & 0.571 (0.007) & 0.554 \\
    \multirow{-2}{*}{\vanilla, Arm 0} & $2\times 10^5$ & 0.928 (0.082) & 0.924 (0.084) & 0.563 (0.004) & 0.554 \\
    & $1\times 10^5$ & 0.916 (0.088) & 0.905 (0.095) & 0.571 (0.007) & 0.554 \\
    \multirow{-2}{*}{\vanilla, Arm 1} & $2\times 10^5$ & 0.931 (0.081) & 0.926 (0.085) & 0.563 (0.004) & 0.554 \\ \hline
    & $1\times 10^5$ & 0.895 (0.102) & 0.883 (0.107) & 0.575 (0.028) & 0.554 \\
    \multirow{-2}{*}{\sipw, Arm 0} & $2\times 10^5$ & 0.918 (0.084) & 0.912 (0.087) & 0.565 (0.016) & 0.554 \\
    
    & $1\times 10^5$ & 0.894 (0.101) & 0.880 (0.107) & 0.576 (0.031) & 0.554 \\
    
    \multirow{-2}{*}{\sipw, Arm 1} & $2\times 10^5$ & 0.920 (0.087) & 0.913 (0.092) & 0.567 (0.018) & 0.554 \\ \hline
    & $1\times 10^5$ & 0.789 (0.147) & 0.445 (0.340) & 184.703 (4857.167) & 0.554 \\
    \multirow{-2}{*}{\ipw, Arm 0} & $2\times 10^5$ & 0.810 (0.148) & 0.511 (0.348) & 39.979 (503.013) & 0.554 \\
    & $1\times 10^5$ & 0.785 (0.152) & 0.445 (0.340) & 46.074 (708.325) & 0.554 \\
    \multirow{-2}{*}{\ipw, Arm 1} & $2\times 10^5$ & 0.812 (0.147) & 0.502 (0.352) & 76.530 (1700.863) & 0.554 \\
    \bottomrule
  \end{tabular}
  \caption{{Inference results of degenerate linear regression with modified $\varepsilon$-greedy and different weighting schemes.} Averaged coverage rates and average lengths of the confidence intervals are reported for plug-in estimator and oracle estimator {with 95\% confidence level}. We also include standard error in the parentheses.}
  \label{table-cov-linear1}
\end{table}

\begin{table}[H]
  \centering
  \small
  {\begin{tabular}{c|c|cc|cc}
    \toprule
    \multirow{2}{*}{Weight \& Arm} & \multirow{2}{*}{Sample size} &
    \multicolumn{2}{c|}{\textbf{Non-degenerate}} &
    \multicolumn{2}{c}{\textbf{Degenerate}} \\ 
     &  & Coverage & Lengths & Coverage & Lengths \\ \hline
    & $1\times 10^5$ & 0.934 (0.079) & 0.563 (0.004) & 0.932 (0.081) & 0.563 (0.004) \\
    \multirow{-2}{*}{\vanilla, Arm 0} & $2\times 10^5$ & 0.931 (0.079) & 0.563 (0.004) & 0.930 (0.079) & 0.563 (0.004) \\
    & $1\times 10^5$ & 0.927 (0.083) & 0.563 (0.004) & 0.931 (0.081) & 0.563 (0.004) \\
    \multirow{-2}{*}{\vanilla, Arm 1} & $2\times 10^5$ & 0.935 (0.079) & 0.563 (0.004) & 0.933 (0.079) & 0.563 (0.004) \\ \hline
    & $1\times 10^5$ & 0.929 (0.082) & 0.611 (0.007) & 0.930 (0.082) & 0.563 (0.015) \\
    \multirow{-2}{*}{\sipw, Arm 0} & $2\times 10^5$ & 0.928 (0.078) & 0.612 (0.010) & 0.928 (0.082) & 0.563 (0.005) \\
    & $1\times 10^5$ & 0.928 (0.080) & 0.612 (0.008) & 0.929 (0.081) & 0.563 (0.005) \\
    \multirow{-2}{*}{\sipw, Arm 1} & $2\times 10^5$ & 0.932 (0.081) & 0.612 (0.011) & 0.931 (0.079) & 0.563 (0.007) \\ \hline
    & $1\times 10^5$ & 0.927 (0.094) & 1.342 (0.932) & 0.931 (0.083) & 0.577 (0.237) \\
    \multirow{-2}{*}{\ipw, Arm 0} & $2\times 10^5$ & 0.933 (0.093) & 1.324 (1.109) & 0.930 (0.080) & 0.573 (0.050) \\
    & $1\times 10^5$ & 0.927 (0.089) & 1.328 (0.936) & 0.927 (0.083) & 0.577 (0.118) \\
    \multirow{-2}{*}{\ipw, Arm 1} & $2\times 10^5$ & 0.937 (0.086) & 1.296 (0.440) & 0.933 (0.078) & 0.576 (0.093) \\
    \bottomrule
  \end{tabular}}
  \caption{{Inference results for non-degenerate and degenerate linear regression with exponential policy under different weighting schemes.} 
  Averaged coverage rates and average lengths of confidence intervals are reported for plug-in estimator {with 95\% confidence level}. Standard errors are shown in parentheses.}
  \label{table-cov-linear-combined}
\end{table}

\begin{table}[H]
  \centering
  \small
  {\begin{tabular}{c|c|cc|cc}
    \toprule
    \multirow{2}{*}{Weight \& Arm} & \multirow{2}{*}{Sample size} &
    \multicolumn{2}{c|}{\textbf{Non-degenerate}} &
    \multicolumn{2}{c}{\textbf{Degenerate}} \\ 
     &  & Coverage & Lengths & Coverage & Lengths \\ \hline
    & $1\times 10^6$ & 0.933 (0.079) & 0.755 & 0.938 (0.077) & 0.755 \\
    \multirow{-2}{*}{\vanilla, Arm 0} & $2\times 10^6$ & 0.938 (0.082) & 0.755 & 0.941 (0.072) & 0.755 \\
    & $1\times 10^6$ & 0.936 (0.079) & 0.755 & 0.934 (0.077) & 0.755 \\
    \multirow{-2}{*}{\vanilla, Arm 1} & $2\times 10^6$ & 0.940 (0.077) & 0.755 & 0.941 (0.072) & 0.755 \\ \hline
    & $1\times 10^6$ & 0.929 (0.080) & 0.976 & 0.922 (0.088) & 0.755 \\
    \multirow{-2}{*}{\sipw, Arm 0} & $2\times 10^6$ & 0.938 (0.075) & 0.976 & 0.940 (0.075) & 0.755 \\
    & $1\times 10^6$ & 0.932 (0.081) & 0.976 & 0.923 (0.085) & 0.755 \\
    \multirow{-2}{*}{\sipw, Arm 1} & $2\times 10^6$ & 0.940 (0.076) & 0.976 & 0.933 (0.080) & 0.755 \\ \hline
    & $1\times 10^6$ & 0.679 (0.222) & 3.796 & 0.206 (0.168) & 0.755 \\
    \multirow{-2}{*}{\ipw, Arm 0} & $2\times 10^6$ & 0.781 (0.185) & 3.796 & 0.273 (0.202) & 0.755 \\
    & $1\times 10^6$ & 0.678 (0.221) & 3.796 & 0.201 (0.167) & 0.755 \\
    \multirow{-2}{*}{\ipw, Arm 1} & $2\times 10^6$ & 0.760 (0.187) & 3.796 & 0.273 (0.197) & 0.755 \\
    \bottomrule
  \end{tabular}}
  \caption{{Inference results for non-degenerate and degenerate quantile regression with modified $\varepsilon$-greedy policy under different weighting schemes.} Averaged coverage rates under 
 oracle confidence intervals are reported(95\% confidence level). Standard errors are shown in parentheses.}
  \label{table-cov-linear-combined1}
\end{table}

\spacingset{1.5}
\begin{table}[H]
  \centering
  \small
  \begin{tabular}{cccccccc}
    \toprule
    Weight \& Arm & Parameter & Estimate & S.E. & 95\% LB & 95\% UB & $t$-value & $p$-value \\ \hline
    & $\theta_{1}$ & -2.56 & 0.04 & -2.64 & -2.48 & -65.52 & 0.00 \\
    & $\theta_{2}$ & -0.26 & 0.08 & -0.43 & -0.10 & -3.11 & 0.00 \\
    & $\theta_{3}$ & -0.48 & 0.07 & -0.62 & -0.34 & -6.80 & 0.00 \\
    & $\theta_{4}$ & -0.23 & 0.06 & -0.34 & -0.12 & -4.09 & 0.00 \\
    \multirow{-5}{*}{\vanilla, Arm 0} & $\theta_{5}$ & -0.90 & 0.07 & -1.03 & -0.77 & -13.65 & 0.00 \\ \hline
    & $\theta_{6}$ & -2.55 & 0.05 & -2.65 & -2.44 & -47.77 & 0.00 \\
    & $\theta_{7}$ & -0.24 & 0.08 & -0.40 & -0.09 & -3.06 & 0.00 \\
    & $\theta_{8}$ & -0.45 & 0.07 & -0.58 & -0.32 & -6.76 & 0.00 \\
    & $\theta_{9}$ & -0.41 & 0.11 & -0.62 & -0.19 & -3.71 & 0.00 \\
    \multirow{-5}{*}{\vanilla, Arm 1} & $\theta_{10}$ & -0.91 & 0.07 & -1.05 & -0.77 & -12.31 & 0.00 \\ \hline
    & $\theta_{1}$ & -2.52 & 0.05 & -2.62 & -2.43 & -52.85 & 0.00 \\
    & $\theta_{2}$ & -0.30 & 0.11 & -0.51 & -0.09 & -2.79 & 0.01 \\
    & $\theta_{3}$ & -0.49 & 0.09 & -0.66 & -0.31 & -5.56 & 0.00 \\
    & $\theta_{4}$ & -0.28 & 0.07 & -0.4 & -0.15 & -4.25 & 0.00 \\
    \multirow{-5}{*}{\sipw, Arm 0} & $\theta_{5}$ & -0.80 & 0.09 & -0.97 & -0.63 & -9.33 & 0.00 \\ \hline
    & $\theta_{6}$ & -2.51 & 0.05 & -2.61 & -2.41 & -49.35 & 0.00 \\
    & $\theta_{7}$ & -0.28 & 0.08 & -0.43 & -0.13 & -3.60 & 0.00 \\
    & $\theta_{8}$ & -0.45 & 0.06 & -0.58 & -0.33 & -7.10 & 0.00 \\
    & $\theta_{9}$ & -0.42 & 0.11 & -0.63 & -0.20 & -3.83 & 0.00 \\
    \multirow{-5}{*}{\sipw, Arm 1} & $\theta_{10}$ & -0.81 & 0.07 & -0.94 & -0.68 & -12.02 & 0.00 \\ \hline
    & $\theta_{1}$ & -2.64 & 0.10 & -2.85 & -2.44 & -25.54 & 0.00 \\
    & $\theta_{2}$ & -0.28 & 0.19 & -0.64 & 0.08 & -1.51 & 0.13 \\
    & $\theta_{3}$ & -0.51 & 0.15 & -0.80 & -0.23 & -3.49 & 0.00 \\
    & $\theta_{4}$ & -0.24 & 0.16 & -0.55 & 0.07 & -1.54 & 0.12 \\
    \multirow{-5}{*}{\ipw, Arm 0} & $\theta_{5}$ & -0.91 & 0.16 & -1.23 & -0.59 & -5.64 & 0.00 \\ \hline
    & $\theta_{6}$ & -2.47 & 0.03 & -2.53 & -2.40 & -76.6 & 0.00 \\
    & $\theta_{7}$ & -0.22 & 0.06 & -0.33 & -0.11 & -3.83 & 0.00 \\
    & $\theta_{8}$ & -0.51 & 0.05 & -0.60 & -0.42 & -11.08 & 0.00 \\
    & $\theta_{9}$ & -0.37 & 0.05 & -0.47 & -0.27 & -7.40 & 0.00 \\
    \multirow{-5}{*}{\ipw, Arm 1} & $\theta_{10}$ & -0.88 & 0.05 & -0.98 & -0.78 & -17.67 & 0.00 \\
    \bottomrule
  \end{tabular}
  \caption{Real data analysis with online inference using $\varepsilon$-greedy algorithm with $\varepsilon = 0.2$.}
  \label{table-data-1}
\end{table}

\begin{table}[H]
  \centering
  \small
  \begin{tabular}{cccccccc}
    \toprule
    Weight \& Arm & Parameter & Estimate & S.E. & 95\% LB & 95\% UB & $t$-value & $p$-value \\ \hline
    & $\theta_{1}$ & -2.55 & 0.04 & -2.63 & -2.48 & -68.62 & 0.00 \\
    & $\theta_{2}$ & -0.31 & 0.09 & -0.47 & -0.14 & -3.61 & 0.00 \\
    & $\theta_{3}$ & -0.45 & 0.07 & -0.6 & -0.31 & -6.18 & 0.00 \\
    & $\theta_{4}$ & -0.23 & 0.05 & -0.33 & -0.12 & -4.29 & 0.00 \\
    \multirow{-5}{*}{\vanilla, Arm 0} & $\theta_{5}$ & -0.88 & 0.07 & -1.01 & -0.75 & -13.45 & 0.00 \\ \hline
    & $\theta_{6}$ & -2.54 & 0.06 & -2.66 & -2.42 & -41.76 & 0.00 \\
    & $\theta_{7}$ & -0.29 & 0.09 & -0.45 & -0.12 & -3.36 & 0.00 \\
    & $\theta_{8}$ & -0.42 & 0.07 & -0.57 & -0.28 & -5.88 & 0.00 \\
    & $\theta_{9}$ & -0.42 & 0.19 & -0.79 & -0.04 & -2.18 & 0.03 \\
    \multirow{-5}{*}{\vanilla, Arm 1} & $\theta_{10}$ & -0.89 & 0.08 & -1.04 & -0.73 & -11.25 & 0.00 \\ \hline
    & $\theta_{1}$ & -2.49 & 0.05 & -2.58 & -2.40 & -54.74 & 0.00 \\
    & $\theta_{2}$ & -0.31 & 0.13 & -0.57 & -0.05 & -2.37 & 0.02 \\
    & $\theta_{3}$ & -0.45 & 0.12 & -0.68 & -0.21 & -3.74 & 0.00 \\
    & $\theta_{4}$ & -0.29 & 0.06 & -0.41 & -0.17 & -4.78 & 0.00 \\
    \multirow{-5}{*}{\sipw, Arm 0} & $\theta_{5}$ & -0.82 & 0.08 & -0.98 & -0.66 & -9.80 & 0.00 \\ \hline
    & $\theta_{6}$ & -2.48 & 0.08 & -2.64 & -2.33 & -31.13 & 0.00 \\
    & $\theta_{7}$ & -0.29 & 0.10 & -0.50 & -0.09 & -2.84 & 0.00 \\
    & $\theta_{8}$ & -0.42 & 0.09 & -0.60 & -0.25 & -4.69 & 0.00 \\
    & $\theta_{9}$ & -0.4 & 0.25 & -0.90 & 0.09 & -1.60 & 0.11 \\
    \multirow{-5}{*}{\sipw, Arm 1} & $\theta_{10}$ & -0.82 & 0.10 & -1.01 & -0.63 & -8.49 & 0.00 \\ \hline
    & $\theta_{1}$ & -2.75 & 0.33 & -3.40 & -2.11 & -8.37 & 0.00 \\
    & $\theta_{2}$ & -0.22 & 0.57 & -1.35 & 0.90 & -0.39 & 0.70 \\
    & $\theta_{3}$ & -0.80 & 0.50 & -1.78 & 0.18 & -1.59 & 0.11 \\
    & $\theta_{4}$ & 0.11 & 0.39 & -0.65 & 0.87 & 0.28 & 0.78 \\
    \multirow{-5}{*}{\ipw, Arm 0} & $\theta_{5}$ & -0.90 & 0.51 & -1.89 & 0.09 & -1.78 & 0.08 \\ \hline
    & $\theta_{6}$ & -2.40 & 0.09 & -2.57 & -2.23 & -27.81 & 0.00 \\
    & $\theta_{7}$ & -0.33 & 0.14 & -0.60 & -0.07 & -2.46 & 0.01 \\
    & $\theta_{8}$ & -0.33 & 0.08 & -0.48 & -0.17 & -4.17 & 0.00 \\
    & $\theta_{9}$ & -0.55 & 0.30 & -1.14 & 0.05 & -1.81 & 0.07 \\
    \multirow{-5}{*}{\ipw, Arm 1} & $\theta_{10}$ & -1.14 & 0.20 & -1.53 & -0.76 & -5.81 & 0.00 \\
    \bottomrule
  \end{tabular}
  \caption{Real data analysis with online inference using $\varepsilon$-greedy algorithm with $\varepsilon = 0.02$.}
  \label{table-data-2}
\end{table}

\spacingset{1.8}
\section{Proof of the general asymptotic normality result}
\label{sec:app-clt}

\subsubsection*{Proof of Theorem~\ref{thm:smooth-clt}}

\begin{proof}
By definition, the loss function can be written as
\begin{align*}
\L_{\theta^\prime}(\theta)& = \E[w(\theta^\prime; X, A^\prime)\ell(\theta; X, A^\prime, Y)]\\
& = \E_{\P_X, \pi(X, \theta^\prime)} \left\{\E_{\P_{Y \mid X,A'}}[w(\theta^\prime; X, A^\prime)\ell(\theta; X, A^\prime,Y)]\mid X, A^\prime \right\}.
\end{align*}
By Equation~\eqref{eq:oracle}, $\theta^{*}$ is the minimizer, i.e., $\theta^{*} \in \underset{\theta \in \R^d}{\operatorname{argmin}} \, \L_{\theta^\prime}(\theta)$. Because $\L$ is differentiable, we have $\nabla \L_{\theta^\prime}(\theta^*) = 0$. 
Moreover, we have
\begin{align} \label{eq:bound_grad}
\|\nabla \L_{\theta_{t-1}}(\theta_{t-1}) - H(\theta_{t-1} - \theta^*)\| & = \left\|\nabla \L_{\theta_{t-1}}(\theta_{t-1}) - \nabla \L_{\theta_{t-1}}(\theta^*) - H(\theta_{t-1} - \theta^*) \right\| \notag\\
& = \left\|\int_0^1 \left(\nabla^2 \L_{\theta_{t-1}}(\theta^* + s(\theta_{t-1} - \theta^*)) - H\right)(\theta_{t-1} - \theta^*) \d s \right\| \notag\\
& \leq K \|\theta_{t-1} - \theta^*\|^2,
\end{align}
{for large enough $t$ and $\|\theta_{t-1}-\theta^*\|\leq\delta$ as stated in Assumption \ref{assum:hessian}. Note \eqref{eq:bound_grad} ensures the asymptotic equivalence of two SGD processes, where the gradients for update are  $\nabla \L_{\theta_{t-1}}(\theta_{t-1})$ and $H(\theta_{t-1} - \theta^*)$, respectively. This part verifies Assumption 3.2 in \cite{polyak1992acceleration}.}

By Equation~\eqref{eq:xi}, we have $w(\theta^\prime; X, A)\nabla \ell(\theta; \zeta) = \xi_{\theta^\prime}(\theta; \zeta)  + \nabla \L_{\theta^\prime}(\theta)$. Notice that we have the following inequality by Assumption~\ref{assum:gram},
\begin{align*}
\E[\|w(\theta; X, A)\nabla \ell(\theta; \zeta)\|^2] \leq \overline{w}^2 \E[\|\nabla \ell(\theta; \zeta)\|^2] \leq \overline{w}^2 \kappa (1+\|\theta - \theta^*\|^2),
\end{align*}
{where we let $\kappa=\E [\phi(X)]<\infty$}.
Therefore, by the fact that $\E[\xi_{\theta}(\theta; \zeta)] = 0$, the two terms can be bounded by
\begin{align*}
\left\| \nabla \L_{\theta}(\theta)\right\|^2 &\leq \overline{w}^2 \kappa (1+\|\theta - \theta^*\|^2), \\
\E \left[ \|\xi_{\theta}(\theta; \zeta)\|^2 \right] &\leq \overline{w}^2 \kappa (1+\|\theta - \theta^*\|^2).
\end{align*}

The above bounds can already guarantee the almost surely convergence of $\theta_t$ by Theorem 2 of \cite{polyak1992acceleration}. {Note this property is not influenced by whether it is the degenerate model or not.} Now we need to quantify the difference between $\xi_{\theta_{t-1}}(\theta_{t-1}; \zeta_t)$ and $\xi_{\theta^*}(\theta^*; \zeta_t^*)$ where $\zeta_t=(X_t,A_t,Y_t)$, $\zeta_t^*=(X_t,A_t^*,Y_t^*)$. {Here $Y_t$ and $Y_t^*$ depends on $A_t$ and $A_t^*$ respectively as well as the same $X_t$, especially when $A_t^*=A_t$, $Y_t^*$ is exactly $Y_t$.} For example, in linear contextual bandits, $Y_t=X_t^\top\theta_{A_t}^*+\epsilon_t$, $Y_t^*=X_t^\top\theta_{A_t^*}^*+\epsilon_t$. 
Using the coupling we defined in Equation~\eqref{eq:coupling} and given $\theta_{t-1}$, it can be bounded by
\begin{align*}
&\E\|\xi_{\theta_{t-1}}(\theta_{t-1}; \zeta_t) - \xi_{\theta^\star}(\theta^*; \zeta_t^*)\|^2 \\
\leq \; & 2\|\nabla \L_{\theta_{t-1}}(\theta_{t-1}) - \nabla \L_{\theta^*}(\theta^*)\|^2  + 2 \E\left[\|w(\theta_{t-1}; X_t, A_t)\nabla \ell(\theta_{t-1}; \zeta_t) - w(\theta^*; X_t, A_t^*)\nabla \ell(\theta^*; X_t, A_t^*, Y_t^*)\|^2 \right]\\
:= & 2M_1 + 2M_2.
\end{align*}
From \eqref{eq:bound_grad}, we have the following bound for $M_1$ for large enough $t$,
\begin{align}
\label{eq:bound_M1}
M_1 \leq 2 K^2 \| \theta_{t-1} - \theta^* \|^4 + 2 \| H \|^2 \| \theta_{t-1} - \theta^* \|^2.
\end{align}
Therefore, as $\theta_{t-1}$ converges to $\theta^*$, $M_1$ converges to $0$.

The second term has the following inequality,
\begin{align} \label{eq:bound_M2}
M_2 & \leq \E\left[\Delta(X_t, \theta_{t-1})M_3 + (1-\Delta(X_t, \theta_{t-1}))M_4\right],
\end{align}
where
\begin{align*}
M_3 &:= \E\left[ \left\|w(\theta_{t-1}; X_t, A_t)\nabla \ell(\theta_{t-1}; X_t, A_t, Y_t) - w(\theta^*; X_t, A^*_t)\nabla \ell(\theta^*; X_t, A^*_t, Y_t^*)\right\|^2\mid X_t, A_t \neq A^*_t \right], \\
M_4 &:= \E \left[ \left\|w(\theta_{t-1}; X_t, A_t)\nabla \ell(\theta_{t-1}; X_t, A_t, Y_t) - w(\theta^*; X_t, A^*_t)\nabla \ell(\theta^*; X_t, A^*_t, Y_t)\right\|^2\mid X_t, A_t=A^*_t \right]
\end{align*}
The third term $M_3$ can be bounded as,
\begin{align} \label{eq:bound_M3}
M_3 & \leq 2 \overline{w}^2 \E \left[\|\nabla \ell(\theta_{t-1}; X_t, A_t, Y_t)\|^2 +\|\nabla \ell(\theta^*; X_t, A^*_t, Y^*_t)\|^2\mid X_t, A_t \neq A^*_t \right]\notag\\
&\leq 2 \overline{w}^2\E\|\nabla \ell(\theta_{t-1}; X_t, A_t, Y_t)\|^2+2 \overline{w}^2\E\|\nabla \ell(\theta^*; X_t, A^*_t, Y^*_t)\|^2\notag\\
& \leq 4\overline{w}^2 (1 + \|\theta_{t-1} - \theta^*\|^2)\phi(X_t).
\end{align}
Finally, we have
\begin{align} \label{eq:bound_M4}
M_4 & \leq \max_{A\in \A} \E\left[\|w(\theta_{t-1}; X_t, A)\nabla \ell(\theta_{t-1}; X_t, A, Y_t) - w(\theta^*; X_t, A)\nabla \ell(\theta^*; X_t, A, Y_t)\|^2\mid X_t, A \right] \notag\\
& \leq 2\max_{A\in \A} \E\left[\|w(\theta_{t-1}; X_t, A)\nabla \ell(\theta_{t-1}; X_t, A, Y_t) - w(\theta_{t-1}; X_t, A)\nabla \ell(\theta^*; X_t, A, Y_t)\|^2\mid X_t, A \right] \notag\\
& \quad +2\max_{A\in \A} \E \left[\|w(\theta_{t-1}; X_t, A)\nabla \ell(\theta^*; X_t, A, Y_t) - w(\theta^*; X_t, A)\nabla \ell(\theta^*; X_t, A, Y_t)\|^2\mid X_t, A \right] \notag\\
& \leq 2 \overline{w}^2 \max_{A\in \A} \E \left[\|\nabla \ell(\theta_{t-1}; X_t, A, Y_t) - \nabla \ell(\theta^*; X_t, A, Y_t)\|^2\mid X_t, A \right] \notag\\
& \quad +2\max_{A\in \A}|w(\theta_{t-1}; X_t, A) - w(\theta^*; X_t, A)|^2 \phi(X_t)
\end{align}
Combining \eqref{eq:bound_M3} and \eqref{eq:bound_M4} into \eqref{eq:bound_M2}, we have
\begin{align}
\label{eq:bound_M2b}
M_2 & \leq \E\left[\Delta(X_t, \theta_{t-1})4\overline{w}^2 (1 + \|\theta_{t-1} - \theta^*\|^2)\phi(X_t)\right] \notag \\
&\quad + \E\left[ 2 \overline{w}^2 \max_{A\in \A} \E [\|\nabla \ell(\theta_{t-1}; X_t, A, Y_t) - \nabla \ell(\theta^*; X_t, A, Y_t)\|^2\mid X_t, A] \right] \notag \\
&\quad + \E\left[ 2\max_{A\in \A} |w(\theta_{t-1}; X_t, A) - w(\theta^*; X_t, A)|^2 \phi(X_t)\right].
\end{align}
Using Assumption~\ref{assum:tv}, when $\theta_{t-1} \rightarrow \theta^*$, we have $M_2$ converges to $0$.

We can now conclude from our above results that
\begin{align*}
\lim_{\theta_{t-1} \rightarrow \theta^*}\E\|\xi_{\theta_{t-1}}(\theta_{t-1}; \zeta_t) - \xi_{\theta^*}(\theta^*; \zeta^*_t)\|^2 = 0,
\end{align*}
{which implies Assumption 3.3 in \cite{polyak1992acceleration} is satisfied. Particularly,}
{\begin{align*}
    \E\left(\xi_{\theta_{t-1}}(\theta_{t-1};\zeta_t)\xi_{\theta_{t-1}}(\theta_{t-1};\zeta_t)^\top\mid \mathcal{F}_{t-1}\right)\overset{p}{\rightarrow}S,\quad t\rightarrow\infty.
\end{align*}}
{Note that all three conditions in Theorem 2 of \cite{polyak1992acceleration} are verified under our assumptions, we can conclude that the asymptotic normality result in Theorem~\ref{thm:smooth-clt} holds, where the asymptotic covariance matrix is given by $H^{-1} S H^{-1}$, where $H = \nabla^2 \L_{\theta^*}(\theta^*)$ corresponds to $G$ in Theorem 2 of \cite{polyak1992acceleration}, and $S = \E[\xi_{\theta^*}(\theta^*; \zeta) \xi_{\theta^*}(\theta^*; \zeta)^\top]$ corresponds to $S$ in Theorem 2 of \cite{polyak1992acceleration}.}

\end{proof}

\section{Proof of Results in Linear Regression}\label{sec:supp-example}
\subsection{Proof of Corollary~\ref{corr:ls-reg}}\label{appendix:proof of cor ls}

\begin{proof}
Now let's compute $\L_{\theta_{t-1}}(\theta)$, under modified $\varepsilon$-greedy policy defined in \eqref{eq:modified eps-greedy},
\begin{align} \label{eq:modified loss-lr}
\L_{\theta_{t-1}}(\theta) &= \frac{1}{2} \E \left\{ \E_{\P_{Y\mid A,X}}\E_{\pi(X, \theta_{t-1})} \left[\varphi(\Pr(A \mid X, \theta_{t-1} \,)) \left((1-A)(Y - X^\top \theta_{0})^2 + A (Y - X^\top \theta_{1})^2\right) \mid X \right] \right\} \notag \\
&= \frac{1}{2}(1-\frac{\varepsilon}{2}) \varphi(1-\frac{\varepsilon}{2}) \E \left[ \ID \{\| \theta_{0,t-1} -  \theta_{1,t-1}\|> t^{-\frac{\alpha}{4}},X^\top \theta_{0,t-1}>  X^\top \theta_{1,t-1}\} \left(X^\top \theta^{*}_{0} - X^\top \theta_{0}\right)^2\right] \notag \\
&\quad + \frac{\varepsilon}{4} \varphi(\frac{\varepsilon}{2}) \E\left[ \ID \{\| \theta_{0,t-1} -  \theta_{1,t-1}\|> t^{-\frac{\alpha}{4}},X^\top \theta_{0,t-1} < X^\top \theta_{1,t-1}\} \left(X^\top \theta^{*}_{0} - X^\top \theta_{0}\right)^2 \right] \notag \\
&\quad + \frac{1}{2}(1-\frac{\varepsilon}{2}) \varphi(1-\frac{\varepsilon}{2}) \E\left[ \ID \{\| \theta_{0,t-1} -  \theta_{1,t-1}\|> t^{-\frac{\alpha}{4}},X^\top \theta_{0,t-1} < X^\top \theta_{1,t-1}\} \left(X^\top \theta^{*}_{1} - X^\top \theta_{1}\right)^2 \right] \notag \\
&\quad + \frac{\varepsilon}{4} \varphi(\frac{\varepsilon}{2}) \E \left[ \ID \{\| \theta_{0,t-1} -  \theta_{1,t-1}\|> t^{-\frac{\alpha}{4}},X^\top \theta_{0,t-1}>  X^\top \theta_{1,t-1}\} \left(X^\top \theta^{*}_{1} - X^\top \theta_{1}\right)^2 \right] \notag \\
&\quad +\frac{1}{4}\varphi(\frac{1}{2})\E\left[ \ID \{\| \theta_{0,t-1} -  \theta_{1,t-1}\|\leq t^{-\frac{\alpha}{4}}\} \left(X^\top \theta^{*}_{0} - X^\top \theta_{0}\right)^2 \right] \notag \\
&\quad +\frac{1}{4}\varphi(\frac{1}{2})\E\left[ \ID \{\| \theta_{0,t-1} -  \theta_{1,t-1}\|\leq t^{-\frac{\alpha}{4}}\} \left(X^\top \theta^{*}_{1} - X^\top \theta_{1}\right)^2 \right] \notag \\
&\quad + \sigma^2 \left[\frac{1}{2}(1-\frac{\varepsilon}{2}) \varphi(1-\frac{\varepsilon}{2}) + \frac{\varepsilon}{4} \varphi(\frac{\varepsilon}{2})+\frac{1}{4}\varphi(\frac{1}{2})\right].
\end{align}
{Obviously, replacing $\theta_{t-1}$ in the above equation with $\theta^\prime$ and repeating the same procedure, the first part of Assumption~\ref{assum:loss} is satisfied under this form of loss function $\L_{\theta^\prime}(\theta)$.}

Also, we can calculate the gradient of $\L$ with respect to $\theta$ as follows,
\begin{align*}
\nabla_{\theta_0} \L_{\theta_{t-1}}(\theta)
&= (1-\frac{\varepsilon}{2}) \varphi(1-\frac{\varepsilon}{2}) \E \left[ \ID \{\| \theta_{0,t-1} -  \theta_{1,t-1}\|> t^{-\frac{\alpha}{4}},X^\top \theta_{0,t-1}>  X^\top \theta_{1,t-1}\}  X X^\top \left(\theta_{0} - \theta^{*}_{0}\right)\right] \\
&\quad + \frac{\varepsilon}{2} \varphi(\frac{\varepsilon}{2}) \E\left[ \ID \{\| \theta_{0,t-1} -  \theta_{1,t-1}\|> t^{-\frac{\alpha}{4}},X^\top \theta_{0,t-1}< X^\top \theta_{1,t-1}\}  X X^\top \left(\theta_{0} - \theta^{*}_{0}\right) \right] \\
&\quad + \frac{1}{2} \varphi(\frac{1}{2}) \E\left[ \ID \{\| \theta_{0,t-1} -  \theta_{1,t-1}\|\leq t^{-\frac{\alpha}{4}}\}  X X^\top \left(\theta_{0} - \theta^{*}_{0}\right) \right] \\
\nabla_{\theta_1} \L_{\theta_{t-1}}(\theta)
&=  \frac{\varepsilon}{2} \varphi(\frac{\varepsilon}{2}) \E\left[ \ID \{\| \theta_{0,t-1} -  \theta_{1,t-1}\|> t^{-\frac{\alpha}{4}},X^\top \theta_{0,t-1}>  X^\top \theta_{1,t-1}\}  X X^\top \left(\theta_{1} - \theta^{*}_{1}\right) \right] \\
&\quad + (1-\frac{\varepsilon}{2}) \varphi(1-\frac{\varepsilon}{2}) \E \left[ \ID \{\| \theta_{0,t-1} -  \theta_{1,t-1}\|> t^{-\frac{\alpha}{4}},X^\top \theta_{0,t-1}< X^\top \theta_{1,t-1}\}  X X^\top \left(\theta_{1} - \theta^{*}_{1}\right)\right] \\
&\quad + \frac{1}{2} \varphi(\frac{1}{2}) \E\left[ \ID \{\| \theta_{0,t-1} -  \theta_{1,t-1}\|\leq t^{-\frac{\alpha}{4}}\}  X X^\top \left(\theta_{1} - \theta^{*}_{1}\right) \right].
\end{align*}
Therefore, the second part of Assumption~\ref{assum:loss} is naturally satisfied since replacing $\theta_{t-1}$ with $\theta$, we have the following,
{\begin{align*}
\langle \nabla \L_{\theta}(\theta), \theta - \theta^* \rangle \geq \min\left\{(1-\frac{\varepsilon}{2}) \varphi(1-\frac{\varepsilon}{2}), \frac{\varepsilon}{2} \varphi(\frac{\varepsilon}{2}),\frac{1}{2} \varphi(\frac{1}{2}) \right\} \, \E [X X^\top] \left[\left\|\theta_{0} - \theta^*_{0}\right\|^2 + \left\|\theta_{1} - \theta^*_{1}\right\|^2\right].
\end{align*}}
Obviously, from the definition of $\L_{\theta^\prime}(\theta)$, the Hessian matrix exists for all $(\theta, \theta^\prime) \in \R^d \times \R^d$, and the Hessian matrix at $(\theta^*, \theta^*)$ is positive definite since $\lambda_{\min} \E [X X^\top] > 0$, which verifies the first part of Assumption \ref{assum:hessian}.
We now consider the Hessian matrix under degenerate and non-degenerate models separately to validate \eqref{eq:hessian lipschitz} of Assumption \ref{assum:hessian}.

{\noindent\textbf{Non-degenerate model:}
     When $t$ is large enough,
    according to Lemma \ref{lem:non-degenerate-a.s.}, 
\begin{align*}
\nabla^2 \L_{\theta_{t-1}}(\theta) = \begin{bmatrix} H_{0,t-1} & 0 \\
0 & H_{1,t-1}
\end{bmatrix},
\end{align*}
where
\begin{align*}
H_{0,t-1} &= (1-\frac{\varepsilon}{2}) \varphi(1-\frac{\varepsilon}{2}) \E \left[ \ID \{X^\top \theta_{0,t-1}>  X^\top \theta_{1,t-1}\} X X^\top \right] \\
&\quad + \frac{\varepsilon}{2} \varphi(\frac{\varepsilon}{2}) \E\left[ \ID \{X^\top \theta_{0,t-1} < X^\top \theta_{1,t-1}\} X X^\top \right] \\
H_{1,t-1} &= \frac{\varepsilon}{2} \varphi(\frac{\varepsilon}{2}) \E \left[ \ID \{X^\top \theta_{0,t-1}>  X^\top \theta_{1,t-1}\} X X^\top \right] \\
&\quad + (1-\frac{\varepsilon}{2}) \varphi(1-\frac{\varepsilon}{2}) \E\left[ \ID \{X^\top \theta_{0,t-1} < X^\top \theta_{1,t-1}\} X X^\top \right].
\end{align*}
 We now check \eqref{eq:hessian lipschitz} in Assumption \ref{assum:hessian}, that is,
\begin{align*}
\left\| \nabla^2 \L_{\theta_{t-1}}(\theta) - \nabla^2 \L_{\theta^*}(\theta^*)\right\| \leq K \| \theta - \theta^* \| + K \|\theta_{t-1} - \theta^* \|,
\end{align*}
where $\L_{\theta^*}(\theta)$ is obtained by replacing $\theta_{a,t-1}$ in $H_{a,t-1}$ with $\theta_a^*$, $a=0,1$.
Note $\nabla^2 \L_{\theta_{t-1}}(\theta)$ is a constant function with respect to $\theta$, so we only need to consider its Lipschitz continuity with respect to $\theta_{t-1}$. 
For a smooth and rapidly decreasing integrable function $p(x)$, define the function $J(\theta) = \int \ID(\theta^\top x > 0) p(x) \d x$. It has been proven in Lemma \ref{lem:hessian derivative} that
\begin{align*}
\nabla J(\theta) = \int_{\theta^\top x = 0} p(x) x \d x.
\end{align*}

Apply this formula to $H_0$ and $H_1$, we get
\begin{align}\label{eq:thirdorder}
& \quad \frac{\partial }{\partial \theta_{0,t-1}} H_0 = - \frac{\partial }{\partial \theta_{1,t-1}} H_0 = - \frac{\partial }{\partial \theta_{0,t-1}}  H_1 = \frac{\partial }{\partial \theta_{1,t-1}}  H_1 \\
& = \left((1-\frac{\varepsilon}{2}) p(1-\frac{\varepsilon}{2})- \frac{\varepsilon}{2} p(\frac{\varepsilon}{2})\right) \int_{x^\top \theta_{0,t-1}= x^\top \theta_{1,t-1}} x \otimes x \otimes x p(x) \d x.
\end{align}
So equation \eqref{eq:hessian lipschitz} holds as long as $\int_{x^\top \theta^*_{0}= x^\top \theta^*_{1}} x \otimes x \otimes x p(x) \d x$ exists.

\noindent\textbf{Degenerate model:} When $t$ is large enough, according to Lemma \ref{lem:a.s.}, 
\begin{align*}
\nabla^2 \L_{\theta_{t-1}}(\theta) = \begin{bmatrix} \frac{1}{2}\varphi(\frac{1}{2})\E X X^\top & 0 \\
0 & \frac{1}{2}\varphi(\frac{1}{2})\E X X^\top
\end{bmatrix}.
\end{align*}
Its Lipschitz continuity clearly holds for both $\theta_{t-1}$ and $\theta$, which means \eqref{eq:hessian lipschitz} holds. In fact, when $t$ is large enough, 
\begin{align*}
    \|\nabla^2\L_{\theta_{t-1}}(\theta)-\nabla^2\L_{\theta^*}(\theta^*)\|=0\leq K \| \theta - \theta^* \| + K \|\theta_{t-1} - \theta^* \|.
\end{align*}

Combining the above results, we verified Assumption~\ref{assum:hessian} for both cases.}

For any $A, X$, we can bound
\begin{align*}
\E(\|\nabla \ell(\theta; \zeta)\|^2 \mid X, A) & \leq \|X\|^2\sigma^2 + \|X X^\top X X^\top\| \|\theta - \theta^*\|^2.
\end{align*}
So the first part of Assumption~\ref{assum:gram} is satisfied with $\phi(X) = \|X\|^2\sigma^2 + \|X X^\top \|^2$ for both cases.

We now consider the Gram matrix $S$ under degenerate and non-degenerate models separately to validate  Assumption \ref{assum:gram} and further Assumption \ref{assum:tv}.

{\noindent\textbf{Non-degenerate model:}
The matrix $S$ can be computed by
\begin{align*}
S = \E[w^2(\theta^*;X,A^*)\nabla \ell(\theta^*; \zeta^*) \nabla \ell(\theta^*; \zeta^*)^\top]
:= \begin{bmatrix} S_0 & 0 \\
0 & S_1
\end{bmatrix},
\end{align*}
where
\begin{align*}
S_0 &= (1-\frac{\varepsilon}{2}) \varphi^2(1-\frac{\varepsilon}{2})\sigma^2 \E \left[ \ID \{X^\top \theta^*_{0}>  X^\top \theta^*_{1}\}  X X^\top \right]\\
& \quad + \frac{\varepsilon}{2} \varphi^2(\frac{\varepsilon}{2}) \sigma^2 \E \left[ \ID \{X^\top \theta^*_{0}< X^\top \theta^*_{1}\}  X X^\top \right],\\
S_1 &= (1-\frac{\varepsilon}{2}) \varphi^2(1-\frac{\varepsilon}{2})\sigma^2 \E \left[ \ID \{X^\top \theta^*_{0}< X^\top \theta^*_{1}\}  X X^\top \right]\\
& \quad + \frac{\varepsilon}{2} \varphi^2(\frac{\varepsilon}{2}) \sigma^2 \E \left[ \ID \{X^\top \theta^*_{0}>  X^\top \theta^*_{1}\}  X X^\top \right].
\end{align*}
Assumption \ref{assum:gram} is satisfied now.
By definition
\begin{align}\label{eq:eps-non-deg-TV}
\Delta(X, \theta) = (1-\varepsilon)\left| \ID \left(X^\top \theta^*_{0}> X^\top \theta^*_{1}\right) - \ID \left(X^\top \theta_{0}> X^\top \theta_{1}\right) \right|.
\end{align}
Take any convergent sequence $\theta_t \rightarrow \theta^*$. It is clear that $\Delta(X, \theta_t)\phi(X)$ converges to 0 almost surely. Furthermore, $|\Delta(X, \theta_t)\phi(X)| \leq c\|X\|^4$ for some constant $c>0$ and $\E \|X\|^4 < \infty$. By dominated convergence theorem, $\lim_{t \rightarrow \infty} \E[\Delta(X, \theta_t)\phi(X)] = 0$. 
\begin{align}\label{eq:weight-diff-eps-non-deg}
&\quad|w(\theta_{t-1}; X_t, A) - w(\theta^*; X_t, A)|^2\notag \\
&=|\varphi(\Pr(A|X_t; \theta_{t-1})) - \varphi(\Pr(A|X_t; \theta^*))|^2\notag\\
&= |\varphi(1-\frac{\varepsilon}{2}) - \varphi(\frac{\varepsilon}{2})|^2\left| \ID \left(X^\top \theta^*_{0}> X^\top \theta^*_{1}\right) - \ID \left(X^\top \theta_{0,t-1}>  X^\top \theta_{1,t-1}\right) \right|.
\end{align}
Using the same argument as above, we can derive that
\begin{align*}
\lim_{\theta \rightarrow \theta^*} \E\left[|w(\theta; X, A) - w(\theta^*; X, A)|^2 \phi(X) \mid A \right] = 0.
\end{align*}
\noindent\textbf{Degenerate model:}
The matrix $S$ can be computed by
\begin{align*}
S = \E[w^2(\theta^*;X,A^*)\nabla \ell(\theta^*; \zeta^*) \nabla \ell(\theta^*; \zeta^*)^\top]
:= \begin{bmatrix} S_0 & 0 \\
0 & S_1
\end{bmatrix},
\end{align*}
where
\begin{align*}
S_0 &= \frac{1}{2} \varphi^2(\frac{1}{2}) \sigma^2 \E X X^\top,\\
S_1 &= \frac{1}{2} \varphi^2(\frac{1}{2}) \sigma^2 \E X X^\top.
\end{align*}
Assumption~\ref{assum:gram} is satisfied now. By definition,
\begin{align}\label{eq:eps-de-TV}
\Delta(X, \theta_{t-1}) = \frac{1-\varepsilon}{2} \ID \left(\|\theta_{0,t-1}-\theta_{1,t-1}\|>t^{-\frac{\alpha}{4}}\right),
\end{align}
and
\begin{align}\label{eq:weight-diff-eps-de}
|w(\theta_{t-1}; X_t, A) - w(\theta^*; X_t, A)|^2
\leq \left(|\varphi(1-\frac{\varepsilon}{2})-\varphi(\frac{1}{2})|+|\varphi(\frac{\varepsilon}{2})-\varphi(\frac{1}{2})|\right)^2\ID \left(\|\theta_{0,t-1}-\theta_{1,t-1}\|>t^{-\frac{\alpha}{4}}\right).
\end{align}
Take any convergent sequence $\theta_{t-1} \rightarrow \theta^*$,
since $\E\|X\|^{16}<\infty$, the condition in Lemma \ref{lem:8th rate} holds, that is, 
\begin{align*}
    \E(\|\nabla \ell(\theta;\zeta)\|^8\mid X,A)\leq C(1+\|\theta-\theta^*\|^8),
\end{align*}
which leads to the conclusion in Lemma \ref{lem:a.s.} that 
\begin{align*}
    \lim_{t\rightarrow\infty}\ID \left(\|\theta_{0,t-1}-\theta_{1,t-1}\|>t^{-\frac{\alpha}{4}}\right)=0,\  \text{a.s.}.
\end{align*}
  Using the same argument in the non-degenerate model, it is clear that $\Delta(X, \theta_{t-1})\phi(X)$ and $|w(\theta_{t-1}; X_t, A) - w(\theta^*; X_t, A)|^2\phi(X)$ converge to $0$ in expectation.
}

Finally, for both cases, we have the following inequality,
\begin{align*}
\E\left[\|\nabla \ell(\theta; \zeta) - \nabla \ell(\theta^*; \zeta)\|^2 \mid X,A \right] \leq \left\| X X^\top\right\|\cdot \left\|\theta - \theta^* \right\|^2\leq \left\|X\right\|^4\cdot \left\|\theta - \theta^* \right\|^2.
\end{align*}
Therefore, Assumption~\ref{assum:tv} is satisfied for both cases since $\E\|X\|^4 < \infty$.
\end{proof}

\subsection{Proof of Theorem \ref{prop:ls-reg}}\label{appendix:proof of prop ls-reg}
Theorem \ref{prop:ls-reg} follows directly from Remark \ref{rmk:smooth-clt} and Corollary \ref{corr:ls-reg}. 
\subsection{Explicit forms of the eigenvalues of  $H^{-1}SH^{-1}$ in \eqref{eq:cov-poly-main}.}{\label{appendix:eigenvalues of HSH}
When $\mu=\lambda_0\nu^*$, the rank of $\mu\mu^\top+\nu^*\nu^{*\top}$ is $1$. 
We can write $H^{-1}SH^{-1}$ as follows,
\begin{align} \label{eq:cov-poly}
H^{-1}SH^{-1} = \sigma^2 \begin{bmatrix}
c_1 I + c_2 \nu^*\nu^{*\top}& 0 \\
0 & c_3 I + c_4 \nu^*\nu^{*\top}
\end{bmatrix},
\end{align}
where
\begin{align*}
c_1 &= \frac{(1-\frac{\varepsilon}{2})^{1+2\gamma}\Phi(a) + (\frac{\varepsilon}{2})^{1+2\gamma}(1-\Phi(a))}{\left((1-\frac{\varepsilon}{2})^{1+\gamma}\Phi(a) + (\frac{\varepsilon}{2})^{1+\gamma}(1-\Phi(a))\right)^2}, \\
c_2 &= \frac{(1-\frac{\varepsilon}{2})^{1+2\gamma}\left(\Phi(a)(1+\lambda_0^2) +  \frac{1}{\sqrt{2\pi}}a e^{-\frac{a^2}{2}}\right) + (\frac{\varepsilon}{2})^{1+2\gamma}\left((1-\Phi(a))(1+\lambda_0^2)- \frac{1}{\sqrt{2\pi}}a e^{-\frac{a^2}{2}}\right)}{\left((1-\frac{\varepsilon}{2})^{1+\gamma}\left(\Phi(a)(1+\lambda_0^2) +  \frac{1}{\sqrt{2\pi}}a e^{-\frac{a^2}{2}}\right) + (\frac{\varepsilon}{2})^{1+\gamma}\left((1-\Phi(a))(1+\lambda_0^2)- \frac{1}{\sqrt{2\pi}}a e^{-\frac{a^2}{2}}\right)\right)^2}-c_1,\\
c_3 &= \frac{(1-\frac{\varepsilon}{2})^{1+2\gamma}(1-\Phi(a)) + (\frac{\varepsilon}{2})^{1+2\gamma}\Phi(a)}{\left((1-\frac{\varepsilon}{2})^{1+\gamma}(1-\Phi(a)) + (\frac{\varepsilon}{2})^{1+\gamma}\Phi(a)\right)^2}, \\
c_4 &= \frac{(1-\frac{\varepsilon}{2})^{1+2\gamma}\left((1-\Phi(a))(1+\lambda_0^2) +  \frac{1}{\sqrt{2\pi}}a e^{-\frac{a^2}{2}}\right) + (\frac{\varepsilon}{2})^{1+2\gamma}\left(\Phi(a)(1+\lambda_0^2)- \frac{1}{\sqrt{2\pi}}a e^{-\frac{a^2}{2}}\right)}{\left((1-\frac{\varepsilon}{2})^{1+\gamma}\left((1-\Phi(a))(1+\lambda_0^2) +  \frac{1}{\sqrt{2\pi}}a e^{-\frac{a^2}{2}}\right) + (\frac{\varepsilon}{2})^{1+\gamma}\left(\Phi(a)(1+\lambda_0^2)- \frac{1}{\sqrt{2\pi}}a e^{-\frac{a^2}{2}}\right)\right)^2}-c_3.
\end{align*}

The eigenvalues of the asymptotic covariance matrix are $c_1, c_1 + c_2, c_3, c_3 + c_4$ in the above equations. 

When $\mu$ and $\nu^*$ are linearly independent, the rank of $\mu\mu^\top+\nu^*\nu^{*\top}$ is $2$. 
We can write $H^{-1}SH^{-1}$ as follows,
\begin{align} \label{eq:cov-poly}
H^{-1}SH^{-1} = \sigma^2 \begin{bmatrix}
d_1 I + d_2 \mu\mu^\top + d_3 \nu^*\nu^{*\top}& 0 \\
0 & d_4 I + d_5 \mu\mu^\top + d_6 \nu^*\nu^{*\top}
\end{bmatrix},
\end{align}
where
\begin{align*}
d_1 &= \frac{(1-\frac{\varepsilon}{2})^{1+2\gamma}\Phi(a) + (\frac{\varepsilon}{2})^{1+2\gamma}(1-\Phi(a))}{\left((1-\frac{\varepsilon}{2})^{1+\gamma}\Phi(a) + (\frac{\varepsilon}{2})^{1+\gamma}(1-\Phi(a))\right)^2}, \\
d_4 &= \frac{(1-\frac{\varepsilon}{2})^{1+2\gamma}(1-\Phi(a)) + (\frac{\varepsilon}{2})^{1+2\gamma}\Phi(a)}{((1-\frac{\varepsilon}{2})^{1+\gamma}(1-\Phi(a)) + (\frac{\varepsilon}{2})^{1+\gamma}\Phi(a))^2}, \\
d_2 &= \frac{(1-\frac{\varepsilon}{2})^{1+2\gamma}\Phi(a)(1+\|\mu\|^2) + (\frac{\varepsilon}{2})^{1+2\gamma}(1-\Phi(a))(1+\|\mu\|^2)}{\left(((1-\frac{\varepsilon}{2})^{1+\gamma}\Phi(a)(1+\|\mu\|^2)  + (\frac{\varepsilon}{2})^{1+\gamma}(1-\Phi(a))(1+\|\mu\|^2) \right)^2}-d_1,\\
d_5 &= \frac{(1-\frac{\varepsilon}{2})^{1+2\gamma}(1-\Phi(a))(1+\|\mu\|^2)  + (\frac{\varepsilon}{2})^{1+2\gamma}\Phi(a)(1+\|\mu\|^2)}{\left((1-\frac{\varepsilon}{2})^{1+\gamma}(1-\Phi(a))(1+\|\mu\|^2) + (\frac{\varepsilon}{2})^{1+\gamma}\Phi(a)(1+\|\mu\|^2)\right)^2}-d_4\\
d_3 &= \frac{(1-\frac{\varepsilon}{2})^{1+2\gamma}(\Phi(a) +  \frac{1}{\sqrt{2\pi}}a e^{-\frac{a^2}{2}}) + (\frac{\varepsilon}{2})^{1+2\gamma}(1-\Phi(a)- \frac{1}{\sqrt{2\pi}}a e^{-\frac{a^2}{2}})}{\left((1-\frac{\varepsilon}{2})^{1+\gamma}(\Phi(a) + \frac{1}{\sqrt{2\pi}}a e^{-\frac{a^2}{2}}) + (\frac{\varepsilon}{2})^{1+\gamma}(1-\Phi(a)- \frac{1}{\sqrt{2\pi}}a e^{-\frac{a^2}{2}})\right)^2}-d_1,\\
d_6 &= \frac{(1-\frac{\varepsilon}{2})^{1+2\gamma}(1-\Phi(a)- \frac{1}{\sqrt{2\pi}}a e^{-\frac{a^2}{2}}) + (\frac{\varepsilon}{2})^{1+2\gamma}(\Phi(a) +  \frac{1}{\sqrt{2\pi}}a e^{-\frac{a^2}{2}})}{\left((1-\frac{\varepsilon}{2})^{1+\gamma}(1-\Phi(a)- \frac{1}{\sqrt{2\pi}}a e^{-\frac{a^2}{2}}) + (\frac{\varepsilon}{2})^{1+\gamma}(\Phi(a) +  \frac{1}{\sqrt{2\pi}}a e^{-\frac{a^2}{2}})\right)^2}-d_4.
\end{align*}

The eigenvalues of the asymptotic covariance matrix are $d_1, d_1 + d_2, d_1+d_3, d_4, d_4 + d_5, d_4 + d_6$ in the above equations.}
\subsection{Proof of Corollary~\ref{prop:vanilla-cov-mat}}\label{appendix:proof of prop:vanilla-cov-mat}
\begin{proof}
Our central limit theorem gives the covariance of the form
\begin{align*}
\Sigma = H^{-1} S H^{-1},
\end{align*}
where
\begin{align*}
H = \nabla^2 \L_{\theta^*}(\theta^*) = \E [w(\theta^*;X,A^*)\nabla^2 \ell(\theta^*;\zeta^*)],
\end{align*}
and
\begin{align*}
S =  \E[\xi_{\theta^*}(\theta^*; \zeta^*) \xi_{\theta^*}(\theta^*; \zeta^*)^\top] = \E [(w(\theta^*;X,A^*)^2 \nabla \ell(\theta^*;\zeta^*)\nabla \ell(\theta^*;\zeta^*)^\top].
\end{align*}
One special property of linear regression is that
\begin{align*}
\E\nabla_{\theta_i}^2 \ell(\theta^*;\zeta_t^*) = E_i \otimes \E(X_t X_t^\top),
\end{align*}
where $E_k$ is the matrix with 1 on the $(k, k)$ entry and other entries to be 0, $i$ means we choose the $i$-th bandit, and
\begin{align*}
\E(\nabla_{\theta_i} \ell(\theta^*;\zeta_t^*))(\nabla_{\theta_i} \ell(\theta^*;\zeta_t^*))^\top =\sigma^2E_i \otimes \E(X_t X_t^\top) .
\end{align*}
So they differ only by a constant factor $\sigma^2$ and this also holds for stochastic weight setting. If $w$ is constant instead of stochastic, we denote the corresponding $H, S, \Sigma$ as $H_c, S_c, \Sigma_c$. We claim that $\Sigma \succeq \Sigma_c$, which means equal weight is optimal for linear regression. This is equivalent to show
\begin{small}
    \begin{align*}
(\E [w(\theta^*;X,A^*)\nabla^2 \ell(\theta^*;\zeta^*)])^{-1}\E [(w(\theta^*;X,A^*)^2 \nabla^2 \ell(\theta^*;\zeta^*)](\E [w(\theta^*;X,A^*)\nabla^2 \ell(\theta^*;\zeta^*)])^{-1}
\succeq (\E [\nabla^2 \ell(\theta^*;\zeta^*)])^{-1},
\end{align*}
\end{small}
which is the same as
\begin{small}
    \begin{align*}
\E [(w(\theta^*;X,A^*)^2 \nabla^2 \ell(\theta^*;\zeta^*)]
-(\E [w(\theta^*;X,A^*)\nabla^2 \ell(\theta^*;\zeta^*)])(\E [\nabla^2 \ell(\theta^*;\zeta^*)])^{-1}(\E [w(\theta^*;X,A^*)\nabla^2 \ell(\theta^*;\zeta^*)]) \succeq 0.
\end{align*}
\end{small}
By Schur complement, this is equivalent to
\begin{align*}
\begin{bmatrix}
\E [(w(\theta^*;X,A^*)^2 \nabla^2 \ell(\theta^*;\zeta^*)] & \E [w(\theta^*;X,A^*)\nabla^2 \ell(\theta^*;\zeta^*)] \\
\E [w(\theta^*;X,A^*)\nabla^2 \ell(\theta^*;\zeta^*)] & \E [\nabla^2 \ell(\theta^*;\zeta^*)]
\end{bmatrix} \succeq 0,
\end{align*}
since $\E [\nabla^2 \ell(\theta^*;\zeta)]$ is invertible.
Note that
\begin{align*}
\begin{bmatrix}
w(\theta^*;X,A^*)^2\nabla^2 \ell(\theta^*;\zeta^*)&  w(\theta^*;X,A^*)\nabla^2 \ell(\theta^*;\zeta^*)\\
w(\theta^*;X,A^*)\nabla^2 \ell(\theta^*;\zeta^*) & \nabla^2 \ell(\theta^*;\zeta^*)
\end{bmatrix} \succeq 0.
\end{align*}
Therefore, our conclusion holds.
\end{proof}

\subsection{Proof of Corollary~\ref{corr:exp3 ls-reg}}\label{sec:pf exp3 ls-reg}
{\begin{proof}
    For simplicity, denote $\P_a=\Pr(A=a\mid X,\theta^\prime)$, $a=0,1$.
By definition,
\begin{align*}
    \L_{\theta^\prime}(\theta)
    &=\frac{1}{2}\P_0\varphi(\P_0)\E\left(X^\top\theta^\prime_0-X^\top\theta^*_0\right)^2
    +\frac{1}{2}\P_1\varphi(\P_1)\E\left(X^\top\theta^\prime_1-X^\top\theta^*_1\right)^2\\
    &\quad +\sigma^2\left(\frac{1}{2}\P_0\varphi(\P_0)+\frac{1}{2}\P_1\varphi(\P_1)\right).
\end{align*}
Accordingly, 
\begin{align*}
    \nabla_{\theta_a}\L_{\theta^\prime}(\theta)&=\P_a\varphi(\P_a)\E XX^\top(\theta_a-\theta^*_a), \ a=0,1.\\
    \nabla^2_{\theta_a}\L_{\theta^\prime}(\theta)&=\P_a\varphi(\P_a)\E XX^\top, \ a=0,1.
\end{align*}
    It is easy to check Assumptions \ref{assum:loss}-\ref{assum:gram} hold since $\varphi(\P_a)$ and $\varphi^\prime(\P_a)$ are bounded and $\P_a$ is smooth with respect to $\theta^\prime_a$, $a=0,1$. Particularly, \eqref{eq:hessian lipschitz} in Assumption \ref{assum:hessian} can be satisfied for all $(\theta,\theta^\prime)$ when $\theta$ and $\theta^\prime$ are both within a small neighborhood of $\theta^*$. 

    For Assumption \ref{assum:tv}, take a convergent sequence $\theta_t\rightarrow\theta^*$. First we want to verify
    \begin{align*}
    \lim_{t\rightarrow\infty}\E[\Delta(X,\theta_{t-1})\phi(X)]=0,
    \end{align*}
    where $\E\phi(X)\leq\E\|X\|^4<\infty$.
    
    Since $\left\|\nabla_{\theta_a}\frac{e^{\lambda X^\top\theta_{a}}}{\sum_{a^\prime\in\A}e^{\lambda X^\top\theta_{a^\prime}}}\right\|\leq C\|X\|$. $a=0,1$, we have
    \begin{align*}
        \left|\frac{e^{\lambda X^\top\theta_0}}{\sum_{a^\prime\in\A}e^{\lambda X^\top\theta_{a^\prime}}}-\frac{e^{\lambda X^\top\theta^*_0}}{\sum_{a^\prime\in\A}e^{\lambda X^\top\theta^*_{a^\prime}}}\right| +\left|\frac{e^{\lambda X^\top\theta_1}}{\sum_{a^\prime\in\A}e^{\lambda X^\top\theta_{a^\prime}}}-\frac{e^{\lambda X^\top\theta^*_1}}{\sum_{a^\prime\in\A}e^{\lambda X^\top\theta^*_{a^\prime}}}\right|
        &\leq C\|X\|\cdot|X^\top(\theta-\theta^*)|\\
        &\leq C\|X\|^2\|\theta-\theta^*\|.
    \end{align*}
Therefore, $2\E[\Delta(X,\theta_{t-1})\phi(X)]$ equals to
\begin{align}\label{eq:exp3-TV}
&\quad \E\left[\left(\left|\frac{e^{\lambda X^\top\theta_{0,t-1}}}{\sum_{a^\prime\in\A}e^{\lambda X^\top\theta_{a^\prime,t-1}}}-\frac{e^{\lambda X^\top\theta^*_0}}{\sum_{a^\prime\in\A}e^{\lambda X^\top\theta^*_{a^\prime}}}\right| +\left|\frac{e^{\lambda X^\top\theta_{1,t-1}}}{\sum_{a^\prime\in\A}e^{\lambda X^\top\theta_{a^\prime,t-1}}}-\frac{e^{\lambda X^\top\theta^*_1}}{\sum_{a^\prime\in\A}e^{\lambda X^\top\theta^*_{a^\prime}}}\right|\right)  \phi(X) \right]\notag\\
&\leq \sqrt{\E\left(\left|\frac{e^{\lambda X^\top\theta_{0,t-1}}}{\sum_{a^\prime\in\A}e^{\lambda X^\top\theta_{a^\prime,t-1}}}-\frac{e^{\lambda X^\top\theta^*_0}}{\sum_{a^\prime\in\A}e^{\lambda X^\top\theta^*_{a^\prime}}}\right| +\left|\frac{e^{\lambda X^\top\theta_{1,t-1}}}{\sum_{a^\prime\in\A}e^{\lambda X^\top\theta_{a^\prime,t-1}}}-\frac{e^{\lambda X^\top\theta^*_1}}{\sum_{a^\prime\in\A}e^{\lambda X^\top\theta^*_{a^\prime}}}\right|\right)^2\cdot 
\E\phi^2(X)}\notag\\
&\leq C\sqrt{\E \|X\|^4\E\|X\|^8}\|\theta_{t-1}-\theta^*\|.
\end{align}
Because $\varphi(\epsilon)$ is Lipschitz on the interval $[\delta_0,1)$, 
\begin{align*}
    |w(\theta_{t-1}; X, A) - w(\theta^*; X, A)|^2\leq C\Delta^2(X,\theta_{t-1}).
\end{align*}
Hence, repeating the procedure in \eqref{eq:exp3-TV},
\begin{align}\label{eq:weight-diff-exp3}
    \E\left[|w(\theta_{t-1}; X, A) - w(\theta^*; X, A)|^2 \phi(X) \mid A \right]
    \leq C\E\left[\Delta^2(X,\theta_{t-1}) \phi(X)\right]
    \leq C\E\|X\|^8\|\theta_{t-1}-\theta^*\|^2.
\end{align}
Additionally,
\begin{align*}
\E\left[\|\nabla \ell(\theta; \zeta) - \nabla \ell(\theta^*; \zeta)\|^2 \mid X,A \right] \leq \left\| X X^\top\right\|\cdot \left\|\theta - \theta^* \right\|^2\leq \left\|X\right\|^4\cdot \left\|\theta - \theta^* \right\|^2.
\end{align*}
Using the dominated convergence theorem to all the above results, we can verify Assumption \ref{assum:tv}. {Additionally, here we have $\beta_1=\alpha/2,\beta_2=\alpha$ as specified in the Bahadur representation (Theorem \ref{thm:ls-fs}).}
\end{proof}}

\section{{Theoretical results and proofs for quantile regression}}
\label{sec:pf quant-reg}
In addition to linear regression, our framework can also allow much broader settings than the class of smooth individual loss functions. Under our assumptions, the individual loss function $\ell(\theta; \zeta)$ can be non-smooth. We justify this argument in  quantile regression below.

{\begin{corollary} \label{corr:quant-reg}
Under the stated conditions in Corollary~\ref{corr:ls-reg} with the modified $\varepsilon$-greedy policy in \eqref{eq:modified eps-greedy} applied to quantile regression example we used in Example~\ref{eg:quant-reg}, we further assume
the p.d.f. of $\cE$, denoted as $q(x)$, is smooth, rapidly decreasing, and bounded. We also assume $q(0)> 0$ and $q^\prime(x)$ is bounded.
Under above conditions, Assumptions~\ref{assum:bound}--\ref{assum:tv} are satisfied and Theorem~\ref{thm:smooth-clt} holds.
\end{corollary}}
{Besides the modified $\varepsilon$-greedy, the asymptotic normality also holds for the exponential policy in \eqref{eq:exp3}.}
{\begin{corollary}\label{corr:exp3 quan-reg}
Under the stated conditions in Corollary~\ref{corr:quant-reg} with the exponential  policy in \eqref{eq:exp3} applied to quantile regression example we used in Example~\ref{eg:quant-reg}, we further assume the weight $
  w_t(\theta_{t-1};X_t, A_t) = \varphi(\Pr(A_t \mid X_t, \theta_{t-1}))$ where the function $\varphi(\cdot): (0,1) \mapsto \R^+$ is  differentiable, and $\varphi^\prime(\epsilon)$ is bounded for $\epsilon\in[\delta_0,1)$ where $\delta_0$ is the clipping parameter applied to $\text{clip}_{\delta_0}$.
Under above conditions, the Assumptions~\ref{assum:bound}--\ref{assum:tv} are satisfied and Theorem~\ref{thm:smooth-clt} holds. 
\end{corollary}}
Corollaries~\ref{corr:quant-reg} and \ref{corr:exp3 quan-reg} both state that we can also obtain the limiting distribution for some non-smooth loss functions like a quantile loss. 
Note that although the quantile loss $\rho_\tau(u)= u (\tau-\ID(u<0))$ is not smooth, under the further assumption in Corollaries~\ref{corr:quant-reg} and \ref{corr:exp3 quan-reg} which guarantees certain smoothness of the noise $\cE$, the first and second derivatives of $\E\rho_\tau(u + \cE)$ exists almost surely, which indicates the first and second derivatives of population objective $\L_{\theta^\prime}(\theta)$ exist just as Assumptions~\ref{assum:loss} and \ref{assum:hessian} 
require. Moreover, the population loss $\L_\theta(\theta)$ is still strongly convex under these assumptions when the parameter $\theta$ is sufficiently close to $\theta^*$ while the individual loss is just convex. We further note that the notation of gradient $\nabla \ell$ in Assumptions~\ref{assum:gram} and \ref{assum:tv} can be extended to certain subgradient in this setting, which suggests the possibility of expanding the asymptotic normality analysis to non-smooth models.  

\subsection{Proof of Corollary~\ref{corr:quant-reg}}\label{sec:proof of qtl with eps}
\begin{proof}
We first define $\psi_\tau(u)$ as follows,
\begin{align*}
\psi_\tau(u) = \E\rho_\tau(u + \cE)= \int_{-\infty}^{-u} (u+x)(\tau -1) q(x)\d x + \int_{-u}^{\infty} (u+x)\tau q(x)\d x,
\end{align*}
where $q(\cdot)$ is the p.d.f. for $\cE$ which satisfies $\P(\cE\leq 0)=\tau$.
The first and second order derivative of $\psi_\tau(u)$ can be computed as
\begin{align*}
\psi^\prime_\tau(u) &= \int_{-\infty}^{-u} (\tau -1) q(x)\d x + \int_{-u}^{\infty} \tau q(x)\d x, \\
\psi^{\prime\prime}_\tau(u) &= -(\tau -1)q(-u) + \tau q(-u) = q(-u).
\end{align*}
Because $\psi^\prime_\tau(0) = 0, \psi^{\prime\prime}_\tau(0) = q(0) > 0$, there exists $\delta > 0$ such that for all $|u| < \delta$,
\begin{align*}
\psi^\prime_\tau(u)u \geq \frac{1}{2}q(0)u^2.
\end{align*}
Also, observe that $\psi^{\prime\prime}_\tau(u)>0$ for every $u$, we know $u\psi^\prime_\tau(u)\geq0$ for every $u$.

Now let's compute $\nabla \L_{\theta^\prime}(\theta)$, under $\varepsilon$-greedy policy defined in \eqref{eq:eps-greedy}.
\begin{align*}
\nabla_{\theta_0} \L_{\theta_{t-1}}(\theta)
&= (1-\frac{\varepsilon}{2}) \varphi(1-\frac{\varepsilon}{2}) \E \left[ \ID \{\| \theta_{0,t-1} -  \theta_{1,t-1}\|> t^{-\frac{\alpha}{4}},X^\top \theta_{0,t-1}>  X^\top \theta_{1,t-1}\} X \psi^\prime_\tau\left( X^\top \left(\theta^*_{0} - \theta_{0}\right)\right)\right] \\
&\quad + \frac{\varepsilon}{2} \varphi(\frac{\varepsilon}{2}) \E\left[ \ID \{\| \theta_{0,t-1} -  \theta_{1,t-1}\|> t^{-\frac{\alpha}{4}},X^\top \theta_{0,t-1}< X^\top \theta_{1,t-1}\} X \psi^\prime_\tau\left( X^\top \left(\theta^*_{0} - \theta_{0}\right)\right) \right] \\
&\quad + \frac{1}{2} \varphi(\frac{1}{2}) \E\left[ \ID \{\| \theta_{0,t-1} -  \theta_{1,t-1}\|\leq t^{-\frac{\alpha}{4}}\} X \psi^\prime_\tau\left( X^\top \left(\theta^*_{0} - \theta_{0}\right)\right) \right]\\
\nabla_{\theta_1} \L_{\theta_{t-1}}(\theta)
&= \frac{\varepsilon}{2} \varphi(\frac{\varepsilon}{2}) \E \left[\ID \{\| \theta_{0,t-1} -  \theta_{1,t-1}\|> t^{-\frac{\alpha}{4}},X^\top \theta_{0,t-1}>  X^\top \theta_{1,t-1}\} X \psi^\prime_\tau\left( X^\top \left(\theta^*_{1} - \theta_{1}\right)\right) \right] \\
&\quad + (1-\frac{\varepsilon}{2}) \varphi(1-\frac{\varepsilon}{2}) \E\left[ \ID \{\| \theta_{0,t-1} -  \theta_{1,t-1}\|> t^{-\frac{\alpha}{4}},X^\top \theta_{0,t-1}< X^\top \theta_{1,t-1}\} X \psi^\prime_\tau\left( X^\top \left(\theta^*_{1} - \theta_{1}\right)\right) \right]\\
&\quad + \frac{1}{2} \varphi(\frac{1}{2}) \E\left[ \ID \{\| \theta_{0,t-1} -  \theta_{1,t-1}\|\leq t^{-\frac{\alpha}{4}}\} X \psi^\prime_\tau\left( X^\top \left(\theta^*_{1} - \theta_{1}\right)\right) \right].
\end{align*}
Because $\E[X X^\top]$ is positive definite, there exists a constant $C > 0$ such that $\E[\ID \{ \|X\| < C\}X X^\top]$ is positive definite. For any $\|\theta_{0} - \theta^*_{0}\| \leq \delta/C$, we have the following,
\begin{small}
    \begin{align*}
&\quad \langle \nabla_{\theta_0} \L_{\theta}(\theta), \theta_{0} - \theta^*_{0} \rangle \\
&\geq \frac{1}{2}q(0)\min\left\{(1-\frac{\varepsilon}{2}) \varphi(1-\frac{\varepsilon}{2}), \frac{\varepsilon}{2} \varphi(\frac{\varepsilon}{2}), \frac{1}{2} \varphi(\frac{1}{2}) \right\} \E \left[ \ID \{|X^\top \theta_{0}- X^\top \theta^*_{0}| < \delta\}  \left|X^\top \left(\theta_{0} - \theta^*_{0}\right)\right|^2\right]\\
&\quad+\min\left\{(1-\frac{\varepsilon}{2}) \varphi(1-\frac{\varepsilon}{2}), \frac{\varepsilon}{2} \varphi(\frac{\varepsilon}{2}), \frac{1}{2} \varphi(\frac{1}{2}) \right\}\E \left[ \ID \{|X^\top \theta_{0}- X^\top \theta^*_{0}| \geq \delta\}  X^\top \left(\theta_{0} - \theta^*_{0}\right)\psi^\prime_\tau\left( X^\top \left(\theta_{0} - \theta^{*}_{0}\right)\right)\right]\\
&\geq \frac{1}{2}q(0)\min\left\{(1-\frac{\varepsilon}{2}) \varphi(1-\frac{\varepsilon}{2}), \frac{\varepsilon}{2} \varphi(\frac{\varepsilon}{2}),\frac{1}{2} \varphi(\frac{1}{2}) \right\}   \left(\theta_{0} - \theta^*_{0}\right)^\top\E\left(XX^\top\ID \{\left\|X\right\| < C\}\right)\left(\theta_{0} - \theta^*_{0}\right)\\
&\geq  C^\prime \left\|\theta_{0} - \theta^*_{0}\right\|^2,
\end{align*}
\end{small} 
for some constant $C^\prime > 0$. So the second part of Assumption~\ref{assum:loss} is satisfied.

We now consider the Hessian matrix, we have
\begin{align*}
\nabla^2 \L_{\theta_{t-1}}(\theta) = \begin{bmatrix} H_{0,t-1} & 0 \\
0 & H_{1,t-1}
\end{bmatrix},
\end{align*}
where
\begin{align*}
H_{0,t-1} &= (1-\frac{\varepsilon}{2}) \varphi(1-\frac{\varepsilon}{2}) \E \left\{ \ID \{\| \theta_{0,t-1} -  \theta_{1,t-1}\|> t^{-\frac{\alpha}{4}},X^\top \theta_{0,t-1}>  X^\top \theta_{1,t-1}\} X X^\top q\left[ X^\top \left(\theta_{0} - \theta^{*}_{0}\right)\right]\right\} \\
&\quad + \frac{\varepsilon}{2} \varphi(\frac{\varepsilon}{2}) \E\left\{ \ID \{\| \theta_{0,t-1} -  \theta_{1,t-1}\|> t^{-\frac{\alpha}{4}},X^\top \theta_{0,t-1}< X^\top \theta_{1,t-1}\} X X^\top q\left[ X^\top \left(\theta_{0} - \theta^{*}_{0}\right)\right]\right\} \\
&\quad + \frac{1}{2} \varphi(\frac{1}{2}) \E\left\{ \ID \{\| \theta_{0,t-1} -  \theta_{1,t-1}\|\leq t^{-\frac{\alpha}{4}}\} X X^\top q\left[ X^\top \left(\theta_{0} - \theta^{*}_{0}\right)\right]\right\} \\
H_{1,t-1} &= \frac{\varepsilon}{2} \varphi(\frac{\varepsilon}{2}) \E \left\{ \ID \{\| \theta_{0,t-1} -  \theta_{1,t-1}\|> t^{-\frac{\alpha}{4}},X^\top \theta_{0,t-1}>  X^\top \theta_{1,t-1}\} X X^\top q\left[ X^\top \left(\theta_{1} - \theta^{*}_{1}\right)\right]\right\} \\
&\quad + (1-\frac{\varepsilon}{2}) \varphi(1-\frac{\varepsilon}{2}) \E\left\{ \ID \{\| \theta_{0,t-1} -  \theta_{1,t-1}\|> t^{-\frac{\alpha}{4}},X^\top \theta_{0,t-1}< X^\top \theta_{1,t-1}\} X X^\top q\left[ X^\top \left(\theta_{1} - \theta^{*}_{1}\right)\right]\right\}\\
&\quad + \frac{1}{2} \varphi(\frac{1}{2}) \E\left\{ \ID \{\| \theta_{0,t-1} -  \theta_{1,t-1}\|\leq t^{-\frac{\alpha}{4}}\} X X^\top q\left[ X^\top \left(\theta_{1} - \theta^{*}_{1}\right)\right]\right\}.
\end{align*}
{Obviously, the Hessian matrix exists for all $(\theta, \theta^\prime) \in \R^d \times \R^d$, and the Hessian matrix at $(\theta^*; \theta^*)$ is positive definite since $\lambda_{\min} \E [X X^\top] > 0$. For the non-degenerate model, we can check the Lipschitz continuity of $\nabla^2 \L_{\theta_{t-1}}(\theta)$ at $(\theta^*, \theta^*)$. Its Lipschitz continuity with respect to $\theta_{t-1}$ can be checked by the same argument as in the linear case. It is clearly differentiable with respect to $\theta$, so it is also Lipschitz continuous with respect to $\theta$, which means \eqref{eq:hessian lipschitz} holds. For the degenerate model, when $t$ is large enough, 
\begin{align*}
\nabla^2 \L_{\theta_{t-1}}(\theta) = \begin{bmatrix} \frac{1}{2} \varphi(\frac{1}{2}) \E\left\{ X X^\top q\left[ X^\top \left(\theta_{0} - \theta^{*}_{0}\right)\right]\right\} & 0 \\
0 & \frac{1}{2} \varphi(\frac{1}{2}) \E\left\{ X X^\top q\left[ X^\top \left(\theta_{1} - \theta^{*}_{1}\right)\right]\right\}
\end{bmatrix}.
\end{align*}
Its Lipschitz continuity only depends on $\theta$, and since it is clearly differentiable with respect to $\theta$, $\nabla^2 \L_{\theta_{t-1}}(\theta)$ also has the same Lipschitz continuity at $(\theta^*,\theta^*)$ as the non-degenerate model, that is, \eqref{eq:hessian lipschitz} holds. Therefore, we verified Assumption \ref{assum:hessian} for both cases.}

For any $A, X$, we can bound
\begin{align*}
\E(\|\nabla \ell(\theta; \zeta)\|^2 \mid X, A) &\leq \|X\|^2 .
\end{align*}
So the first part of Assumption~\ref{assum:gram} is satisfied with $\phi(X) = \|X\|^2$.

The matrix $S$ can be computed by
\begin{align*}
S = \E[w^2(\theta^*;X,A^*)\nabla \ell(\theta^*; \zeta^*) \nabla \ell(\theta^*; \zeta^*)^\top]
:= \begin{bmatrix} S_0 & 0 \\
0 & S_1
\end{bmatrix},
\end{align*}
where
{\begin{itemize}
    \item Non-degenerate model:
    \begin{align*}
S_0 &= (1-\frac{\varepsilon}{2}) \varphi^2(1-\frac{\varepsilon}{2})
\tau(1-\tau) \E \left[ \ID \{X^\top \theta^*_{0}>  X^\top \theta^*_{1}\}  X X^\top \right]\\
& \quad + \frac{\varepsilon}{2} \varphi^2(\frac{\varepsilon}{2}) \tau(1-\tau) \E \left[ \ID \{X^\top \theta^*_{0}< X^\top \theta^*_{1}\}  X X^\top \right],\\
S_1 &= (1-\frac{\varepsilon}{2}) \varphi^2(1-\frac{\varepsilon}{2})\tau(1-\tau) \E \left[ \ID \{X^\top \theta^*_{0}< X^\top \theta^*_{1}\}  X X^\top \right]\\
& \quad + \frac{\varepsilon}{2} \varphi^2(\frac{\varepsilon}{2}) \tau(1-\tau) \E \left[ \ID \{X^\top \theta^*_{0}>  X^\top \theta^*_{1}\}  X X^\top \right].
\end{align*}
\item Degenerate model:
\begin{align*}
    S_0=S_1=\frac{1}{2} \varphi^2(\frac{1}{2}) \tau(1-\tau) \E  X X^\top.
\end{align*}
\end{itemize}}
Therefore, Assumption \ref{assum:gram} can be verified for both cases.

Using the same argument in the linear case, we can derive that
\begin{align*}
&\lim_{\theta \rightarrow \theta^*} \E[\Delta(X, \theta)\phi(X)] = 0, \\
&\lim_{\theta \rightarrow \theta^*} \E\left[|w(\theta; X, A) - w(\theta^*; X, A)|^2 \phi(X) \mid A \right] = 0.
\end{align*}
Finally, we have the following inequality,
\begin{align*}
\E \left(\ID \{u + \cE<0\} - \ID \{\cE<0\}\right)^2
\leq \Pr (|\cE| < |u|).
\end{align*}
\begin{align*}
    \E\left[\|\nabla \ell(\theta; \zeta) - \nabla \ell(\theta^*; \zeta)\|^2 \mid X,A \right]
    &\leq \|X\|^2 \left(\E\left(\ID\{\cE<X^\top(\theta-\theta^*)-\ID\{\cE<0\}\right)^2 \mid X,A\right)\\
    &\leq \|X\|^2 \Pr (|\cE| < |X^\top(\theta_A - \theta_A^*)|).
\end{align*}
Again, we can use dominated convergence theorem to prove this term converges to 0 as $\theta \to \theta^*$. Therefore, Assumption~\ref{assum:tv} is satisfied.

\end{proof}

\subsection{Proof of Corollary~\ref{corr:exp3 quan-reg}}
{\begin{proof}
    Repeat the same procedure in Sections \ref{sec:proof of qtl with eps} and \ref{sec:pf exp3 ls-reg}, we can finish the proof. Note that in the quantile regression setting, $\E\phi^2(X)\leq \E\|X\|^4<\infty$.
\end{proof}}

\section{Proofs of the Bahadur representation}
\label{sec:app-remainder}
We first restate Theorem \ref{thm:ls-fs} with a detailed decomposition. 
\begin{lemmax} 
{For any policy and weighting scheme satisfying the conditions in Theorem \ref{thm:smooth-clt}, 
we further assume $\E \left(\|\nabla \ell(\theta_{t-1}; \zeta)-\nabla \ell(\theta^*; \zeta)\|^2 \mid X, A \right) \leq C\|\theta_{t-1} - \theta^*\|^2$ in Assumption \ref{assum:tv}, and 
\begin{enumerate}[(a)]
  \item \label{a}Given $\theta^*$, the following inequality holds for some constant $\beta_1,\beta_2>0$,
\begin{align*}
&\E\left[ \Delta(X,\theta_{t-1})\phi(X) \mid \right] \leq C t^{-\beta_1},\\
&\E\left[|w(\theta_{t-1}; X, A) - w(\theta^*; X, A)|^2 \phi(X) \mid A \right] \leq C t^{-\beta_2},
\end{align*}
where $\phi$ is defined in Assumption~\ref{assum:gram};
\item \label{b}For any action $A \in \A$ and covariate $X$, assume $\|\nabla \ell(\theta; \zeta)\|^4$ exists almost surely under $\P_{Y\mid X,A}$, and $\E \left(\|\nabla \ell(\theta; \zeta)\|^4 \mid X, A \right) \leq C(1+\|\theta - \theta^*\|^4 )$, where $C$ is some positive constant,
\end{enumerate}}
We have the following decomposition
\begin{align} \label{eq:bahadur-rep2}
\sqrt{t}\Sigma^{-1/2}(\bar{\theta}_t - \theta^*) &= \underbrace{\frac{1}{\sqrt{t}} \sum_{i=1}^{t-1} \Sigma_t^{-1/2} Q_i^t \xi_{\theta^*}(\theta^*; \zeta_i^*)}_\text{$W$} + \underbrace{\frac{1}{\sqrt{t}}\sum_{i=1}^{t-1} \Sigma^{-1/2} Q_i^t (\xi_{\theta_{i-1}}(\theta_{i-1}; \zeta_i)-\xi_{\theta^*}(\theta^*; \zeta_i^*))}_\text{$R_1$} \notag \\
& \quad + \underbrace{\frac{1}{\sqrt{t}\eta_0}\Sigma^{-1/2} Q_0^t(\theta_0 - \theta^*)}_\text{$R_2$} +\underbrace{\frac{1}{\sqrt{t}}\sum_{i=1}^{t-1} \Sigma^{-1/2} Q_i^t (\nabla\L_{\theta_i}(\theta_i) - H(\theta_i - \theta^*))}_\text{$R_3$} \notag \\
& \quad + \underbrace{\frac{1}{\sqrt{t}}\sum_{i=1}^{t-1} (\Sigma^{-1/2}-\Sigma_t^{-1/2})Q_i^t \xi_{\theta^*}(\theta^*; \zeta_i^*) }_\text{$R_4$} \notag\\
&= W + R_1 + R_2 + R_3 + R_4,
\end{align}
where $\E [W] = 0, \E [W W^\top] = I_d$, $\Sigma_t= \frac{1}{t}\sum_{i=1}^{t-1} Q_i^t S Q_i^t$, $Q_i^t = \eta_i \sum_{j = i}^{t-1} \prod_{k = i+1}^j (I_d - \eta_k H)$ for $t>0$ and $i>0$, Furthermore, we have,
\begin{align*}
{\E \|R_1\|^2 \lesssim t^{-\min\{\alpha,\beta_1,\beta_2\}},}
\; \;\E \|R_2\|^2 \lesssim t^{-1}, \; \; \E \|R_3\| \lesssim t^{-\alpha + \frac{1}{2}}, \; \; \E \|R_4\|^2 \lesssim t^{2\alpha-2}.
\end{align*}
\end{lemmax}

\subsubsection*{Proof of Theorem~\ref{thm:ls-fs}}
\begin{proof}
To address the randomness in the adaptive policy $A_t$, it is necessary to define a coupling for all categorical distributions with $|\A|$ categories simultaneously. Previously in the proof of Theorem~\ref{thm:smooth-clt}, we used the total variation distance to bound $\P(A_t \neq A^*)$. Here we need a generalized coupling defined as follows.

Consider the $(|\A|-1)$-simplex $S = \{(x_1, \dots, x_{|\A|}) \mid x_i \geq 0, \sum x_i = 1\}$. It has $|\A|$ vertices given by $V_i = (0, \dots, 0, 1, 0, \dots, 0)$ where $1$ is in the $i$-th coordinate. Take a point $P$ uniformly from $S$. For any categorical distribution with probability $(p_1, \dots, p_{|\A|})$, define $K = (p_1, \dots, p_{|\A|})$. The probability that $P$ lies in the sub-simplex with vertices $\{V_1, \dots, V_{i-1},V_{i+1}, \dots , V_{|\A|}, K\}$ ($V_i$ is deleted) is exactly $p_i$. Thus, $K$ gives a partition of $S$ that has the required categorical distribution and we can use this to define the action $A$. Furthermore, given two different distribution $K, K^\prime$, it is easy to see that the quantity $\P(A \neq A^\prime)$ is bounded by $C d_{\rm TV}(K, K^\prime)$, where $C$ is some positive constant which only depends on $|\A|$. So all previous bounds still holds up to a constant.

In conclusion, the probability space we have used for stochastic gradient descent can be redefined using \emph{i.i.d.} random variables $(X_t, Y_t, P_t), t\geq 1$, where $P_t$ obeys a uniform distribution on a $(|\A|-1)$-simplex. We also redefine $\zeta_t = (X_t, Y_t, P_t)$.

We would like to note that in this proof and the proofs thereafter, with a slight abuse of notation, we will use $C$ to represent different positive constants.

We can now give a decomposition of $\sqrt{t}\Sigma^{-1/2}(\bar{\theta}_t - \theta^*)$ as follows,
\begin{align*}
\sqrt{t}\Sigma^{-1/2}(\bar{\theta}_t - \theta^*) &= \underbrace{\frac{1}{\sqrt{t}} \sum_{i=1}^{t-1} \Sigma_t^{-1/2} Q_i^t \xi_{\theta^*}(\theta^*; \zeta_i^*)}_\text{$W$} + \underbrace{\frac{1}{\sqrt{t}}\sum_{i=1}^{t-1} \Sigma^{-1/2} Q_i^t (\xi_{\theta_{i-1}}(\theta_{i-1}; \zeta_i)-\xi_{\theta^*}(\theta^*; \zeta_i^*))}_\text{$R_1$}\\
& \quad + \underbrace{\frac{1}{\sqrt{t}\eta_0}\Sigma^{-1/2} Q_0^t(\theta_0 - \theta^*)}_\text{$R_2$} +\underbrace{\frac{1}{\sqrt{t}}\sum_{i=1}^{t-1} \Sigma^{-1/2} Q_i^t (\nabla\L_{\theta_i}(\theta_i) - H(\theta_i - \theta^*))}_\text{$R_3$}\\
& \quad + \underbrace{\frac{1}{\sqrt{t}}\sum_{i=1}^{t-1} (\Sigma^{-1/2}-\Sigma_t^{-1/2})Q_i^t \xi_{\theta^*}(\theta^*; \zeta_i^*) }_\text{$R_4$}\\
&= W + R_1 + R_2 + R_3 + R_4,
\end{align*}
From \eqref{eq:bound_grad}, we can estimate
\begin{align}\label{eq:intermediate}
    \|\L_{\theta_{t-1}}(\theta_{t-1})-H(\theta_{t-1}-\theta^*)\|\leq C\|\theta_{t-1}-\theta^*\|^2.
\end{align}
{From Assumption (\ref{b}), we also have
\begin{align*}
&\quad \E \left[\|\xi_{\theta_{t-1}}(\theta_{t-1}; \zeta_t)  + \nabla \L_{\theta_{t-1}}(\theta_{t-1})\|^4  \mid \mathcal{F}_{t-1} \right]\\
&= \E \left[\|w(\theta_{t-1}; \zeta_t) \nabla \ell(\theta_{t-1}; \zeta_t)\|^4  \mid \mathcal{F}_{t-1} \right] \\
&\leq \bar{w}^4 \E\left[\|\nabla \ell(\theta_{t-1}; \zeta_t)\|^4 \mid \mathcal{F}_{t-1} \right]\\
& \leq C (1+\|\theta_{t-1} - \theta^*\|^4).
\end{align*}
Therefore, 
\begin{align*}
\E \left[\|\xi_{\theta_{t-1}}(\theta_{t-1}; \zeta_t)\|^4 \mid\mathcal{F}_{t-1} \right] \leq C (1+\|\theta_{t-1} - \theta^*\|^4),
\end{align*}
and $\E\|\xi_{\theta^*}(\theta^*; \zeta^*)\|^4 \leq C$,
which implies the following bounds from Lemma 3.2 in \cite{chen2016statistical},}
\begin{align*}
\E\|\theta_t - \theta^*\|^2 &\leq C t^{-\alpha}, \\
\E\|\theta_t - \theta^*\|^4 &\leq C t^{-2\alpha}.
\end{align*}
These two inequalities above are also derived as Lemma 5.12 and 5.14 in \cite{shao2022berry}.

From \cite{polyak1992acceleration}, we know that $\|Q_i^t\| \leq C$. Moreover, $\|\Sigma_t^{-1/2} Q_i^t\| \leq C$ is guaranteed in \cite{shao2022berry}. In the proof of Lemma 1 of \cite{polyak1992acceleration}, it can be seen that,
\begin{align*}
H^{-1} - Q_i^t &= H^{-1} - \eta_i \sum_{j = i}^{t-1} \prod_{k = i+1}^j (I_d - \eta_k H)\\
&= \sum_{j = i}^{t-1} (\eta_j - \eta_i) \prod_{k = i+1}^j (I_d - \eta_k H) + H^{-1} \prod_{k = i+1}^t (I_d - \eta_k H),
\end{align*}
and the first term is $O(i^{\alpha-1})$, {which is because 
\begin{align*}
    \left\|\sum_{j = i}^{t-1} (\eta_j - \eta_i) \prod_{k = i+1}^j (I_d - \eta_k H)\right\|
    &=\left\|\sum_{j = i+1}^{t-1}\sum_{s=i}^{j-1} (\eta_{s+1} - \eta_s) \prod_{k = i+1}^j (I_d - \eta_k H)\right\|\\
    &\leq C_1^\prime\sum_{j = i+1}^{t-1}\sum_{s=i}^{j-1}\alpha s^{-\alpha-1}\left\|\prod_{k = i+1}^j (I_d - \eta_k H)\right\|\\
    &\leq i^{-1}C_1^\prime\sum_{j=i+1}^{t-1}m_i^je^{-C_2^\prime m_{i}^{j}}\\
    &\leq  i^{-1}C_1^\prime\sum_{j=i+1}^{t-1}\frac{m_{i}^{j}e^{-C_2^\prime m_{i}^{j}}(m_{i}^{j+1}-m_{i}^{j})}{\eta_j}\\
    & \leq\frac{C_1^\prime i^{-1}}{\eta_i}\sum_{j=i+1}^{t-1}m_{i}^{j}e^{-C_2^\prime m_{i}^{j}}(m_{i}^{j+1}-m_{i}^{j}),
\end{align*}
where $C_1^\prime,C_2^\prime$ are some constants and $m_i^j=\sum_{l=i}^{j-1}\eta_l$. The last inequality holds because we have $i\eta_i\leq C^\prime j\eta_j$ for large enough $j$, and $m_i^j\geq K^\prime \ln(j/i)$ for arbitrarily large $K^\prime$ when $j$ is large enough. Hence,
\begin{align*}
    \frac{1}{\eta_j}\leq C^\prime\frac{j}{i\eta_i}\leq\frac{C^\prime}{\eta_i}e^{\frac{m_i^j}{K^\prime}}.
\end{align*}
Let $f(x)=xe^{-C_2^\prime x}$ and we know $f(x)$ is increasing at $(0,x^*)$ and decreasing at $(x^*,\infty)$, $x^*=1/C_2^\prime$. Define $S_1^{i,t}=\{j\in[i+1,t-1]: m_{i}^{j}\leq x^*\}$, and $S_2^{i,t}=\{j\in[i+1,t-1]: m_{i}^{j}> x^*\}$. Especially, we can check $f(m_{i}^j)<Kf(m_i^{j+1})$ for some constant $K$ where $j\in S_2^{i,t}$. Then,
\begin{align*}
    \left\|\sum_{j = i}^{t-1} (\eta_j - \eta_i) \prod_{k = i+1}^j (I_d - \eta_k H)\right\|
    &\leq \frac{C_1^\prime i^{-1}}{\eta_i}\left(\sum_{j\in S^{i,t}_1}+\sum_{j\in S^{i,t}_2}\right)m_{i}^{j}e^{-C_2^\prime m_{i}^{j}}(m_{i}^{j+1}-m_{i}^{j})\\
    &\leq \frac{C_1^\prime i^{-1}}{\eta_i}\left(\int_{0}^{x^*}me^{-C_2^\prime m}\d m+K\int_{x^*}^\infty me^{-C_2^\prime m}\d m\right)\\
    &\leq C_1^\prime i^{\alpha-1}.
\end{align*}
}
In Lemma D.2 of \cite{chen2016statistical}, it is proved that
\begin{align*}
\left\|\prod_{k = i+1}^t (I_d - \eta_k H)\right\| \leq e^{-C(t-i)\eta_t}.
\end{align*}
So we have
\begin{align} \label{eq:cov-diff}
\|\Sigma_t - \Sigma\| & \leq  \frac{C}{t}\sum_{i=1}^t (i^{\alpha-1} + e^{-C(t-i)\eta_t}) \notag \\
& \leq   Ct^{\alpha-1} + \frac{e^{-C\eta_t}}{t(1-e^{-C\eta_t})} \notag\\
& \leq  Ct^{\alpha-1} + Ct^{-\alpha-1} \notag\\
& \leq Ct^{\alpha-1}.
\end{align}

Now we proceed to the main part of the proof. Similar to \eqref{eq:bound_M2b}, we have
\begin{align*}
&\quad \E\|\xi_{\theta_{t-1}}(\theta_{t-1}; \zeta_t)  - \xi_{\theta^*}(\theta^*; \zeta_t^*) \|^2 \\
&\leq C\E[(\Delta(X, \theta_{t-1})+\max_{A\in \A}|w(\theta_{t-1};X,A)-w(\theta^*;X,A)|^2)(1+\|\theta_{t-1} - \theta^*\|^2)\phi(X)]+C\|\theta_{t-1} - \theta^*\|^2.
\end{align*}
{Because of Assumption (\ref{a}),
the whole term can be estimated by
\begin{align*}
&\quad \E[(\Delta(X, \theta_{t-1})+\max_{A\in \A}|w(\theta_{t-1};X,A)-w(\theta^*;X,A)|^2)(1+\|\theta_{t-1} - \theta^*\|^2)\phi(X)]+C\E\|\theta_{t-1} - \theta^*\|^2\\
&\leq C t^{-\beta_1}+Ct^{-\beta_2}+Ct^{-\alpha}\\
&\leq C t^{-\min\{\alpha,\beta_1,\beta_2\}}.
\end{align*}
Combining the results above, we obtain that
\begin{align} \label{eq:grad-diff}
\E\|\xi_{\theta_{t-1}}(\theta_{t-1}; \zeta_t)  - \xi_{\theta^*}(\theta^*; \zeta_t^*) \|^2 \leq C t^{-\min\{\alpha,\beta_1,\beta_2\}}.
\end{align}}

With all these intermediate results in hand, we can proceed to the conclusion as follows. First of all, by inequality~\eqref{eq:grad-diff}, we have the following bound for $R_1$,
{\begin{align*}
\E \|R_1\|^2 &\leq Ct^{-1}\E \| \sum_{i = 1}^t\xi_{\theta_{i-1}}(\theta_{i-1}; \zeta_i)-\xi_{\theta^*}(\theta^*; \zeta_i^*)\|^2\\
& \leq Ct^{-1} \sum_{i = 1}^t\E \| \xi_{\theta_{i-1}}(\theta_{i-1}; \zeta_i)-\xi_{\theta^*}(\theta^*; \zeta_i^*)\|^2 \\
& \leq Ct^{-1} \sum_{i = 1}^t i^{-\min\{\alpha,\beta_1,\beta_2\}} \leq C t^{-\min\{\alpha,\beta_1,\beta_2\}}.
\end{align*}}
Also, it is easy to derive that
\begin{align*}
\E \|R_2\|^2 \leq Ct^{-1}.
\end{align*}
Using the above intermediate result \eqref{eq:intermediate}, the $R_3$ term has the convergence rate below,
\begin{align*}
\E \|R_3\| &\leq t^{-1/2}\sum_{i=0}^{t-1} \E\|\nabla \L_{\theta_i}(\theta_i) - H(\theta_i - \theta^*)\|\\
&\leq Ct^{-1/2} \sum_{i=0}^{t-1} \E\|\theta_i-\theta^*\|^2\\
&\leq Ct^{-1/2} \sum_{i=0}^{t-1} i^{-\alpha} \leq Ct^{-\alpha + \frac{1}{2}}.
\end{align*}
Finally, we can bound $R_4$ using our result in \eqref{eq:cov-diff},
\begin{align*}
\E \|R_4\|^2 \leq \frac{C}{t} \left\|\Sigma^{-1/2}-\Sigma_t^{-1/2} \right\|^2\sum_{i=1}^{t-1}\E\|\xi_{\theta^*}(\theta^*; \zeta_i^*)\|^2 \leq C t^{2\alpha -2}.
\end{align*}
\end{proof}

\subsection{Proof of Corollary \ref{corr:bahadur eps}}\label{appendix:lower bound on R1}

\textbf{Non-degenerate model:}
Under the modified $\varepsilon$-greedy policy in \eqref{eq:modified eps-greedy}, 
for a non-degenerated normal variable $X$, equation (\ref{eq:hessian lipschitz}) in Assumption~\ref{assum:hessian} clearly holds. Through equation \eqref{eq:eps-non-deg-TV} and \eqref{eq:weight-diff-eps-non-deg}, Assumption (\ref{a}) can also be transformed into a (stronger) differentiability condition. Denote the common part of the left hand side of Assumption (\ref{a}) as $F(\theta)$,
\begin{align*}
F(\theta) &= \E\left[ \left| \ID \left(X^\top \theta^*_{0}>  X^\top \theta^*_{1}\right) - \ID \left(X^\top \theta_{0}> X^\top \theta_{1}\right) \right| \phi(X) \right]\\
&=\E\left[ \left| \ID \left(\Tilde{X}^\top \Tilde{\theta}^*> 0\right) - \ID \left(\Tilde{X}^\top \Tilde{\theta}> 0\right) \right| \phi(X) \right],
\end{align*}
where $\Tilde{X}^\top=[X^\top,X^\top]$, $\Tilde{\theta}^\top=[\theta^\top_{0},-\theta^\top_{1}]$ for any $\theta$.
It is not differentiable at $\theta^*$. In fact, we have for some constants $\underline{C}, \overline{C}$
\begin{align*}
0 \leq \underline{C} \leq \frac{\partial F(\theta)}{\partial v} \big|_{\theta = \theta^*}\leq \overline{C},
\end{align*}
for any $v$ orthogonal to $\Tilde{\theta}^*$, and the directional derivatives paralleled $\Tilde{\theta}^*$ is 0. The proof is similar to the proof of Lemma \ref{lem:hessian derivative} 
 {under one additional condition that the p.d.f. of $X$, $p(\cdot)$, is greater than some positive constant when $\|X\|$ is near 0.} Furthermore, we can also deduce that 
\begin{align*}
\underline{C} |P(\theta - \theta^*)|/2 \leq F(\theta) \leq 2\overline{C} |P(\theta - \theta^*)|,
\end{align*}
where $P$ is the projection to the orthogonal complement of $\Tilde{\theta}^*$ and $\|\theta - \theta^*\|$ is sufficiently small. 
{This implies Assumption (\ref{a}) with $\beta_1=\beta_2=\alpha/2$.} 
So Theorem~\ref{thm:ls-fs} holds. By same argument, we can also prove that
\begin{align*}
0 \leq \underline{C} \leq \frac{\partial \E\left[\Delta(X, \theta)\right]}{\partial v} \big|_{\theta = \theta^*}\leq \overline{C},
\end{align*}
where $\Delta(X, \theta) = (1-\varepsilon)\left|\ID \left(X^\top \theta^*_{0}>  X^\top \theta^*_{1}\right) - \ID \left(X^\top \theta_{0}> X^\top \theta_{1}\right)\right|$. So it can be bounded by
\begin{align*}
\underline{C} |P(\theta - \theta^*)| \leq \E\left[\Delta(X, \theta)\right] \leq \overline{C} |P(\theta - \theta^*)|.
\end{align*}

For $R_1$, we first decompose the term $\xi_{\theta_{t-1}}(\theta_{t-1}; \zeta_t) - \xi_{\theta^*}(\theta^*; \zeta_t^*)$ as follows
\begin{align*}
&\quad \E[\|\xi_{\theta_{t-1}}(\theta_{t-1}; \zeta_t) - \xi_{\theta^*}(\theta^*; \zeta_t^*)\|^2] \\
&=\E\left[\|w(\theta_{t-1}; X_t, A_t)\nabla \ell(\theta_{t-1}; \zeta_t) - w(\theta^*; X_t, A^*)\nabla \ell(\theta^*; X_t, A^*, Y_t)\|^2 \right] - \|\nabla \L_{\theta_{t-1}}(\theta_{t-1}) - \nabla \L_{\theta^*}(\theta^*)\|^2\\
&= M_2 - M_1,
\end{align*}
where $M_1, M_2, M_3, M_4$ has been defined in the proof of  Theorem~\ref{thm:smooth-clt}. From previous estimates, $M_1 \leq C t^{-\alpha}$. Previous decomposition can also provide lower bounds,
\begin{align*}
M_2 & \geq C \E [\Delta(X_t, \theta_{t-1})M_3],\\
M_3 & \geq C \E \left[\|\nabla \ell(\theta_{t-1}; X_t, A_t, Y_t)\|^2 +\|\nabla \ell(\theta^*; X_t, A^*, Y_t^*)\|^2\mid X_t, A_t \neq A^* \right] \geq C^\prime.
\end{align*}
Combining all inequalities together plus $\E\|\theta_t-\theta^*\|\leq Ct^{-\frac{\alpha}{2}}$, we have $\E\|\xi_{\theta_{t-1}}(\theta_{t-1}; \zeta_t) - \xi_{\theta^*}(\theta^*; \zeta_t^*)\|^2 \geq C \E [\Delta(X_t, \theta_{t-1})]$. The proof of Theorem~\ref{thm:ls-fs} implies that $\Sigma_t^{-1/2}$ and $Q_i^t$ are bounded from below for sufficiently large $i, t > i_0$. So
\begin{align*}
\E \|R_1\|^2 & = \frac{1}{t}\sum_{i=1}^{t-1} \|\Sigma_t^{-1/2} Q_i^t (\xi_{\theta_{t-1}}(\theta_{t-1}; \zeta_t) - \xi_{\theta^*}(\theta^*; \zeta_t^*))\|^2\\
& \geq \frac{C}{t} \sum_{i > i_0}^{t-1} \E \|\xi_{\theta_{t-1}}(\theta_{t-1}; \zeta_t) - \xi_{\theta^*}(\theta^*; \zeta_t^*)\|^2\\
& \geq \frac{C}{t} \sum_{i > i_0}^{t-1} \E [\Delta(X_i, \theta_{i-1})]\\
& \geq \frac{C}{t} \sum_{i > i_0}^{t-1} \E|P(\theta_{i-1} - \theta^*)|.
\end{align*}
Theorem~\ref{thm:ls-fs} implies that 
\begin{align*}
\frac{1}{t}\sum_{i=1}^t \E |P(\theta_i - \theta^*)| \geq \E |P(\bar{\theta}_t - \theta^*)| \geq C t^{-1/2}.
\end{align*}
Therefore, we can come to the conclusion that
\begin{align*}
\E \|R_1\|^2 \geq C t ^{-1/2}.
\end{align*}

{\textbf{Degenerate model:}
Since $X$ is sub-Gaussian, Lemma \ref{lem:a.s.} holds. Through equation \eqref{eq:eps-de-TV} and \eqref{eq:weight-diff-eps-de}, by Lemma \ref{lem:a.s.}, we have the common part of the left hand side of Assumption (\ref{a}) is 
\begin{align*}
    \E\ID\left(\|\theta_{0,t-1}-\theta_{1,t-1}\|>t^{-\frac{\alpha}{4}}\right)=\Pr\left(\|\theta_{0,t-1}-\theta_{1,t-1}\|>t^{-\frac{\alpha}{4}}\right)\leq Ct^{-2\alpha}.
\end{align*}
{Therefore, $\beta_1=\beta_2=2\alpha$ here}.}

\subsection{Proof of Corollary \ref{corr:bahadur exp3}}\label{appendix:bahadur exp3}
\begin{proof}
    The rate of the left hand side of Assumption (\ref{a}) is shown in \eqref{eq:exp3-TV} and \eqref{eq:weight-diff-exp3}, which is bounded by $\|\theta_{t-1}-\theta^*\|$ and its square respectively. {Therefore, $\beta_1=\alpha/2,\beta_2=\alpha$}.
\end{proof}

\section{Asymptotic normality under $\varepsilon_t$-greedy policy with varying $\varepsilon_t$}
\label{sec:app-eps-inf}

We work under the linear regression setting with assumptions in Corollary~\ref{corr:ls-reg}. We relax the $\varepsilon$-greedy policy by the $\varepsilon_t$-greedy policy is used, the new policy $A_t \sim \pi_t$ is defined by
\begin{align*}
\Pr(A_t = 0 \mid X_t, \theta_{t-1}) = (1-\varepsilon_t)\ID \{X_t^\top \theta_{0,t-1}> X_t^\top \theta_{1,t-1}\} + \frac{\varepsilon_t}{2},
\end{align*}
{and the modified $\varepsilon$-greedy policy in \eqref{eq:modified eps-greedy} is updated accordingly, which could be called the modified $\varepsilon_t$-greedy policy.} The weight $w_t$ is again defined as some functions of $\Pr(A_t = 0 \mid X_t, \theta_{t-1})$. Assume $\lim_{t \to \infty}\varepsilon_t = \varepsilon_{\infty}$, for some constant $\varepsilon_\infty \in (0,1)$. Notice that $\varepsilon_t$ is a deterministic sequence, meaning it does not change with respect to $X, A, Y$ and $\theta, \theta^\prime$.

The definition of $\L_{\theta^\prime}(\theta)$ should be change accordingly, i.e.,
\begin{align*}
\L_{t, \theta^\prime}(\theta) &= \E_{\P} \left[ \E_{\pi_t(X, \theta^\prime)} \left(w_t(\theta^\prime; X, A)\ell(\theta; X, A, Y) \mid X \right) \right],\\
\L_{\infty, \theta^\prime}(\theta) &= \E_{\P} \left[ \E_{\pi_{\infty}(X, \theta^\prime)} \left(w_{\infty}(\theta^\prime; X, A)\ell(\theta; X, A, Y) \mid X \right) \right].
\end{align*}
Furthermore, the matrix $H, S$ should be defined with respect to $\L_{\infty}$.

\begin{theorem} \label{thm:clt-eps-inf}
{Under the modified $\varepsilon_t$-greedy policy we discussed above,}
with the same conditions as Corollary~\ref{corr:ls-reg}, the asymptotic normality also holds for averaged SGD estimator $\bar{\theta}_t$, i.e.,
\begin{align*}
\sqrt{t}(\bar{\theta}_t - \theta^*) \rightarrow N(0, H^{-1}SH^{-1}).
\end{align*}
\end{theorem}

\begin{proof}
We will follow the steps in the proof of Theorem~\ref{thm:smooth-clt} in Section~\ref{sec:app-clt} of the supplement. To simplify the notation, we denote $R(\theta)$ as $\nabla \L_{\infty, \theta}(\theta)$, {$\xi_t$ as $w_t(\theta_{t-1}; X, A)\nabla \ell(\theta_{t-1}; \zeta) - \nabla \L_{t, \theta_{t-1}}(\theta_{t-1})$}, $\xi_t(0)$ as $w_{\infty} (\theta^*; X, A^*)\nabla \ell(\theta^*; \zeta^*) - \nabla \L_{\infty, \theta^*}(\theta^*)$, and $\xi_t(\theta_{t-1})$ as $\xi_t - \xi_t(0)$.

Here we check the Assumption 3.3 in \cite{polyak1992acceleration}. Because $\varepsilon_t \to \varepsilon_{\infty}$, $\varepsilon_t$ is uniformly bounded away from $0$ and $1$ for sufficiently large $t$. So we still have the following inequality,
\begin{align*}
\E [\|\xi_t\|^2 \mid \mathcal{F}_{t-1}] + \|R(\theta_{t-1})\|^2 \leq K_2(1+\|\theta_{t-1}-\theta^*\|^2).
\end{align*}
The only thing that remains unproved is
\begin{align*}
\E [\|\xi_t(\theta_{t-1})\|^2 \mid \mathcal{F}_{t-1}] \leq \delta(\theta_{t-1}),
\end{align*}
with $\lim_{\theta\to \theta^*} \delta(\theta) = 0$.

Similarly, $\{\xi_t(0)\}$ are \emph{i.i.d.}, and $\xi_t(0)$ can be coupled with $\xi_t$ so that the distance between them can be measured in TV distance between $\pi_t$ and $\pi_{\infty}$.
\begin{align*}
&\quad \E[\|\xi_t(\theta_{t-1})\|^2 \mid \mathcal{F}_{t-1}] \\
&= \E[\|w_t(\theta_{t-1}; X, A)\nabla \ell(\theta_{t-1}; \zeta) - \nabla \L_{t, \theta_{t-1}}(\theta_{t-1})- w_{\infty} (\theta^*; X, A^*)\nabla \ell(\theta^*; \zeta^*) + \nabla \L_{\infty, \theta^*}(\theta^*)\|^2 \mid \theta_{t-1}]\\
&\leq C\E[\|w_t(\theta_{t-1}; X, A)\nabla \ell(\theta_{t-1}; \zeta) - w_{\infty} (\theta^*; X, A^*)\nabla \ell(\theta^*; \zeta^*)\|^2 \mid \theta_{t-1}]\\
&\quad + C\E[\| \nabla \L_{t, \theta_{t-1}}(\theta_{t-1})- \nabla \L_{\infty, \theta^*}(\theta^*)\|^2 \mid \theta_{t-1}]\\
&\leq C\E[\|w_t(\theta_{t-1}; X, A)\nabla \ell(\theta_{t-1}; \zeta) - w_{\infty} (\theta^*; X, A^*)\nabla \ell(\theta^*; \zeta^*)\|^2 \mid \theta_{t-1}]\\
&\quad {+ C\| \nabla \L_{t, \theta_{t-1}}(\theta_{t-1})-\nabla \L_{\infty, \theta_{t-1}}(\theta_{t-1})\|^2}\\
&\quad +C\| \nabla \L_{\infty, \theta_{t-1}}(\theta_{t-1})-\nabla \L_{\infty, \theta^*}(\theta^*)\|^2 
\end{align*}
The third term have been bounded by \eqref{eq:bound_M1}. {Since the upper bound for the first two terms can be obtained similarly, we only show the first one.} The first term can be decomposed as
\begin{align*}
&\quad \E[\|w_t(\theta_{t-1}; X, A)\nabla \ell(\theta_{t-1}; \zeta) - w_{\infty} (\theta^*; X, A^*)\nabla \ell(\theta^*; \zeta^*)\|^2 \mid \theta_{t-1}]\\
&\leq C \E\left[\delta_t(X_t, \theta_{t-1}) (1 + \|\theta_{t-1} - \theta^*\|^2)\phi(X_t)\right]  \\
&\quad+C \E\left[ \max_{A\in \A} \E [\|\nabla \ell(\theta_{t-1}; X, A, Y) - \nabla \ell(\theta^*; X, A, Y)\|^2\mid X, A] \right]  \\
&\quad + C\E\left[ \max_{A\in \A} |w_t(\theta_{t-1}; X_t, A) - w_{\infty}(\theta^*; X_t, A)|^2 (1 + \|\theta_{t-1} - \theta^*\|^2) \phi(X_t)\right],
\end{align*}
where $\delta_t(X, \theta) = d_{\rm TV}(\pi_t(X, \theta), \pi_\infty(X, \theta^*))$. This decomposition is similar to \eqref{eq:bound_M2b}. We now have
\begin{align*}
d_{\rm TV}(\pi_t(X, \theta), \pi_\infty(X, \theta^*)) \leq C d_{\rm TV}(\pi_\infty(X, \theta), \pi_\infty(X, \theta^*)) + C|\varepsilon_t - \varepsilon_\infty|
\end{align*}
Similarly, we have the following upper bound as well,
\begin{align*}
&\quad |w_t(\theta_{t-1}; X_t, A) - w_{\infty}(\theta^*; X_t, A)|\\
&\leq C |\Pr_{\pi_{t}}(A|X_t; \theta_{t-1}) - \Pr_{\pi_{\infty}}(A|X_t; \theta^*)|\\
&\leq C |\varepsilon_t - \varepsilon_\infty| + d_{\rm TV}(\pi_\infty(X, \theta), \pi_\infty(X, \theta^*)),
\end{align*}
{where $d_{\rm TV}(\pi_\infty(X, \theta), \pi_\infty(X, \theta^*))$ has been shown in \eqref{eq:eps-non-deg-TV} and \eqref{eq:eps-de-TV} for the non-degenerate and degenerate models, respectively.}
Combining these bounds above and Theorem~\ref{thm:smooth-clt}, it is sufficient to guarantee the validity of the central limit theorem result for both the degenerate and non-degenerate models.

Also notice that, from the above proof, it is easy to see that as long as $|\varepsilon_t - \varepsilon_\infty| = \mathcal{O}(t^{-\alpha/2})$,  we can still obtain the same result as Corollary \ref{corr:bahadur eps}.
\end{proof}

\section{Consistency of plug-in estimators} \label{sec:app-plugin}

\subsubsection*{Proof of Proposition~\ref{thm:plugin}}
\begin{proof}
Recall the definition of our estimators
\begin{align*}
\hat S_n=\frac1n\sum_{t=1}^nw_t^2 \nabla \ell(\theta_{t-1}; \zeta_t)\nabla \ell(\theta_{t-1}; \zeta_t)^\top,\quad \hat H_n=\frac1n\sum_{t=1}^nw_t\nabla^2 \ell(\theta_{t-1}; \zeta_t).
\end{align*}
In the proof of Theorem~\ref{thm:smooth-clt}, the bound on $M_2$ (Equation~\eqref{eq:bound_M2b}) implies the following convergence in $L^2$,
\begin{align*}
w(\theta_{t-1};X_t, A) \nabla \ell(\theta_{t-1}; \zeta_t) \to w(\theta^*;X_t, A^*) \nabla \ell(\theta^*; \zeta_t^*).
\end{align*}
Therefore we have the following convergence of $\hat S_n$ in $L^1$,
\begin{align*}
\hat S_n - \frac1n\sum_{t=1}^n w(\theta^*;X_t, A^*)^2 \nabla \ell(\theta^*; \zeta_t^*)\nabla \ell(\theta^*; \zeta_t^*)^\top \to 0
\end{align*}
Notice that by Law of Large Numbers,
\begin{align*}
\frac1n\sum_{t=1}^n w(\theta^*;X_t, A^*)^2 \nabla \ell(\theta^*; \zeta_t^*)\nabla \ell(\theta^*; \zeta_t^*)^\top \rightarrow S,
\end{align*}
in probability. Thus, combining our findings above, we can easily see that our plug-in estimator for gram matrix $\hat S_n \to S$ in probability.

Now we come to the consistency proof of $\hat H_n$. Notice that our Assumption~\ref{assum:plugin} is simply a repetition of Assumption~\ref{assum:gram} and Assumption~\ref{assum:tv} in Theorem~\ref{thm:smooth-clt} with $\phi$ replaced by $\psi$ and with gradient replaced by Hessian. So our proof of Theorem~\ref{thm:smooth-clt} from bound~\eqref{eq:bound_M2} to bound~\eqref{eq:bound_M2b} can be adapted here to prove the following convergence in $L^2$,
\begin{align*}
w(\theta_{t-1};X_t, A) \nabla^2 \ell(\theta_{t-1}; \zeta_t) \to w(\theta^*;X_t, A) \nabla^2 \ell(\theta^*; \zeta_t^*).
\end{align*}
{Similarly, we have $\hat H_n \to H$ in probability. By construction, $\|\hat{H}_n-\widetilde{H}_n\|\leq\|\hat{H}_n-H\|$, hence we also have $\hat H_n \to H$ in probability.}
\end{proof}

{\section{Uniform Asymptotic Normality}\label{appendix:uniform asymp}
We first introduce the definition for ``Uniform Convergence in Distribution''. See (\cite{van2013weak, zhang2022statistical}).
\begin{definition}\label{def:uniform}
Let $Z(\mathcal{P}) \in \mathbb{R}^{d}$ and $\{ Z_T(\mathcal{P}) \}_{T \ge 1} \subseteq \mathbb{R}^{d}$ 
be a sequence of random variables whose distributions are determined by some 
$\mathcal{P} \in \mathbf{P}$. 
We say that 
\[
Z_T(\mathcal{P}) \xrightarrow{d} Z(\mathcal{P})
\quad \text{uniformly over } \mathcal{P} \in \mathbf{P} \text{ as } T \to \infty
\]
if for any $\epsilon>0$, there exists $T_0(\epsilon)>0$ that does not depend on $\mathcal{P}$, such that for all $\mathcal{P} \in \mathbf{P}$,
\begin{equation}
\sup_{f \in BL_1}
\left| 
\mathbb{E}_{\mathcal{P}}\!\left[ f(Z_T(\mathcal{P})) \right] 
- 
\mathbb{E}_{\mathcal{P}}\!\left[ f(Z(\mathcal{\mathcal{P}})) \right]
\right|
\leq \epsilon,  \quad\text{for all } T>T_0(\epsilon),
\label{eq:uniform_conv_dist}
\end{equation}
where $BL_1$ denotes the set of functions 
$f : \mathbb{R}^{d_z} \to \mathbb{R}$ satisfying 
$\|f\|_\infty \le 1$ and 
$|f(z) - f(z')| \le \| z - z' \|$ for all $z,z' \in \mathbb{R}^{d_z}$.
\end{definition}

In our setting, $Z(\mathcal{P})$ denotes the Gaussian random variable $\Sigma^{1/2}W$, where
$W\sim \mathcal{N}(0,I_d)$ and $\Sigma=H^{-1}SH^{-1}$,
and $Z_T(\mathcal{P})$ represents the proposed estimator 
$\sqrt{T}\,(\bar\theta_T-\theta^*)$. Actually, the left side of \eqref{eq:uniform_conv_dist} can be upper bounded by $\E\|\sqrt{T}(\bar\theta_T-\theta^*)-\Sigma^{1/2}W\|$, which is exactly the approximation error of our Bahadur representation shown in Theorem \ref{thm:ls-fs}.

\subsection{Uniform Convergence for Modified $\varepsilon$-Greedy}\label{sec:uniform for eps}

The modified $\varepsilon$-greedy policy does not exhibit uniform convergence due to the behavior of the Hodges estimator, which depends on the indicator condition 
$\ID\!\left\{\|\theta_{0,t-1}-\theta_{1,t-1}\|\leq t^{-\frac{\alpha}{4}}\right\}$  to decide whether $\theta_0^*=\theta_1^*$. 
This condition ensures only \emph{pointwise} convergence, i.e., convergence for a specific contextual bandit environment $\P\in\mathbf{P}$, as $T$ becomes sufficiently large to determine whether $\theta_0^*=\theta_1^*$ holds. 
However, it does not guarantee that \eqref{eq:uniform_conv_dist} holds uniformly with a common threshold $T_0(\epsilon)$ such that the convergence error remains below any fixed $\epsilon>0$ across all $\P\in\mathbf{P}$. 

Figure~\ref{fig:heat-map-uniform} illustrates this failure of uniform convergence. 
In the non-degenerate model considered in Theorem~\ref{prop:ls-reg}, the confidence interval (CI) length remains constant, as shown in the left panel for the $\varepsilon$-greedy policy. 
In the right panel, the upper-right yellow region indicates cases where the model converges to the non-degenerate regime, while the remaining area represents the transition to the non-degenerate model. 
As $T$ increases, more rows (each corresponding to a specific $\P$) turn yellow, implying convergence toward the non-degenerate model. 
Nevertheless, there is no uniform time point $T$ along the $x$-axis such that all cubes become yellow simultaneously, demonstrating the lack of a uniform $T_0(\epsilon)$.

\begin{figure}[!t]
    \centering
    \includegraphics[width=0.7\textwidth]{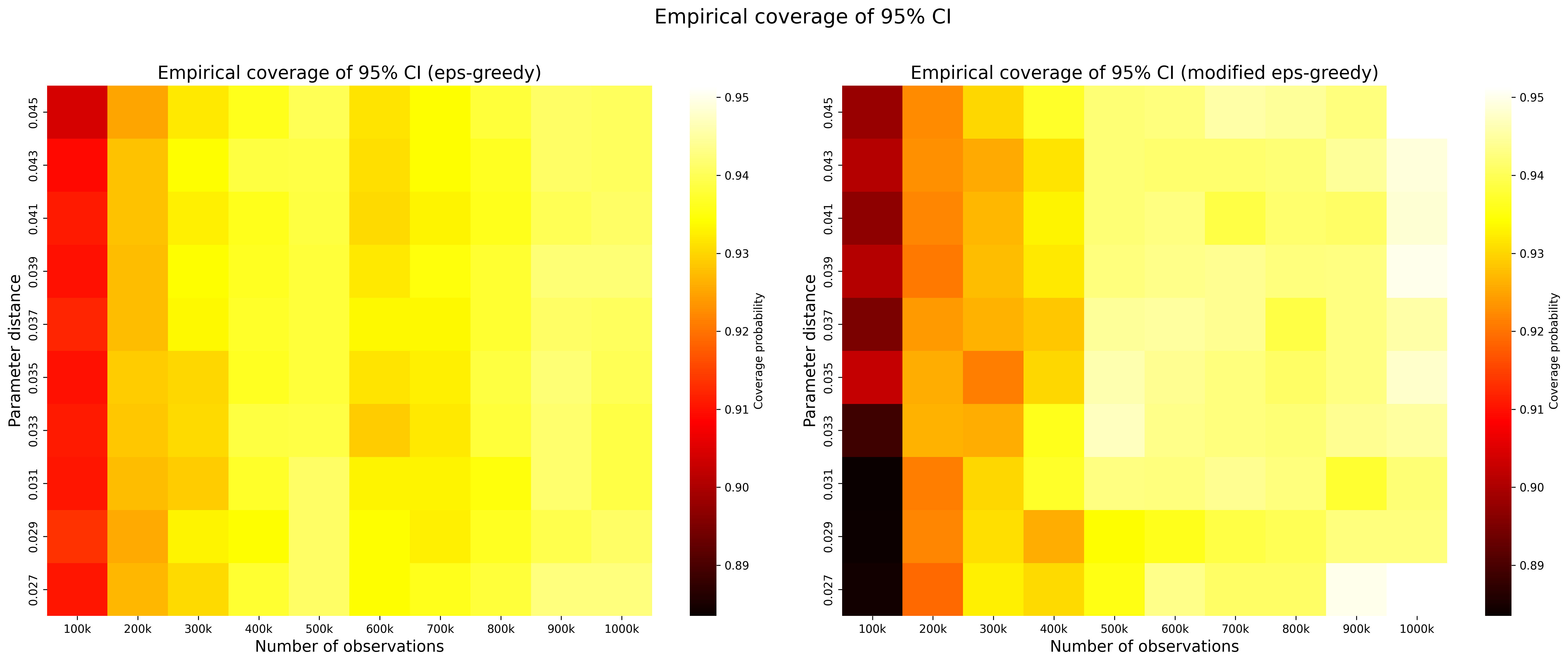}
    \includegraphics[width=0.7\textwidth]{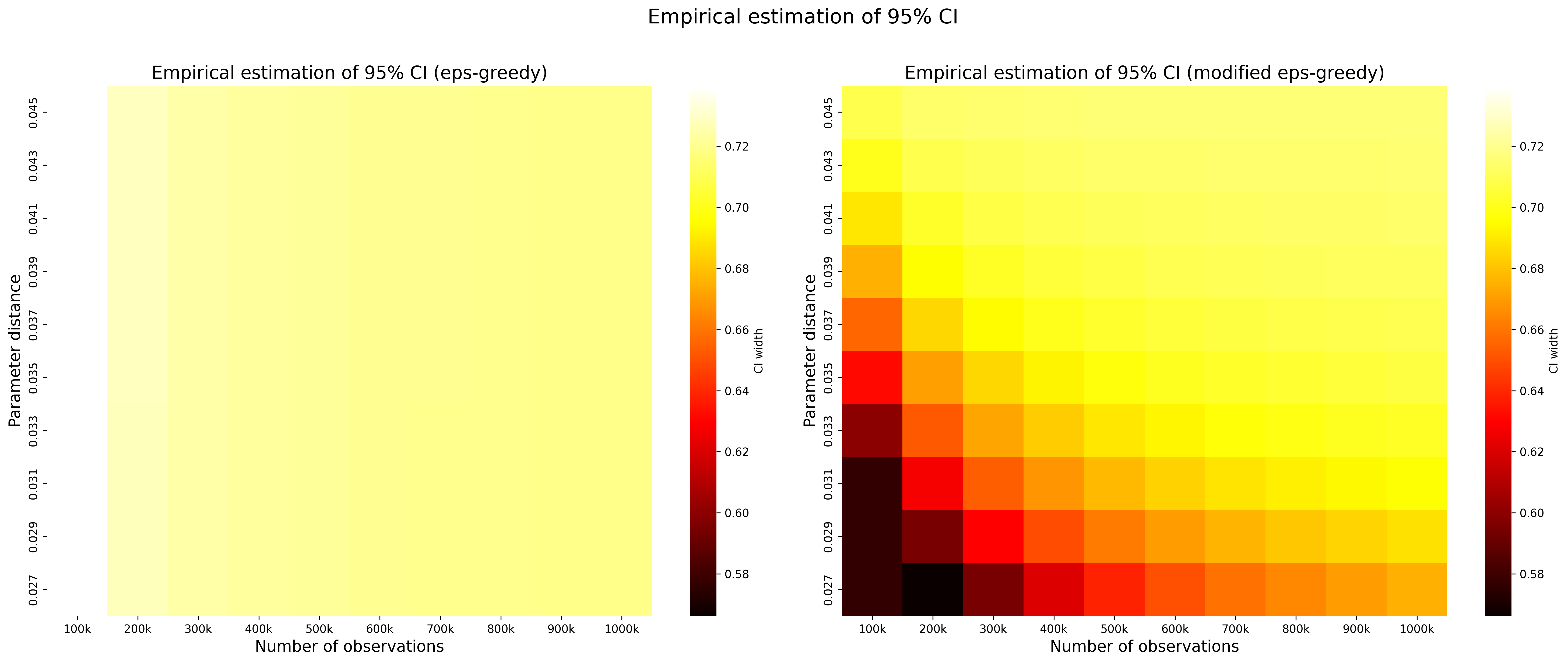}
    \caption{SGD on linear regression with \sipw \ in the non-degenerate model. We report the empirical coverage rate and its corresponding 95\% CI length.}
    \label{fig:heat-map-uniform}
\end{figure}

\subsection{Uniform Convergence for Exponential Policy}\label{sec:uniform for exp3}

The exponential policy admits uniform convergence because the Bahadur error 
$\E\|\sqrt{T}(\bar{\theta}_T - \theta^*) - \Sigma^{1/2}W\|$ 
can be controlled by 1) the convergence of $\theta_T$ to $\theta^*$, and 2) the discrepancy between the arm distributions induced by $\theta_T$ and $\theta^*$. 
Both terms can be bounded by a constant $C$ independent of $T$, multiplied by a decaying factor of order $T^{-\beta}$ for some $\beta > 0$. In particular, the constant $C$ in the latter term arises from the Lipschitz continuity of the softmax function. Specifically,
\begin{align*}
\sum_{a\in\A}\left|
\frac{e^{\lambda X^\top\theta_a}}{\sum_{a'\in\A} e^{\lambda X^\top\theta_{a'}}}
-
\frac{e^{\lambda X^\top\theta^*_a}}{\sum_{a'\in\A} e^{\lambda X^\top\theta^*_{a'}}}
\right|
\le C\|X\|\cdot |X^\top(\theta - \theta^*)|
\le C\|X\|^2 \|\theta - \theta^*\|,
\end{align*}
since
\[
\left\|\nabla_{\theta_a} 
\frac{e^{\lambda X^\top \theta_a}}{\sum_{a'\in\A} e^{\lambda X^\top \theta_{a'}}}
\right\|
= \lambda p_a(1-p_a)\|X\|
\le C\|X\|,
\]
where $p_a$ denotes the arm-selection probability for $a = 0, 1$.

}

\section{Auxiliary Lemmas}\label{appendix:Auxiliary Lemmas}
{We first introduce a special function class called Schwartz space 
${\mathcal {S}}$, which is the function space of all functions whose derivatives are rapidly decreasing. In detail, let  $\mathbb{N}$  be the set of non-negative integers, and for any  $n \in \mathbb{N}$ , let  $\mathbb{N}^{n}:=\underbrace{\mathbb{N} \times \cdots \times \mathbb{N}}_{n \text { times }}$  be the  $n$ -fold Cartesian product. The Schwartz space or space of rapidly decreasing functions on  $\mathbb{R}^{n}$  is the function space
\begin{align*}
    S\left(\mathbb{R}^{n}, \mathbb{C}\right):=\left\{f \in C^{\infty}\left(\mathbb{R}^{n}, \mathbb{C}\right) \mid \forall \alpha, \beta \in \mathbb{N}^{n},\|f\|_{\alpha, \beta}<\infty\right\},
\end{align*}
where  $C^{\infty}\left(\mathbb{R}^{n}, \mathbb{C}\right)$  is the function space of smooth functions from  $\mathbb{R}^{n}$  into  $\mathbb{C} $, and
$$\|f\|_{\alpha, \beta}:=\sup _{x \in \mathbb{R}^{n}}\left|x^{\alpha}\left(D^{\beta} f\right)(x)\right|,$$
where $x^\alpha:=x_1^{\alpha_1}\cdots x_n^{\alpha_n}$ and $D^\beta:=\partial_1^{\beta_1}\cdots\partial_n^{\beta_n}$.
}
\begin{lemma}\label{lem:hessian derivative}
For a integral function $p(x)$ defined over $\mathbb{R}^p$, define the function $J(\theta) = \int \ID(\theta^\top x > 0) p(x) \d x$. For $\theta \neq 0$, assume $\int_{\theta^\top x = 0} p(x) \|x\| \d x < \infty$, {$p(x)$ bounded and $\|\nabla p(x)\| \leq C\|x\|^a$ where $a < -p-1$ and $C>0$.} The gradient of $J(\theta)$ is given by
\begin{align*}
\nabla J(\theta) = \int_{\theta^\top x = 0} p(x) x \d x.
\end{align*}
\end{lemma}
\begin{proof}
Without loss of generality, we can assume $\theta = e_1 = (1, 0, \dots, 0)$. Define function $q(x_1, x_2, \dots, x_p) = p(0, x_2, \dots, x_p)$. Pick any $\Delta \theta$ such that $\Delta \theta^\top \e_1 = 0$. Then $\Delta \theta^\top x$ does not depend on $x_1$.
\begin{align*}
&\quad\int \ID((\theta+\Delta\theta)^\top x > 0) q(x) \d x - \int \ID(\theta^\top x > 0) q(x) \d x\\
&= \int (\int(\ID(x_1>-\Delta\theta^\top x)-\ID(x_1>0)) q(x) \d x_1) \d x_2 \dots \d x_p\\
&= \int (\int_{-\Delta\theta^\top x}^0 q(x) \d x_1) \d x_2 \dots \d x_p\\
&= \int \Delta\theta^\top x p(0, x_2, \dots, x_p) \d x_2 \dots \d x_p\\
&= \Delta\theta^\top  \int_{\theta^\top x = 0} p(x) x \d x.
\end{align*}

By the above equality, we can deduce that
\begin{align*}
&\quad |J(\theta + \Delta \theta) - J(\theta) - \Delta\theta^\top  \int_{\theta^\top x = 0} p(x) x \d x|\\
&= |\int \ID((\theta+\Delta\theta)^\top x > 0) p(x) \d x - \int \ID(\theta^\top x > 0) p(x) \d x - \Delta\theta^\top  \int_{\theta^\top x = 0} p(x) x \d x |\\
&= \int |\ID((\theta+\Delta\theta)^\top x > 0) - \ID(\theta^\top x > 0)) (p(x)-q(x))| \d x\\
&\leq \int \ID(|x_1| \leq |\Delta\theta^\top x|) |p(x)-q(x)| \d x\\
&\leq \int (\int_{x_1^2 \leq \|\Delta \theta\|^2(x_2^2+\dots + x_p^2)} |p(x)-q(x)| \d x_1) \d x_2 \dots \d x_p\\
&\leq C\int (\int_{x_1^2 \leq \|\Delta \theta\|^2(x_2^2+\dots + x_p^2)}|x_1|(x_2^2 + \cdots + x_p^2)^{a/2} \d x_1) \d x_2 \dots \d x_p\\
&\leq C\|\Delta \theta\|^2 \int (x_2^2 + \cdots + x_p^2)^{a/2+1} \d x_2 \dots \d x_p\\
&\leq C\|\Delta \theta\|^2.
\end{align*}

For a general $\Delta \theta$ that is not necessarily orthogonal to $\theta$, and $\|\Delta \theta\|<1/2$, notice that $J(\theta) = J(c\theta)$ for any $c>0$. So
\[
J(\theta + \Delta \theta) = J((\theta + \Delta \theta)/(1+\Delta \theta_1)).
\]
and $((\theta + \Delta \theta)/(1+\Delta \theta_1)-\theta)^\top \theta = 0$. It is easy to verify that $\|(\theta + \Delta \theta)/(1+\Delta \theta_1)-\theta\| \leq \|\Delta \theta\|$ and $\theta^\top \int_{\theta^\top x = 0} p(x) x \d x = 0$. Hence we can use previous results and get
\begin{align*}
&\quad |J(\theta + \Delta \theta) - J(\theta) - \Delta \theta \int_{\theta^\top x = 0} p(x) x \d x|\\
&\leq |((\theta + \Delta \theta)/(1+\Delta \theta_1)-\theta - \Delta \theta)^\top \int_{\theta^\top x = 0} p(x) x \d x| + C\|(\theta + \Delta \theta)/(1+\Delta \theta_1)-\theta\|^2\\
&\leq \|\Delta \theta \Delta \theta_1/(1+\Delta \theta_1)\|\cdot\|\int_{\theta^\top x = 0} p(x) x \d x\| + C \|\Delta \theta\|^2\\
&\leq C'\|\Delta \theta\|^2.
\end{align*}
\end{proof}

\begin{lemma}\label{lem:8th rate}
    Assume the almost sure convergence of $\theta_{t-1}$ to $\theta^*$ holds. We further assume the local strong convexity in Assumption \ref{assum:loss} holds, and for any action $A \in \A$ and covariate $X$, $\|\nabla \ell(\theta; \zeta)\|^8$ exists almost surely under $\P_{Y\mid X,A}$, and $\E \left(\|\nabla \ell(\theta; \zeta)\|^8 \mid X, A \right) \leq C_3(1+\|\theta - \theta^*\|^8 )$ where $C_3$ is some positive constant. Then, with step size $\eta_t=\eta_0 t^{-\alpha}$, $\alpha\in(1/2,1)$ and error $\delta_t=\theta_{t}-\theta^*$, we have:
    \begin{enumerate}[(a)]
        \item There exists a positive constant $t_0$, such that for $t>s\geq t_0$,
        \begin{align*}
            \E\left(\|\delta_t\|^8\mid \mathcal{F}_s\right)\leq \exp\left(-\frac{1}{2}\mu\sum_{i=s}^t\eta_i\right)\|\delta_s\|^8+Ct^{-4\alpha}.
        \end{align*}
        \item As a consequence, 
        \begin{align*}
            \E\|\delta_t\|^8\leq Ct^{-4\alpha}.
        \end{align*}
    \end{enumerate}
\end{lemma}
\begin{proof}
    For the first part, through the SGD iteration,
    \begin{align*}
        \delta_t=\delta_{t-1}-\eta_{t}\left(\nabla\L_{\theta_{t-1}}(\theta_{t-1})+\xi_{\theta_{t-1}}(\theta_{t-1};\zeta_{t})\right).
    \end{align*}
    For simplicity, let $\nabla\L_{t-1}:=\nabla\L_{\theta_{t-1}}(\theta_{t-1})$ and $\xi_{t-1}:=\xi_{\theta_{t-1}}(\theta_{t-1};\zeta_{t})$,
hence,
\begin{align*}
    \|\delta_t\|^2=\underbrace{\|\delta_{t-1}\|^2}_A+\underbrace{\eta_t^2\|\nabla\L_{t-1}+\xi_{t-1}\|^2}_B+\underbrace{2\eta_t\langle\delta_{t-1},-\nabla\L_{t-1}-\xi_{t-1}\rangle}_C.
\end{align*}
Then, 
\begin{align*}
    \|\delta_t\|^8
    &=(A+B+C)^4\\
    &=A^4+B^4+C^4+12(A^2BC+AB^2C+ABC^2)\\
    &\quad+6(A^2B^2+A^2C^2+B^2C^2)\\
    &\quad+4(A^3B+A^3C+AB^3+B^3C+AC^3+BC^3).
\end{align*}
Using H\"{o}lder's inequality and Young's inequality,
\begin{itemize}
    \item $12A^2BC+4AC^3$:
    \begin{align*}
        12A^2BC+4AC^3
        &\leq56\|\delta_{t-1}\|^5\|\nabla\L_{t-1}+\xi_{t-1}\|^3\eta_t^3\\
        &= 56^{\frac{5}{8}}\eta_{t}^{3-\frac{15}{8}}\mu^{\frac{5}{8}}\|\delta_{t-1}\|^5\cdot 56^{\frac{3}{8}}\|\nabla\L_{t-1}+\xi_{t-1}\|^3\eta_t^{\frac{15}{8}}\mu^{-\frac{5}{8}}\\
        &\leq 35\mu\|\delta_{t-1}\|^8\eta_t^{\frac{9}{5}}+21\mu^{-\frac{5}{3}}\|\nabla\L_{t-1}+\xi_{t-1}\|^8\eta_t^5.
    \end{align*}
    Similarly,
    \item $12AB^2C+4BC^3$:
    \begin{align*}
        12AB^2C+4BC^3\leq 56\|\delta_{t-1}\|^3\|\nabla\L_{t-1}+\xi_{t-1}\|^5\eta_t^5\leq 21\|\delta_{t-1}\|^8\eta_t^5+35\|\nabla\L_{t-1}+\xi_{t-1}\|^8\eta_t^5.
    \end{align*}
    \item $12ABC^2+6A^2B^2+C^4$:
    \begin{align*}
        12AB^2C+4BC^3\leq 70\|\delta_{t-1}\|^4\|\nabla\L_{t-1}+\xi_{t-1}\|^4\eta_t^4\leq 35\|\delta_{t-1}\|^8\eta_t^3+35\|\nabla\L_{t-1}+\xi_{t-1}\|^8\eta_t^5.
    \end{align*}
    \item $6A^2C^2+4A^3B$:
    \begin{align*}
        6A^2C^2+4A^3B\leq 28\|\delta_{t-1}\|^6\|\nabla\L_{t-1}+\xi_{t-1}\|^2\eta_t^2\leq 3\mu\|\delta_{t-1}\|^8\eta_t+\cdot7^4\mu^{-3}\|\nabla\L_{t-1}+\xi_{t-1}\|^8\eta_t^5.
    \end{align*}
    \item $6B^2C^2+4AB^3$:
    \begin{align*}
        6B^2C^2+4AB^3\leq 28\|\delta_{t-1}\|^2\|\nabla\L_{t-1}+\xi_{t-1}\|^6\eta_t^6\leq 7\|\delta_{t-1}\|^8\eta_t^9+21\|\nabla\L_{t-1}+\xi_{t-1}\|^8\eta_t^5.
    \end{align*}
    \item $4B^3C$:
    \begin{align*}
        4B^3C\leq 8\|\delta_{t-1}\|\|\nabla\L_{t-1}+\xi_{t-1}\|^7\eta_t^7\leq \|\delta_{t-1}\|^8\eta_t^{21}+7\|\nabla\L_{t-1}+\xi_{t-1}\|^8\eta_t^5.
    \end{align*}
    Additionally, we have
    \item $4A^3C$:
    \begin{align*}
        4\E_{t-1}A^3C=8\eta_t\|\delta_{t-1}\|^6\langle\delta_{t-1},\nabla\L_{t-1}+\xi_{t-1}\rangle\leq -8\mu\eta_t\|\delta_{t-1}\|^8,
    \end{align*}
    since $\E[\xi_{t-1}\mid \mathcal{F}_{t-1}]=0$ and $\langle\delta_{t-1},\nabla\L_{t-1}\rangle\geq \mu\|\delta_{t-1}\|^2$, when $t>t_0^\prime$ for some $t_0^\prime>0$. When $t\leq t_0^\prime$,
    \begin{align*}
        4\E_{t-1}A^3C
        \leq 8\eta_t\|\delta_{t-1}\|^7\|\nabla\L_{t-1}+\xi_{t-1}\|
        \leq 7\eta_t^{\frac{3}{7}}\|\delta_{t-1}\|^8+\eta_t^5\|\nabla\L_{t-1}+\xi_{t-1}\|^8.
    \end{align*}
    \item $A^4=\|\delta_{t-1}\|^8$.
    \item $B^4=\eta_t^8\|\nabla\L_{t-1}+\xi_{t-1}\|^8$.
\end{itemize}
Using the condition in this lemma, 
\begin{align*}
    \E\left(\|\nabla\L_{t-1}+\xi_{t-1}\|^8\mid \mathcal{F}_{t-1}\right)\leq C\left(1+\|\delta_{t-1}\|^8\right).
\end{align*}
Combining all the results, when $t>0$,
\begin{footnotesize}
    \begin{align}\label{eq:8th recursive original}
    &\quad\E\left(\|\delta_{t}\|^8\mid \mathcal{F}_{t-1}\right)\notag\\
    &\leq \left[1+\mu\left(3\eta_t+35\eta_t^{9/5}\right)+7\eta_t^{\frac{3}{7}}+35\eta_t^3+C\left(21\mu^{-5/3}+7^4\mu^{-3}+120\right)\eta_t^5+C\eta_t^8+7\eta_t^9+\eta_t^{21}\right]\|\delta_{t-1}\|^8\notag\\
    &\quad+\left(21\mu^{-5/3}+7^4\mu^{-3}+99\right)C\eta_t^5+C\eta_t^8.
\end{align}
\end{footnotesize}
Specifically,
when $t>t_0^\prime$,
\begin{footnotesize}
    \begin{align*}
    &\quad\E\left(\|\delta_{t}\|^8\mid \mathcal{F}_{t-1}\right)\\
    &\leq \left[1-\mu\left(5\eta_t-35\eta_t^{9/5}\right)+35\eta_t^3+C\left(21\mu^{-5/3}+7^4\mu^{-3}+119\right)\eta_t^5+C\eta_t^8+7\eta_t^9+\eta_t^{21}\right]\|\delta_{t-1}\|^8\\
    &\quad+\left(21\mu^{-5/3}+7^4\mu^{-3}+98\right)C\eta_t^5+C\eta_t^8.
\end{align*}
\end{footnotesize}
Let 
\begin{footnotesize}
    \begin{align*}
    t_0''=\min\left\{t:-\mu\left(5\eta_t-35\eta_t^{9/5}\right)+35\eta_t^3+C\left(21+21\mu^{-5/3}+7^4\mu^{-3}+119\right)\eta_t^5+C\eta_t^8+7\eta_t^9+\eta_t^{21}\leq-0.5\mu\eta_t\right\},
\end{align*}
and $t_0=\max\{t_0^\prime,t_0''\}$,
\end{footnotesize} 
then
for $t>t_0$,
\begin{align}\label{eq:8th recursion}
    \quad\E\left(\|\delta_{t}\|^8\mid\mathcal{F}_{t-1}\right)\leq\left(1-\frac{1}{2}\mu\eta_t\right)\|\delta_{t-1}\|^8+C\eta_t^5.
\end{align}
Using Lemma B.2 in \cite{chen2016statistical}, we can directly obtain
\begin{align*}
    \E\left(\|\delta_t\|^8\mid\mathcal{F}_s\right)\leq \exp\left(-\frac{1}{2}\mu\sum_{i=s}^t\eta_i\right)\|\delta_s\|^8+Ct^{-4\alpha},
\end{align*}
for $t>s\geq t_0$.

For the second part, by using discrete Gronwall's inequality to \eqref{eq:8th recursive original}, we have $\E\|\delta_{t_0}\|^8\leq C$. Therefore, according to claim (a), for $t\geq 2t_0$, we have
\begin{align*}
    &\E\left(\left\|\delta_{\frac{t}{2}}\right\|^8\mid\mathcal{F}_{t_0}\right)
    \leq \|\delta_{t_0}\|^8+Ct_0^{-4\alpha}, \\
    &\E\left(\|\delta_t\|^8\mid\mathcal{F}_{\frac{t}{2}}\right)
    \leq \exp\left(\frac{-\mu t\eta_t}{4}\right)\left\|\delta_{\frac{t}{2}}\right\|^8+C\left(\frac{t}{2}\right)^{-4\alpha}.
\end{align*}
In conclusion, we have
\begin{align*}
    \E\|\delta_{t}\|^8
    &\leq \exp\left(\frac{-\mu t\eta_t}{4}\right)\left(\E\|\delta_{t_0}\|^8+Ct_0^{-4\alpha}\right)+C\left(\frac{t}{2}\right)^{-4\alpha}\\
    &\leq C\exp\left(\frac{-\mu t^{1-\alpha}}{4}\right)+Ct^{-4\alpha}\\
    &\leq Ct^{-4\alpha}.
\end{align*}
\end{proof}

\begin{lemma}\label{lem:a.s.}
    Under the degenerate model with modified $\varepsilon$-greedy policy in equation (\ref{eq:modified eps-greedy}), if the condition in Lemma \ref{lem:8th rate} holds, then it leads to $\|\theta_{0,t-1}-\theta_{1,t-1}\|\leq t^{-\frac{\alpha}{4}}$ holds almost surely, which implies $\Tilde{\Pr}(A_t = 0 \mid X_t, \theta_{t-1})=1/2$ holds almost surely.
\end{lemma}
\begin{proof}
For the first part, because of Lemma \ref{lem:8th rate}, we have 
$\E\|\theta_{i,t-1}-\theta^*\|^8\leq Ct^{-4\alpha}$, $i=0,1$. Through Markov's inequality, 
\begin{align*}
    \Pr\left(\|\theta_{i,t-1}-\theta^*\|> t^{-\frac{\alpha}{4}}\right)
    =\Pr\left(\|\theta_{i,t-1}-\theta^*\|^8> t^{-2\alpha}\right)
    \leq \frac{Ct^{-4\alpha}}{t^{-2\alpha}}=Ct^{-2\alpha},\quad i=0,1.
\end{align*}
Then, 
\begin{align*}
    \Pr\left(\|\theta_{0,t-1}-\theta_{1,t-1}\|> t^{-\frac{\alpha}{4}}\right)
    &\leq \Pr\left(\|\theta_{0,t-1}-\theta_0^*\|> \frac{1}{2}t^{-\frac{\alpha}{4}}\right)
    +\Pr\left(\|\theta_{1,t-1}-\theta_1^*\|> \frac{1}{2}t^{-\frac{\alpha}{4}}\right)\\
    &\leq Ct^{-2\alpha}.
\end{align*}
Now define event $B_t=\left\{\|\theta_{0,t-1}-\theta_{1,t-1}\|> t^{-\frac{\alpha}{4}}\right\}$, then $\sum_{t=1}^\infty \Pr(B_t)< \infty$. Using Borel-Cantelli's lemma, we have $\Pr\left(\limsup_{t\rightarrow\infty}B_t\right)=0$, which is equivalent to $\Pr\left(\liminf_{\substack{t\rightarrow\infty}}B_t^C\right)=1$, meaning $B_t^C$ holds almost surely. That is, $\|\theta_{0,t-1}-\theta_{1,t-1}\|\leq t^{-\frac{\alpha}{4}}$ holds almost surely.

For the second part, from the definition of modified $\varepsilon$-greedy in \eqref{eq:modified eps-greedy}, once $\|\theta_{0,t-1}-\theta_{1,t-1}\|\leq t^{-\frac{\alpha}{4}}$ holds, $\Tilde{\Pr}(A_t = 0 \mid X_t, \theta_{t-1})=1/2$ can always be triggered.

\end{proof}

\begin{lemma}\label{lem:non-degenerate-a.s.}
    Under the non-degenerate model with modified $\varepsilon$-greedy policy in equation (\ref{eq:modified eps-greedy}), if the condition in Lemma \ref{lem:8th rate} holds, we have 
    \begin{align*}
        \Tilde{\Pr}(A_t = a \mid X_t, \theta_{t-1})=\Pr(A_t = a \mid X_t, \theta_{t-1}), \ a=0,1,\text{ holds almost surely,}
    \end{align*}
    where $\Pr(A_t = a \mid X_t, \theta_{t-1})(1-\varepsilon)\ID \left\{\argmax_{a=0,1}X_t^\top \theta_{a,t-1}\right\} + \frac{\varepsilon}{2}$.
\end{lemma}
\begin{proof}
    The almost sure convergence of $\theta_{t-1}$ to $\theta^*$ is equivalent to $\lim_{t\rightarrow\infty}\|\theta_{t-1}-\theta^*\|=0$, a.s. Hence,
\begin{align*}
    \lim_{t\rightarrow\infty}\Big|\|\theta_{0,t-1}-\theta_{1,t-1}\|-\|\theta_0^*-\theta_1^*\|\Big|\leq
    \lim_{t\rightarrow\infty}\|\theta_{0,t-1}-\theta_{1,t-1}-(\theta_0^*-\theta_1^*)\|=0,\ a.s.
\end{align*}
Therefore, for any $\epsilon>0$, then there exists $T_0(\epsilon)>0$ and event 
\begin{align*}
    \Gamma^\prime=\cap_{t=T_0(\epsilon)+1}^\infty\left\{\|\theta_0^*-\theta_1^*\|-\epsilon\leq\|\theta_{0,t-1}-\theta_{1,t-1}\|\leq \|\theta_0^*-\theta_1^*\|+\epsilon\right\},
\end{align*}
such that $\Pr(\Gamma^\prime)=1$. 
Let $\epsilon=\|\theta_0^*-\theta_1^*\|/2$, then we have 
$\|\theta_{0,t-1}-\theta_{1,t-1}\|\in [\|\theta_0^*-\theta_1^*\|/2,3\|\theta_0^*-\theta_1^*\|/2]$ holds for large enough $t$, implying that $\|\theta_{0,t-1}-\theta_{1,t-1}\|>t^{-\alpha/4}$ holds for large enough $t$. Therefore we come to the conclusion.

\end{proof}

\end{document}